\documentclass[twoside,11pt]{article}

\usepackage[preprint]{jmlr2e}
\usepackage{blindtext}
\usepackage{graphicx} 
\usepackage{xcolor}
\usepackage{hyperref}
\usepackage{url}
\usepackage{mlmath}
\usepackage{array}
\usepackage{tabularx} 
\let\proof\relax

\usepackage{amsmath,amsfonts,amssymb,mathtools,amsthm}

\newtheorem{theorem}{Theorem}
\newtheorem*{theorem*}{Theorem}
\newtheorem{lemma}{Lemma}
\newtheorem*{lemma*}{Lemma}

\newtheorem*{property*}{Property}
\newtheorem{remark}{Remark}

\newtheorem{assumption}{Assumption}
\newtheorem*{assumption*}{Assumption}
\newtheorem*{prop*}{Proposition}
\newtheorem{prop}{Proposition}

\newtheorem*{setting*}{Setting}

\newtheorem{task}{Task}

\usepackage{lastpage}
\jmlrheading{23}{2026}{1-\pageref{LastPage}}{}{}{21-0000}{Junjie Yao, Liangkai Hang, and Zhi-Qin John Xu}

\ShortHeadings{Context Staircase: Signature-Aligned Token Embedding Dynamics}{Yao, Hang, and Xu}
\firstpageno{1}

\begin{document}

\title{Context Staircase: Signature-Aligned Dynamics of Token Embeddings under Small Initialization}

\author{\name Junjie Yao \email yaojj22@sjtu.edu.cn \\
       \addr School of Mathematical Sciences\\
       Shanghai Jiao Tong University\\
       Shanghai 200240, China
       \AND
       \name Liangkai Hang \email hangliangkai@sjtu.edu.cn \\
       \addr School of Mathematical Sciences\\
       Shanghai Jiao Tong University\\
       Shanghai 200240, China
       \AND
       \name Zhi-Qin John Xu\textsuperscript{*} \email xuzhiqin@sjtu.edu.cn \\
       \addr Institute of Natural Sciences, School of Mathematical Sciences\\
       Shanghai Jiao Tong University\\
       Shanghai 200240, China\\
       Corresponding author}


\maketitle

\begin{abstract}
Token embeddings are the basic representational units that connect discrete tokens with continuous computation in language models. Although modern language models learn embeddings from random initialization through gradient-based training, the dynamical mechanism by which meaningful embedding structures emerge remains unclear. In this work, we identify that the evolving embedding structures are closely related to token-conditioned label and contextual distributions, which we formalize as probability signatures. We observe a progressive learning process, which we term \emph{Context Staircase}: embeddings learn the low-order statistic signatures of the data before the high-order ones. More specifically, we observe that early in training they align with the simplest, context-free signature linking a token to its label, and as training proceeds, they progressively reflect signatures involving more and more context tokens. We then analyze the gradient flow of embeddings under small initialization to explain this phenomenon, deriving embedding evolution equations for feed-forward and self-attention architectures. We further extend these observations to real language-model training. Finally, we show that these embedding structures play an important role in both task learning and the incorporation of semantic structure into the embedding space. Overall, our results provide a dynamic explanation of how data statistics and architecture jointly shape token embeddings in language models, and reveal an implicit bias in the space of data statistics: training proceeds from simpler, low-order statistical relations toward increasingly complex, context-dependent ones.
\end{abstract}


\section{Introduction}

Large language models (LLMs) have become powerful general-purpose systems, achieving strong performance across a wide range of language tasks~\citep{brown2020language,touvron2023llama,openai2023gpt4,guo2025deepseekr1}. At the foundation of these systems lies the token embedding space, which connects discrete symbols to continuous computation. Early static word representations were often constructed from distributional statistics, such as word co-occurrence, or through matrix factorization~\citep{pennington2014glove,levy2014neural}. Predictive approaches such as CBOW and Skip-gram instead learned word embeddings through dedicated context-prediction objectives, by predicting a target word from its surrounding context or vice versa~\citep{mikolov2013efficient,mikolov2013distributed}. Later approaches further incorporated subword information into these representations~\citep{bojanowski2017enriching}. Modern language models initialize token embeddings randomly and optimize them jointly with the rest of the model under the language modeling objective~\citep{vaswani2017attention,radford2019language,brown2020language}. This raises a basic question: what statistical structures are written into token embeddings during training, and how do these structures evolve over time?

Prior work has shown that learned representations contain meaningful semantic organization~\citep{ethayarajh2019contextual,hewitt2019structural,gurnee2024language,han2024word}. Most of these studies analyze representations after training or probe what information they contain. Much less is known about how these structures emerge in the embedding. In particular, it remains unclear whether the geometry observed at different stages of training can be related to identifiable statistics of the data and the model architecture. Recent studies have further shown that initialization can substantially affect the evolution and geometry of token embeddings~\citep{zhang2024init,zhang2025complexity,yao2025analysis}. In particular, small initialization tends to produce more structured embedding spaces that better reflect the underlying relations of the task, and such embedding structures have been closely associated with improved reasoning ability.

Motivated by these observations, we systematically investigate the evolution of token embeddings under small initialization. We introduce a hierarchy of token-conditioned statistical quantities, called \emph{probability signatures}, which describe the joint distributions among a token, the target label, and the context tokens. Across controlled sequence-prediction tasks, we identify a characteristic evolution that we term \emph{Context Staircase}: the early embedding geometry is closely associated with low-order signatures, while structures associated with higher-order signatures become increasingly visible as training proceeds. As illustrated by the addition task in Figure~\ref{fig:overview}, the embeddings first form a clear linear structure, closely matching that of the low-order signature. Later in training, the embedding geometry develops a pronounced odd--even separation, which is likewise reflected in the higher-order signature. More broadly, Context Staircase reveals an \emph{implicit bias in the space of data statistics}: training preferentially incorporates simpler, low-order statistical relations involving fewer context variables before progressively incorporating more complex, higher-order relations involving richer context.

To explain these empirical observations, we analyze the gradient flow of token embeddings under small initialization. We show that the embedding gradient can be decomposed into contributions associated with label signatures of different orders. Small initialization strongly suppresses higher-order contributions at the beginning of training, explaining the early association with low-order signatures. As training proceeds and the parameter scale grows, contributions from all accessible signature orders reach the same leading scale and can jointly affect the embedding dynamics, with the highest-order contribution being the largest under our late-stage conditions. Thus, small initialization provides a dynamical mechanism for the observed implicit bias, inducing an ordered progression from simple to increasingly context-dependent statistical structures.

For self-attention models, the corresponding signature terms are position-dependent and modulated by attention scores, which control how strongly and how rapidly different tokens affect the embedding dynamics. We further connect these analyses to distinct computational paths in a full decoder-only Transformer and test the resulting hypotheses on a real language-model training run.

\begin{figure}[htb]
    \centering
    \includegraphics[width=1\linewidth]{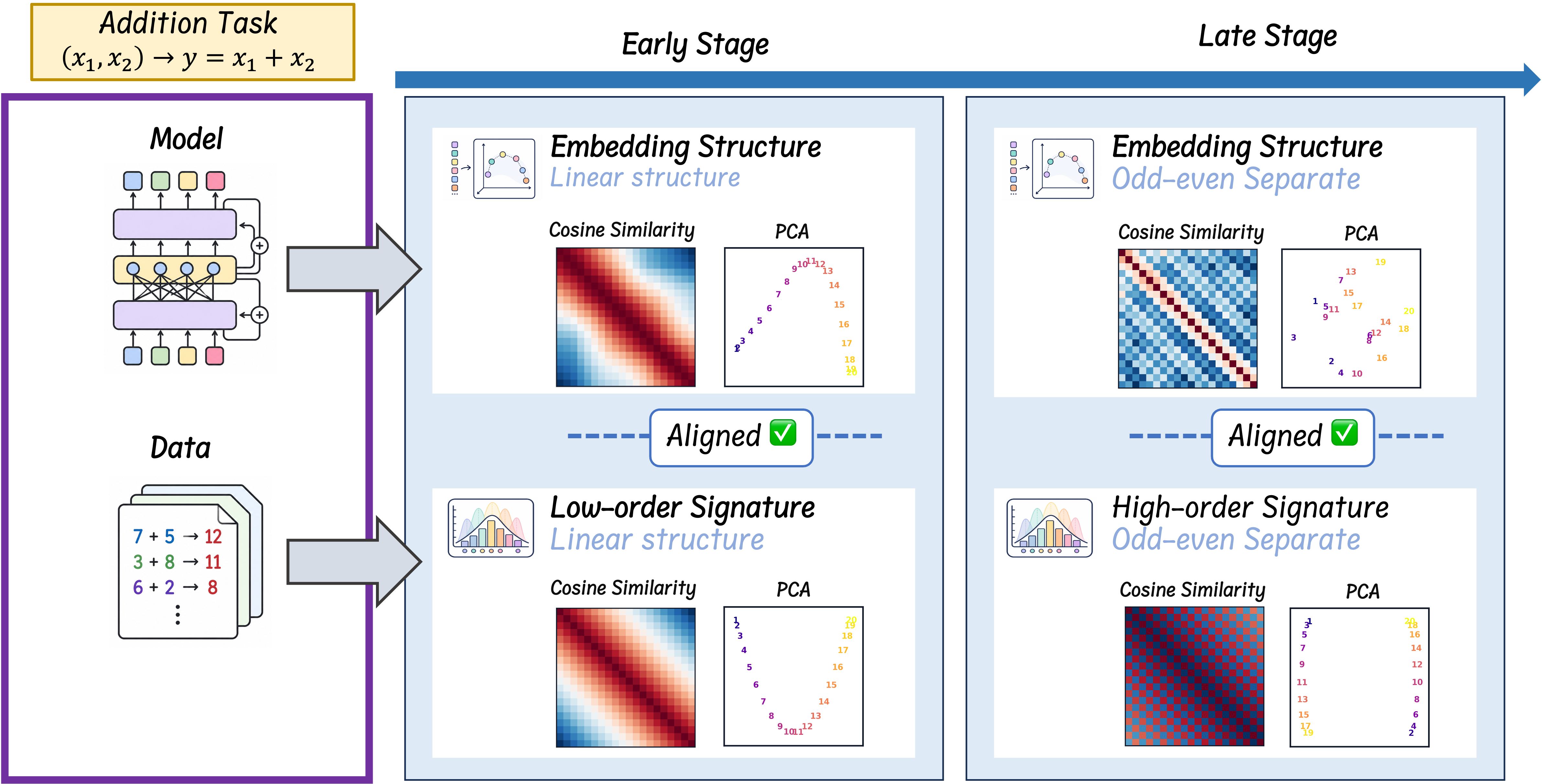}
    \caption{Under small initialization, the embedding geometry evolves progressively during training.}
    \label{fig:overview}
\end{figure}

Finally, we investigate the functional role of these embedding structures and show that they can directly facilitate model training and task learning. Using a simple numerical comparison task, we demonstrate that once the embedding space acquires the structure encoded by the corresponding signature, the model can naturally learn the underlying comparison rule, suggesting that signature-aligned geometry provides a useful inductive basis for solving the task. We further examine signatures derived from real-world language corpora and find that they already exhibit meaningful semantic organization, including structured relations among numbers, cyclic relations among weekdays, and category-level clustering of semantically related nouns. These results suggest that signatures capture statistical structures in language that can serve as a basis for the emergence of semantic organization in the learned embedding space.

Our main contributions are as follows:
\begin{enumerate}
\item We identify a progressive evolution of token-embedding structures, which we term \emph{Context Staircase}: embeddings are closely associated with low-order signatures in the early stage of training, while higher-order signature structures gradually emerge later. This reveals an implicit bias in the space of data statistics, from simpler, low-order relations to increasingly complex, context-dependent ones.

\item We provide a dynamical explanation of this phenomenon through gradient-flow analysis. We show that signatures of different orders explicitly enter the embedding updates, and that small initialization induces their ordered influence over training.

\item We extend the analysis from feed-forward networks to self-attention architectures and further to real language-model training.

\item We investigate the role of the signature-aligned embedding structures, showing that they have important effects on task learning and the incorporation of semantic structure into the embedding space.
\end{enumerate}

\section{Related Work}
\paragraph{Token embeddings and representation geometry.} Early studies of word embeddings focused on constructing static vector representations that encode the distributional properties of words. Neural language models introduced distributed word representations as trainable parameters learned from language modeling objectives~\citep{bengio2003neural}, while word2vec showed that efficient local context-prediction objectives can produce word vectors with strong semantic and syntactic regularities~\citep{mikolov2013distributed}. GloVe further connected word embeddings to global word--word co-occurrence statistics through a matrix-factorization-style objective~\citep{pennington2014glove}, and fastText enriched static embeddings with subword information by representing words as sums of character $n$-gram vectors~\citep{bojanowski2017enriching}. These works established the classical view that embedding geometry is closely tied to contextual distributions in text.

With the rise of contextualized and large language models, subsequent studies shifted from constructing embeddings to analyzing what linguistic and semantic structures are encoded in learned representations. Contextual representations in ELMo, BERT, and GPT-2 were shown to exhibit nontrivial geometric properties such as anisotropy and context dependence~\citep{ethayarajh2019contextual}, and syntactic structures were found to be recoverable from neural word representation spaces through structural probes~\citep{hewitt2019structural}. More recent work has shown that modern large language models organize world knowledge, semantic concepts, and output word embeddings through structured directions or regions in representation space~\citep{gurnee2024language,han2024word,park2025geometry}.

A closely related line of recent work further shows that the geometry of learned representations can depend strongly on initialization. In particular, studies of Transformer models have found that different initialization scales can lead to qualitatively different representation structures and learning behaviors~\citep{zhang2024init,zhang2025complexity,yao2025analysis}. Small initialization tends to produce more structured embedding spaces that more clearly reflect the underlying relations of the task, and these structures have been associated with improved reasoning and compositional generalization. These findings highlight initialization as an important factor in the emergence of embedding structure during training.

However, it remains unclear what data-dependent statistical structures these evolving embeddings correspond to and how such structures emerge progressively through gradient-based training. Our work addresses this question by directly analyzing the training dynamics of token embeddings under small initialization and relating their evolving geometry to a hierarchy of probability signatures derived from the data.

\paragraph{Training dynamics and gradient-flow analyses.} A complementary line of work studies deep learning through the dynamics induced by gradient descent or its continuous-time limit, gradient flow. Early analyses of deep linear networks showed that even linear input--output maps can exhibit highly nontrivial learning dynamics, including plateaus and rapid transitions~\citep{saxe2014exact}. Kernel-based analyses, such as the neural tangent kernel framework, characterize training in the lazy regime where feature representations change only weakly~\citep{jacot2018neural}. More relevant to our setting are recent studies of the feature-learning regime under small initialization. These works show that small initialization can induce qualitatively different training dynamics, and that gradient flow may drive parameters toward structured or low-complexity configurations through phenomena such as condensation~\citep{xu2025overview,zhou2022towards,boursier2022gradient}. Such results suggest that small initialization provides a particularly useful regime for studying how nontrivial representation structures emerge during training.

For Transformer models, the dynamical effects of initialization and feature learning are only beginning to be understood. Mean-field analysis has been used to study gradient scales and training stability in Transformers~\citep{xiong2020layer}, while other works analyze how attention-based architectures acquire in-context learning or token-level dependencies through gradient descent and gradient-flow dynamics~\citep{vonoswald2023transformers,yang2024training,yang2025transformers}. Under small initialization, \citet{yao2025analysis} analyzes the early-stage training dynamics of language models and characterizes the coupled evolution of token embeddings and self-attention during training. More recently, \citet{chen2026focus} develops a stage-wise gradient-flow analysis of attention learning, revealing a multi-stage process in which projection parameters first undergo condensation, attention subsequently becomes activated and develops structured focus. Most closely related to our work, \citet{im2026associate} analyzes how attention-based language models learn token associations from natural language data by approximating early-stage gradients with their leading terms. They derive closed-form characterizations of Transformer weights in terms of corpus-dependent basis functions, including bigram, token-interchangeability, and context mappings, thereby connecting parameter formation with interpretable data statistics. Their analysis, however, focuses primarily on parameters inside Transformer blocks, such as attention- and MLP-related weights, whereas the evolution of the token embedding matrix itself remains largely unexplored. In contrast, we focus on token embeddings under small initialization: starting from their empirically observed geometric evolution, we use gradient-flow analysis to identify the data-dependent probability signatures that enter the embedding updates through different computational paths.

\section{Preliminary}

\subsection{Prediction Problem}

We consider a token prediction problem over a finite vocabulary $\mathcal V$, with vocabulary size $|\mathcal V|=V$. Each sample consists of an input sequence $X$ and a target token $y$:
\begin{equation}\label{eq:task_pre}
X=(x_1,x_2,\ldots,x_L)\in\mathcal V^L,\qquad y\in\mathcal V .
\end{equation}
For each token $\alpha\in\mathcal V$, we use $\ve_\alpha\in\mathbb R^V$ to denote its one-hot vector. Consider a model $\vf: \fV^L\rightarrow \mathbb{R}^V$, which contains an input embedding matrix $\vW^E\in\mathbb R^{d\times V}$. The embedding vector of token $\alpha\in\mathcal V$ is denoted by $\vw^E_\alpha\in\mathbb R^d$.
The model also contains an output unembedding matrix $\vW^U\in\mathbb R^{V\times d}$ and $\vg(X)\in\mathbb R^d$, which satisfies
\begin{equation*}
\vf(X)=\vW^U\vg(X)\in\mathbb R^V .
\end{equation*}
The predicted distribution is $\vp(X)=\operatorname{softmax}(\vf(X))$, and training objective is the standard cross-entropy loss $\ell(X,y)=-\log \vp(X)_y$.

Throughout this work, we study the gradient-flow dynamics of the input embedding vectors:
\begin{equation}\label{eq:grad_all}
\begin{aligned}
\frac{d\vw^E_\alpha}{dt}&=-\mathbb E_{\mathcal D}\left[\frac{\partial \ell(X,y)}{\partial \vw^E_\alpha}\right]=-\mathbb E_{\mathcal D}\left[\left(\frac{\partial \vf(X)}{\partial \vw^E_\alpha}\right)^T\frac{\partial \ell(X,y)}{\partial \vf(X)}\right]\\
&= \mathbb E_{\mathcal D}\left[\left(\frac{\partial \vf(X)}{\partial \vw^E_\alpha}\right)^T\left(\ve_y - \vp(X)\right)\right] = \mathbb E_{\mathcal D}\left[\left(\frac{\partial \vf(X)}{\partial \vw^E_\alpha}\right)^T\ve_y\right]- \mathbb E_{\mathcal D}\left[\left(\frac{\partial \vf(X)}{\partial \vw^E_\alpha}\right)^T\vp(X)\right]\\
& := \left.\frac{d\vw^E_\alpha}{dt}\right|^y - \left.\frac{d\vw^E_\alpha}{dt}\right|^p.
\end{aligned}
\end{equation}
Our goal is to characterize how the data distribution and model components shape the embedding updates during training, and to relate these updates to the empirical evolution of embedding geometry.

\subsection{Token-conditioned Signatures}

In language modeling, the meaning of a token is often characterized by the distributional context in which it appears~\citep{harris1954distributional,firth1957papers,turney2010frequency,lenci2018distributional}. Motivated by this distributional view, we introduce a hierarchy of token-conditioned distributions for our sequence prediction setting~\eqref{eq:task_pre}.
For a token $\alpha\in\mathcal V$, we denote its occurrence times in $X$ by $N_{\alpha}(X)$. Let $I\sim {\rm Unif}([L])$
be sampled independently of the training sample $(X,y)$, and condition on the event $x_I=\alpha$. We define the occurrence probability of $\alpha$ by
\begin{equation*}
    r_\alpha:=\mathbb P(x_I=\alpha)=\frac{1}{L}\mathbb E[N_\alpha(X)].
\end{equation*}
For notational simplicity, throughout this paper we write $\mathbb P(\,\cdot\mid\alpha)$ as shorthand for $\mathbb P(\,\cdot\mid x_I=\alpha)$, and use the same shorthand for conditional expectations. The most basic statistical object associated with $\alpha$ is the conditional label distribution
\begin{equation*}
    \mathbb P\left(y=\nu\mid\alpha\right),\qquad \nu\in\fV.
\end{equation*}
Conditioned on $\alpha$, we remove the occurrence at position $I$ from the sequence and write the remaining context tokens as
\begin{equation*}
    X_{-\alpha}=(x_1,x_2,\ldots,x_{L-1}).
\end{equation*}
For $m\in\{1,\ldots,L-1\}$, let
\begin{equation*}
    \mathcal I_m
    :=
    \left\{(j_1,\ldots,j_m)\in[L-1]^m:
    j_a\neq j_b\ \text{for all }a\neq b\right\},
\end{equation*}
so that $|\mathcal I_m|=(L-1)_m=(L-1)(L-2)\cdots(L-m)$. We sample an ordered tuple of distinct context positions
\begin{equation*}
    (J_1^{(m)},\ldots,J_m^{(m)})\sim {\rm Unif}(\mathcal I_m),
\end{equation*}
independently of $(X,y,I)$. The order-$m$ co-occurrence and label signatures are defined componentwise by
\begin{align*}
    \varphi_\alpha^m(\beta_1,\ldots,\beta_m)
    &:={\mathbb P}\left(
    x_{J_1^{(m)}}=\beta_1,\ldots,x_{J_m^{(m)}}=\beta_m
    \mid \alpha\right),\\
    \varphi_\alpha^{y;m}(\nu,\beta_1,\ldots,\beta_m)
    &:={\mathbb P}\left(
    y=\nu,
    x_{J_1^{(m)}}=\beta_1,\ldots,x_{J_m^{(m)}}=\beta_m
    \mid \alpha\right).
\end{align*}
Equivalently,
\begin{align*}
    \varphi_\alpha^{y;m}(\nu,\beta_1,\ldots,\beta_m)
    =\frac{1}{(L-1)_m}
    \sum_{(j_1,\ldots,j_m)\in\mathcal I_m}
    {\mathbb P}\left(
    y=\nu,x_{j_1}=\beta_1,\ldots,x_{j_m}=\beta_m
    \mid\alpha\right),
\end{align*}
with the analogous expression for $\varphi_\alpha^m$. For $m=0$, we define
\begin{equation*}
    \varphi_\alpha^{y;0}(\nu)
    :=\mathbb P\left(y=\nu\mid\alpha\right).
\end{equation*}
The corresponding tensors are
\begin{align*}
    \vvarphi_\alpha^m
    &: =
    \sum_{\beta_1,\ldots,\beta_m\in\fV}
    \varphi_\alpha^m(\beta_1,\ldots,\beta_m)
    \bigotimes_{i=1}^{m}\ve_{\beta_i},\\
    \vvarphi_\alpha^{y;m}
    &: =
    \sum_{\nu,\beta_1,\ldots,\beta_m\in\fV}
    \varphi_\alpha^{y;m}(\nu,\beta_1,\ldots,\beta_m)
    \ve_\nu\otimes\bigotimes_{i=1}^{m}\ve_{\beta_i}.
\end{align*}
We refer to $\vvarphi_\alpha^m$ and $\vvarphi_\alpha^{y;m}$ as the order-$m$ co-occurrence signature and the order-$m$ label signature of $\alpha$, respectively. Each signature is a normalized probability tensor. Moreover, the hierarchy is marginally consistent: marginalizing any context mode of an order-$m$ signature recovers the corresponding order-$(m-1)$ signature. We formalize this property below and use it both in the small-initialization analysis and in the frequency interpretation of the signature hierarchy.

In this work, we first focus on sequence prediction settings in which the context tokens are exchangeable, so that their statistical role does not depend on their specific positions in the sequence. This abstraction is also consistent with classical permutation-invariant embedding models such as CBOW, where context tokens are aggregated without distinguishing their order. We later relax this assumption in our analysis of self-attention, where different token positions are explicitly distinguished and position-dependent signatures are considered.

\subsection{Parameter Initialization}
Throughout this paper, all parameter matrices are initialized as
\begin{equation*}
W_{ij}\sim\mathcal N(0,d_{\rm in}^{-2\gamma}),
\end{equation*}
where $d_{\rm in}$ is the input dimension of the matrix. Equivalently, each coordinate has standard deviation $d^{-\gamma}$. The hyperparameter $\gamma$ controls the initialization scale: larger $\gamma$ corresponds to smaller initialization. We use the small-initialization regime $\gamma>\frac{1}{2}$ when analyzing the relative scale of different signature contributions at initialization. This regime is closely related to the condensed/nonlinear training regime studied in prior work~\citep{luo2021phase,zhou2022towards}. 

\subsection{Main Result: Context Staircase}
\label{sec:main_result}

With the hierarchy of signatures defined above, we now introduce the central phenomenon studied in this work. We find that \emph{small initialization} induces a characteristic progressive evolution of token embeddings, which we term \emph{Context Staircase}. Under small initialization, the embedding geometry is initially dominated by low-order signatures, especially the zeroth-order label signature. As training proceeds, the influence of higher-order signatures gradually increases, and richer contextual structures become increasingly visible in the embedding space.

Figure~\ref{fig:overview} illustrates this phenomenon in the addition task. 
At the early stage of training, the number embeddings form a smooth ordered structure that closely matches the zeroth-order signature. 
Later in training, an additional odd--even structure emerges, which is contained in the higher-order signature. 
Similar transitions are observed across the controlled tasks studied in the following sections.

The theoretical results of this paper provide a simple explanation for this progressive evolution.

\begin{prop*}[Low-order dominance under small initialization, informal]
Under small initialization, higher-order signatures have a much weaker influence on the embedding dynamics than the zeroth-order signature. Consequently, the embedding space is shaped predominantly by low-order statistical structure at the beginning of training.
\end{prop*}

The intuition is that higher-order signature contributions involve more embedding factors. When the embeddings are initialized at a sufficiently small scale, these additional factors strongly suppress the corresponding contributions. This creates a natural hierarchy among signature orders at the beginning of training.

\begin{prop*}[Emergence of higher-order signatures during training, informal]
As training proceeds and the parameter scale grows, the initial suppression of higher-order signatures gradually disappears. 
Higher-order signatures can then have a substantial influence on the embedding dynamics together with the lower-order signatures. 
In the late-stage regime considered in our analysis, the highest accessible signature order has the strongest influence among them.
\end{prop*}
Together, these two results explain the basic mechanism of \emph{Context Staircase}: small initialization first favors simple, low-order statistical relations, while training progressively allows richer, higher-order contextual statistics to shape the embedding space. 
Importantly, higher-order signatures do not necessarily replace lower-order ones. 
Rather, multiple signature orders can jointly influence the late-stage embedding geometry, with the highest accessible order becoming the most prominent under the conditions of our analysis.

The following sections provide the detailed empirical evidence and formal analysis for this phenomenon. 
We first study Context Staircase in a controlled feed-forward model, and then examine how it extends to self-attention architectures and real language-model training.

\section{Signature-Aligned Dynamics in Feed-Forward Models}
We first study a controlled feed-forward model. This setting allows us to present the empirical signature-alignment phenomenon in its simplest form and then derive a gradient-flow explanation for why the same statistical structures enter the embedding updates.
\subsection{Model Definition}

We first consider the feed-forward network and study how the embedding space develops. Given an input sequence $X=(x_1,\ldots,x_L)$, we define that:
\begin{equation}\label{eq:ffn}
\vf_{\mathrm{ffn}}(X):=\vW^U\vg_{\mathrm{ffn}}(X)=\vW^U\sigma\left(\sum_{i=1}^{L}\vw^E_{x_i}\right),
\end{equation}
where $\sigma$ denotes an element-wise activation. This model can be viewed as a CBOW-type architecture~\citep{mikolov2013efficient}. We use this controlled model first to expose the empirical signature-alignment phenomenon and then to derive the gradient-flow mechanism that can generate it.

\subsection{Empirical Signature Alignment}
We begin with the central empirical observation of this work. In controlled sequence-prediction tasks, the geometry of the learned token embeddings evolves through a hierarchy of token-conditioned signatures. The early embedding space is well explained by the zero-order label signature, whereas the effects of higher-order signatures become jointly visible later in training. We first document this empirical change and then use gradient-flow analysis to explain why different signature orders enter the embedding dynamics and how their relative scales evolve during training.

\paragraph{Alignment with $\vvarphi_\alpha^{y;0}$ in the early stage.}
We first examine how the embedding space is organized during the early
stage of training, when the model has only begun to extract statistical
regularities from the data. Across a series of controlled tasks, we observe
that the emerging embedding structures are not arbitrary, but exhibit a
close and systematic relation to the zeroth-order label signature
$\vvarphi_{\alpha}^{y;0}$. We study this phenomenon from two complementary
perspectives: the embedding structure of an individual token and the
geometric organization of the embedding space across multiple tokens. In
the first class of tasks, we prescribe a classical label distribution for
a specific token $a$ and compare the resulting early-stage embedding of
$a$ with this distribution. In the second class of tasks, we construct
datasets in which the zeroth-order signatures of different tokens exhibit
predefined geometric patterns, and examine whether the corresponding
embeddings develop analogous structures.

We construct the first class of tasks as follow:
\begin{task}\label{task:random_label}
Let $\mathcal{V}$ be a finite vocabulary and let $\mathcal{Y}$ be a finite label space. 
We fix a single anchor token $a \in \mathcal{V}$ and a predefined context-token set $\mathcal{X} \subseteq \mathcal{V}\setminus\{a\}$. 
The input space associated with the anchor token $a$ is defined as
\begin{equation*}
\mathcal{D}_a := \{(a,x_1,x_2): x_1,x_2 \in \mathcal{X}\}.
\end{equation*}
For each sequence $(a,x_1,x_2)\in \mathcal{D}_a$, we assign its target label by independently sampling
\begin{equation*}
y(a,x_1,x_2) \sim p_a,
\end{equation*}
where $p_a\in\Delta(\mathcal{Y})$ is a prescribed label distribution associated with the anchor token $a$. 
\end{task}

We compare $\vvarphi_{a}^{y;0}$ with the distribution induced by the
embedding logits,
${\rm softmax}\left(\vW^{U}\vw^{E}_{a}\right)$, under four prescribed
label distributions: uniform, normal, Poisson, and exponential.

As shown in Figure~\ref{fig:early_stage}A, after a short period of
training, the distribution induced by the embedding logits of token $a$
closely matches its zeroth-order label signature
$\vvarphi_a^{y;0}$. This correspondence is consistently observed across
all four target distributions, as reflected by the high correlation and
small distributional distance between
${\rm softmax}\left(\vW^{U}\vw^{E}_{a}\right)$ and
$\vvarphi_a^{y;0}$ (In Figure~\ref{fig:early_stage}A, $r$ denotes the Pearson correlation coefficient, KL denotes the Kullback--Leibler divergence, and JSD denotes the Jensen--Shannon divergence between the two distributions). These results show that the statistical profile
encoded by the early-stage embedding of token $a$ is closely associated
with its zeroth-order label distribution.

Then we construct the second type of tasks as follow:
\begin{task}\label{task:2token_math}
Let $\mathcal{X}=\{1,\ldots,N\}$ be a finite set of number tokens, 
we consider input sequences of length two,
\begin{equation*}
X=(x_1,x_2), \qquad x_1,x_2\in\mathcal{X}.
\end{equation*}
For each input sequence $(x_1,x_2)$, the target label is determined by an arithmetic rule
\begin{equation*}
y = F(x_1,x_2),
\end{equation*}
where $F$ is chosen from one of the following operations:
\begin{equation*}
F_{\mathrm{add}}(x_1,x_2)=x_1+x_2,\quad
F_{\mathrm{sub}}(x_1,x_2)=|x_1-x_2|,\quad
F_{\mathrm{mul}}(x_1,x_2)=x_1x_2.
\end{equation*}
In addition, we construct the task $F_{\rm mod}(x_1,x_2)=x_1+x_2+1\mod{N}$, which will be used for the analysis in the next section.

\end{task}

These tasks are chosen because they induce distinctive and interpretable label-distribution structures over number tokens, which makes the comparison between embedding geometry and signature geometry more direct.

Let $\vw^{E,(t)}_a$ denote the learned input embedding of token $a$ at epoch $t$. For every pair of number tokens $a,b\in\mathcal{X}$, we compute the pairwise embedding similarity
\begin{equation}\label{eq:S_emb}
S^{\mathrm{emb},(t)}_{a,b}
=
\frac{\langle \vw^{E,(t)}_a,\vw^{E,(t)}_b\rangle}{\|\vw^{E,(t)}_a\|_2\|\vw^{E,(t)}_b\|_2},
\end{equation}
and the pairwise signature similarity
\begin{equation}\label{eq:S_sig_0}
S^{\mathrm{sig},(0)}_{a,b}
=
\frac{\langle \vvarphi^{y;0}_a,\vvarphi^{y;0}_b\rangle}
{\|\vvarphi^{y;0}_a\|_2\|\vvarphi^{y;0}_b\|_2}.
\end{equation}
We denote $S^{{\rm emb},(t)}$ as the cosine similarity matrix of all tokens, and similar with the $S^{{\rm sig},(0)}$. We then quantify the agreement between these two geometries by computing the correlation between their off-diagonal entries:
\begin{equation}\label{eq:rho_0}
\rho^{(t)}_{(0)}
=
\operatorname{Corr}
\left(
\{S^{\mathrm{emb},(t)}_{a,b}:a,b\in\mathcal{X},a<b\},
\{S^{\mathrm{sig},(0)}_{a,b}:a,b\in\mathcal{X},a<b\}
\right).
\end{equation}
A larger value of $\rho$ indicates that the learned embedding geometry is more consistent with the signature geometry.

As shown in Figure~\ref{fig:early_stage} B-D, the $S^{{\rm emb},(t)}$ exhibit highly similar structural patterns to the corresponding $S^{{\rm sig},(0)}$ across all three arithmetic tasks. 
For the addition task, both matrices display a smooth diagonal band structure, reflecting the gradual change of label distributions across nearby number tokens. 
For the absolute-subtraction task, both matrices show a symmetric ``X''-shaped pattern. 

\begin{figure}[htbp]
    \centering
    \includegraphics[width=0.95\linewidth]{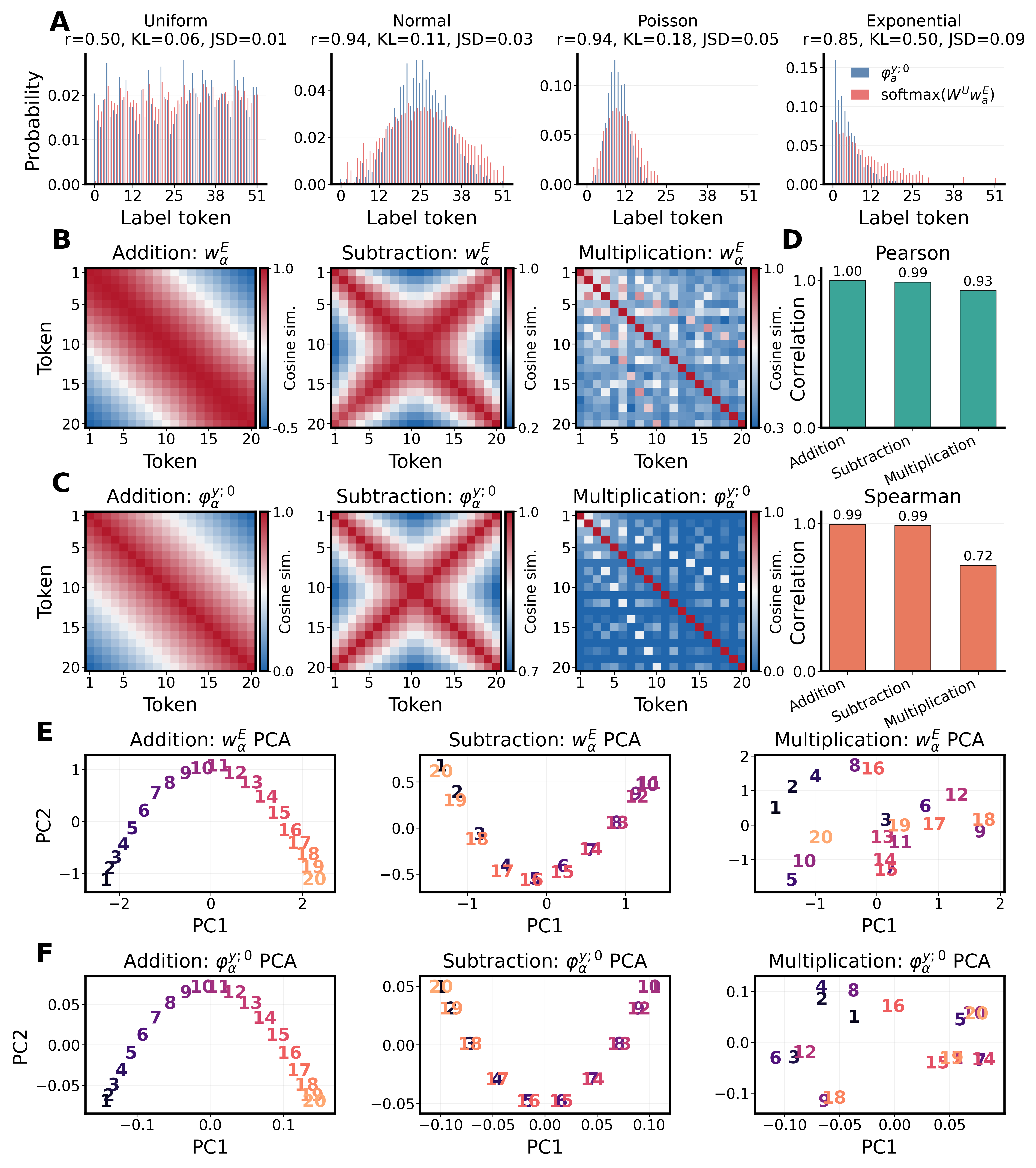}
    \caption{A: Comparison between the $\vvarphi^{y;0}_a$ and $\mathrm{softmax}(\vW^{U}\vw^{E}_{a})$ in Task~\ref{task:random_label} under four target label distributions: uniform, normal, Poisson, and exponential. Here, $r$ denotes the Pearson correlation coefficient, KL denotes the Kullback--Leibler divergence, and JSD denotes the Jensen--Shannon divergence between the two distributions. B: Heatmap of $S^{{\rm emb},(t)}$ for the addition, absolute-subtraction, and multiplication tasks, with $t=100$. C: Heatmap of $S^{{\rm sig}}$ for the three arithmetic tasks. D: Pearson and Spearman correlations between the off-diagonal entries of $S^{{\rm emb},(t)}$ and $S^{{\rm sig}}$ across the three arithmetic tasks, where $t=100$. E: PCA projections of the learned token embeddings for the three arithmetic tasks at epoch 100. F: PCA projections of the corresponding zero-order label signatures for the three arithmetic tasks.
}
    \label{fig:early_stage}
\end{figure}

We further visualize the structures of embedding space and signature space by the lens of PCA projection in Figure~\ref{fig:early_stage} E, F, respectively. The PCA projections provide a more intuitive visualization of the strong correspondence between the two spaces. In the addition and subtraction tasks, both the embedding space and the signature space exhibit the same ordered geometric structure. In the multiplication task, the numbers instead form distinct clusters according to their odd components. For example, $\{1,2,4,8,16\}$ form one cluster; $\{3,6,12\}$ form another; and $\{5,10,20\}$ form a third. These results establish the first part of the empirical phenomenon: during early training, numerical embeddings acquire spatial structures that closely match the geometry of the zero-order label signatures.

\paragraph{Emergence of higher-order signature effects.}
We next examine how the embedding geometry continues to evolve beyond the early stage of training. A natural question is whether the embedding geometry remains confined to the structure associated with the zeroth-order signature throughout training, or whether additional structures emerge as training proceeds. To investigate this question, we
consider two complementary settings: one in which the zeroth-order signature already contains sufficient information for solving the task, and another in which it cannot distinguish different input tokens or capture the structure required by the task. This comparison allows us to examine whether the emergence of additional embedding structures depends
on their necessity for task learning.

To illustrate this process, we compare two arithmetic tasks: $F_{\rm add}\left(x_1,x_2\right)=x_1 + x_2$ and $F_{\rm mod}\left(x_1,x_2\right)=x_1+x_2+1\mod{N}$ in Task~\ref{task:2token_math}. Figure~\ref{fig:dynamics_emb} A,B show how the training loss and accuracy evolve over the training for the two tasks, together with the learned embeddings at the early and late stages of training.
These two tasks differ significantly in their early stage. 
In $F_{\rm add}$, different number tokens induce distinct label distributions. 
Therefore, the embeddings form an ordered geometry.  In contrast, for $F_{\rm mod}$, all number tokens have identical label distributions. 
As a result, the early embeddings are almost aligned in the same direction. 
In this case, the zero-order signature provides almost no useful information for distinguishing different number tokens, and the model fails to learn the task during the early stage. Despite these different early-stage behaviors, both tasks exhibit a clear change in embedding geometry as training continues. At the late stage, the embeddings no longer simply reflect the zero-order signatures, but a more complex structure. In the addition task, the embeddings of different number tokens exhibit an odd-even separated structure in the late stage of training. In the modular-addition task, the embeddings of different number tokens form a modulo-19 ring structure: the successor of token $1$ is $(1+19)\bmod N=20$, and the next token is $(20+19)\bmod N=39$.

To quantify this process, we measure the correlation between $S^{{\rm sig},(0)}$ and $S^{{\rm emb},(t)}$ across $t$ (see \eqref{eq:S_emb},\eqref{eq:S_sig_0} and \eqref{eq:rho_0}). Furthermore, we define that
\begin{equation}\label{eq:S_sig_1}
    S^{\mathrm{sig},(1)}_{a,b}=\frac{\langle \vvarphi^{y;1}_a\vv,\vvarphi^{y;1}_b\vv\rangle}
{\|\vvarphi^{y;1}_a\vv\|_2\|\vvarphi^{y;1}_b\vv\|_2},
\end{equation}
where $\vv \in\sR^{|\fX|}$ is the probe vector which is obtained through optimizing the $\vv$ by maximizing the correlation between $S^{\mathrm{sig},(1)}$ and $S^{{\rm emb},(t)}$ for each $t$. Figure~\ref{fig:dynamics_emb} C exhibits the correlation and MSE between $S^{{\rm emb},(t)}$ and $S^{{\rm sig},(0)}$, $S^{{\rm sig},(1)}$. It's noted that in Modular addition task, $S^{{\rm sig},(0)}$ is the matrix whose all elements equal to 1 and the pearson correlation between $S^{{\rm sig},(0)}$ and $S^{{\rm emb}}$ is invalid. So we measure the correlation here by $1-{\rm MSE}(S^{{\rm sig},(0)},S^{{\rm emb}})$. As shown, $S^{{\rm sig},(0)}$ explains the embedding geometry almost perfectly at the beginning of training, achieving correlation close to $1$ and near-zero MSE. 
As training proceeds, however, the correlation drops substantially. 
In contrast, $S^{\mathrm{sig},(1)}$ becomes a much better predictor of the embedding geometry after the early stage. 
It maintains a high correlation and achieves consistently lower prediction error than $S^{{\rm sig},(0)}$ throughout later training.

\begin{figure}[htb]
    \centering
    \includegraphics[width=1\linewidth]{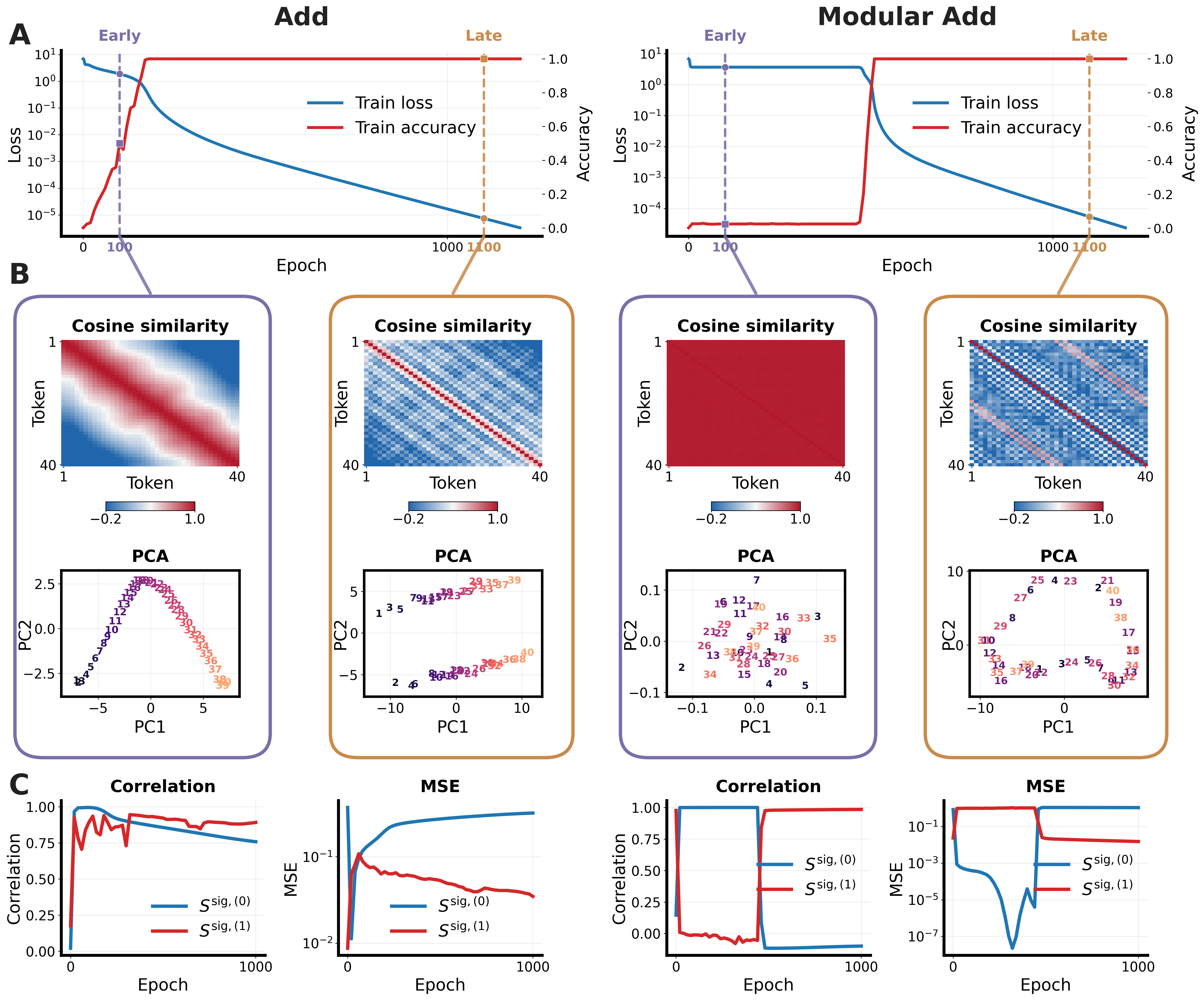}

    \caption{A: Training loss and training accuracy over epochs for the addition task and the modular-addition task. The dashed vertical lines mark the early and late epochs used for visualization. B: Cosine-similarity matrices and PCA projections of the learned token embeddings at the early and late epochs for the addition task and the modular-addition task. C: Temporal evolution of the correlation and MSE between the $S^{\mathrm{emb},(t)}$ and the $S^{\mathrm{sig},(0)}$, as well as the optimized $\widetilde{S}^{\mathrm{sig},(1)}$, for the addition task and the modular-addition task.}
    \label{fig:dynamics_emb}
\end{figure}

These results reveal a progressive change in the statistical structure
reflected by the embeddings. During the early stage of training, the
embedding geometry is closely associated with the zeroth-order signature.
As training proceeds, this association gradually weakens, while structures
related to the first-order signature become increasingly visible. The
late-stage behavior depends on the information contained in the
zeroth-order signature. When it is already informative for the task, part
of the corresponding structure can remain visible in the embedding space.
In contrast, when the zeroth-order signature cannot distinguish different
input tokens, the late-stage embedding geometry exhibits little relation
to it and is instead much more closely associated with the first-order
signature. Since the sequences considered here have length two, the
first-order signature is the only accessible higher-order signature. The
joint effects of multiple higher-order signatures are further examined in
Appendix~\ref{app:length3}.

\subsection{Gradient-Flow Explanation}

We now analyze the gradient flow of the embeddings to explain the evolution of embedding structures observed above, from their early association with the zeroth-order signature to the emergence of higher-order signature effects later in training. We first establish the following assumation
\begin{assumption}\label{assump:polynomial}
    We assume that the activation function $\sigma$ can be approximated by a coordinate-wise polynomial:
    \begin{equation}\label{eq:polynomial}
\sigma(\vz)=\sum_{k=1}^{K}C_k \vz^{\odot k},
\end{equation}
where $\max_k |C_k|<\infty$ and $\vz^{\odot k}$ denotes the coordinate-wise $k$-th power. 
\end{assumption}

We then compute the Jacobian of $\vf_{\rm ffn}(X)$ with respect to the embedding vector $\vw^E_{\alpha}$:
\begin{equation}
    \frac{\partial \vf_{\rm ffn}(X)}{\partial \vw^E_{\alpha}} = N_{\alpha}(X)\vW^{U}{\rm Diag}\left(\sigma^{\prime}\left(\sum_{i=1}^L\vw^E_{x_i}\right)\right).
\end{equation}
By the definition of $I$, for any integrable function $\vH(X,y)$ of the training sample,
\begin{equation*}
\mathbb E\left[N_\alpha(X)\vH(X,y)\right]
=Lr_\alpha\mathbb E\left[\vH(X,y)\mid\alpha\right].
\end{equation*}
Recall~\eqref{eq:grad_all}, then we have that
\begin{equation*}
\begin{aligned}
    \frac{d\vw^E_\alpha}{dt}&=\left.\frac{d\vw^E_\alpha}{dt}\right|^y-\left.\frac{d\vw^E_\alpha}{dt}\right|^p,
\end{aligned}
\end{equation*}
where
\begin{equation*}
\begin{aligned}
&\left.\frac{d\vw^E_\alpha}{dt}\right|^y:=Lr_\alpha\mathbb E\left[\sigma'\left(\sum_{i=1}^L\vw^E_{x_i}\right)\odot \vW^{U,T}\ve_y\mid \alpha\right],\\
&\left.\frac{d\vw^E_\alpha}{dt}\right|^p:=Lr_\alpha\mathbb E\left[\sigma'\left(\sum_{i=1}^L\vw^E_{x_i}\right)\odot \vW^{U,T}\vp(X)\mid \alpha\right].
\end{aligned}
\end{equation*}
For the label-driven term $\left.\frac{d\vw^E_\alpha}{dt}\right|^y$, we have the following result:
\begin{theorem}\label{thm:grad_ffn_y}
    Define that $\fW^{y,m,k}_{\alpha}\in\left(\sR^{d}\right)^{\underbrace{V\times\cdots\times V}_{m+1}}$ is a tensor with order $m+1$, whose elements is a $d$-dimensional vector with
\begin{equation}\label{eq:W}
\fW^{y,m,k}_{\alpha}\left(\nu,\beta_1,\cdots,\beta_m\right)=\binom{L-1}{m}\sum_{\substack{l\geq0,n_1,\cdots,n_m\geq1\\l+\sum_{r=1}^mn_r=k-1}}\frac{(k-1)!}{l!\prod_{r=1}^m n_r!}(\vw^E_\alpha)^{\odot l}\odot\vw^U_{\nu}\odot\bigodot_{r=1}^{m}(\vw^E_{\beta_r})^{\odot n_r}.
\end{equation}
With the Assumption~\ref{assump:polynomial}, we have that
\begin{equation}
    \left.\frac{d\vw^E_\alpha}{dt}\right|^y=Lr_\alpha\sum_{m=0}^{\min\left\{L-1,K-1\right\}} \fA_{\alpha,m}\cdot\vvarphi_{\alpha}^{y;m},
\end{equation}
where $\fA_{\alpha,m}=\sum_{k=m+1}^{K}kC_k\mathcal W^{y,m,k}_{\alpha}$, $C_k$ is the approximation constant in \eqref{eq:polynomial}. Here $\fA_{\alpha,m}\cdot\vvarphi_{\alpha}^{y;m}$ denotes the contraction over the label index and the $m$ context-token indices, resulting in a vector in $\mathbb{R}^d$.
\end{theorem}
Theorem~\ref{thm:grad_ffn_y} provides the central mechanism behind the empirical observations: label signatures from order $0$ up to the largest order appear explicitly in the embedding gradient. The result identifies the statistical components available to drive the embedding updates; the geometric alignment documented above is the corresponding empirical manifestation.

For the prediction-induced term $\left.\frac{d\vw^E_\alpha}{dt}\right|^p$, to have a similar result, we use the Taylor expansion of the softmax around the origin and keep its leading constant term: $\vp(X)=\frac{1}{V}\vone+O(\|\vf(X)\|)$. We have the following result:
\begin{theorem}\label{thm:grad_ffn_p}
    Define that $\mathcal U^{m,k}_{\alpha,0}\in\left(\sR^d\right)^{\underbrace{V\times\cdots\times V}_m}$ is a tensor with order $m$, whose element is a $d$-dimensional vector with
\begin{equation}\label{eq:U}
\mathcal U^{m,k}_{\alpha,0}\left(\beta_1,\ldots,\beta_m\right):=\binom{L-1}{m}\sum_{\substack{l\geq 0,\ n_1,\ldots,n_m\geq 1\\ l+\sum_{r=1}^mn_r=k-1}}\frac{(k-1)!}{l!\prod_{r=1}^mn_r!}\vb_0\odot(\vw^E_\alpha)^{\odot l}\odot\bigodot_{r=1}^m(\vw^E_{\beta_r})^{\odot n_r},
\end{equation} 
where $\vb_0:=\frac{1}{V}\vW^{U,T}\vone$. Under the Assumption~\ref{assump:polynomial} and $\vp(X)=\frac{1}{V}\vone+O(\|\vf(X)\|)$, then we have
\begin{equation}
    \left.\frac{d\vw^E_\alpha}{dt}\right|^p=Lr_\alpha\sum_{m=1}^{\min\left\{L-1,K-1\right\}} \fB_{\alpha,m}\cdot\vvarphi_{\alpha}^{m}+\vR_\alpha,
\end{equation}
where $\fB_{\alpha,m}=\sum_{k=m+1}^{K}kC_k\mathcal U^{m,k}_{\alpha,0}$, $\vR_\alpha$ denotes the remainder term from the softmax expansion.
\end{theorem}
\begin{remark}
In Theorem~\ref{thm:grad_ffn_p}, we only keep the first-order expansion of the softmax probability $p(X)$ around the origin. 
This leads to the $\varphi_\alpha^m$ up to order $\min\{L-1,K-1\}$. 
If $p(X)$ is expanded to higher orders, then higher-order $\varphi_\alpha^m$ will be introduced. For example, if the second-order term in the softmax expansion is retained, the highest order of $\vvarphi_{\alpha}^m$ is $m=\min\left\{L-1,2K-1\right\}$. 
Consequently, for any fixed activation order $K$, higher-order terms in the softmax expansion can in principle introduce $\vvarphi_\alpha^m$ of any order up to $L-1$. 
In Appendix~\ref{app:second_order_softmax}, we provide the explicit form obtained by retaining the second-order term in the softmax expansion.
\end{remark}

\begin{remark}
\label{rem:activation_approx}
Although the above analysis is formulated using a polynomial activation,
this assumption mainly serves to make the contributions associated with
different signature orders explicit. Common activation functions used in
neural networks, including ReLU and other standard nonlinearities, can be
approximated on a bounded neighborhood of the origin by polynomials of
sufficiently high degree. The polynomial degree $K$ can therefore be
chosen much larger than the sequence length $L$, so that all signature
orders accessible to a sequence of length $L$ are included in the
expansion.

Our analysis is concerned with how signatures of different orders enter
and affect the embedding dynamics, rather than with the quantitative
approximation error of the activation function itself. We therefore omit
the approximation residual and use the polynomial representation as an
analytical device throughout the following discussion. Unless otherwise
specified, we henceforth consider the regime $K\gg L$.
\end{remark}

However, in many exchangeable sequence-prediction tasks, the co-occurrence signatures $\vvarphi_\alpha^m$ carry little discriminative information about the token $\alpha$. Intuitively, after conditioning on a selected occurrence of $\alpha$, the remaining context tokens are often sampled from a distribution that does not depend on the identity of $\alpha$. More formally, suppose that the tokens in $X_{-\alpha}$ are sampled independently from position-dependent distributions $\pi_i$, and that each $\pi_i$ is independent of the conditioned token $\alpha$. Then, for any index subset ${j_1,\ldots,j_m}\subseteq [L-1]$, we have
\begin{equation*}
\mathbb P\left(x_{j_1}=\beta_1,\cdots,x_{j_m}=\beta_m\mid \alpha\right)
=
\mathbb P\left(x_{j_1}=\beta_1,\cdots,x_{j_m}=\beta_m\right),
\qquad \beta_1,\ldots,\beta_m\in\fV .
\end{equation*}
Therefore, the $\vvarphi_\alpha^m$ are identical across different $\alpha$, that is, $\vvarphi_{\alpha_1}^m=\vvarphi_{\alpha_2}^m$ for any $\alpha_1,\alpha_2\in\fV$. Such signatures do not separate different tokens. For this reason, we do not further analyze their effect in the main text, and instead focus on the label signatures $\vvarphi_\alpha^{y;m}$. 

\subsection{Order Hierarchy under Small Initialization}
Theorem~\ref{thm:grad_ffn_y} identifies the signature components that can enter the embedding gradient, but it does not by itself determine their relative scales. We next show that small initialization creates a pronounced asymmetry at the beginning of training: higher-order terms contain more small embedding factors and are therefore suppressed relative to the zero-order term.

\begin{prop}\label{prop:initial_zero}
Assume that $L$ and $K$ are fixed and that the model is initialized with scale $d^{-\gamma}$ for some $\gamma>1/2$. Then, for every token $\alpha$ with $r_\alpha=\mathbb{P}\left(x_I=\alpha\right)>0$ and every fixed $m\geq1$,
\begin{equation*}
\frac{
\left(\mathbb E_{\rm init}\left\|\fA_{\alpha,m}\cdot\boldsymbol{\varphi}^{y;m}_\alpha\right\|_2^2\right)^{1/2}
}{
\left(\mathbb E_{\rm init}\left\|\fA_{\alpha,0}\cdot\boldsymbol{\varphi}^{y;0}_\alpha\right\|_2^2\right)^{1/2}
}
=O\!\left(d^{-m\gamma}\right)=o(1).
\end{equation*}

\end{prop}

Proposition~\ref{prop:initial_zero} indicates that in root-mean-square scale, the zero-order label signature gives the leading contribution to the label-driven embedding gradient at initialization. This hierarchy provides a mechanistic explanation for the robust early-stage zero-order alignment observed in Figures~\ref{fig:early_stage} and~\ref{fig:dynamics_emb}.

Proposition~\ref{prop:initial_zero} explains why the zeroth-order signature is most visible at the beginning of training. Under small initialization, higher-order signature terms are strongly suppressed because they involve additional powers of the small parameter scale. As training proceeds and the parameter scales grow, this suppression becomes weaker. Consequently, higher-order signatures can increasingly contribute to the embedding dynamics together with lower-order signatures. Proposition~\ref{prop:late_stage} formalizes this late-stage behavior.

\begin{prop}
\label{prop:late_stage}
Assume that $K$ is sufficiently large relative to $L$, and that $C_K\neq 0$. Let $s_E(t)$ and $s_U(t)$ denote the typical scales of the embedding and unembedding vectors at time $t$. Suppose that $0<s_E(t),s_U(t)<\infty$, that $s_E(t)$ becomes sufficiently large during training, and that the normalized contractions and label signatures are uniformly non-degenerate. Then, in this late-stage regime, all accessible signature-order contributions have the same leading scale:
\begin{equation*}
\left\|\fA_{\alpha,m}(t)\cdot\boldsymbol{\varphi}^{y;m}_\alpha\right\|_2
\asymp
s_U(t)s_E(t)^{K-1},
\qquad 0\leq m\leq L-1,
\end{equation*}
where the comparison constants are independent of $t$. Thus, no signature order is suppressed by an additional power of the parameter scale. Moreover, among these same-scale contributions, the highest-order term is the largest:
\begin{equation*}
\left\|\fA_{\alpha,L-1}(t)\cdot\boldsymbol{\varphi}^{y;L-1}_\alpha\right\|_2
\geq
\left\|\fA_{\alpha,m}(t)\cdot\boldsymbol{\varphi}^{y;m}_\alpha\right\|_2,
\qquad 0\leq m<L-1.
\end{equation*}
\end{prop}

The mechanism can be seen directly from the definition $\fA_{\alpha,m}=\sum_{k=m+1}^{K}kC_k\fW_{\alpha}^{y;m,k}$. For a fixed signature order $m$, the summation starts from degree $m+1$ and ends at the common maximum degree $K$. The degree-$k$ term contains one unembedding factor and $k-1$ embedding factors, and therefore has a typical scale proportional to $s_U(t)s_E(t)^{k-1}$. This simple structure explains the different behaviors at the beginning and later stages of training.

At small initialization, $s_E\ll 1$, so the lowest-degree term in each $\fA_{\alpha,m}$ is the largest. Since the lowest possible degree is $k=m+1$, the order-$m$ contribution starts at the scale $s_U s_E^m$. Thus, increasing the signature order by one introduces one additional factor of the small embedding scale. This produces the hierarchy $m=0>m=1>m=2>\cdots$ in magnitude at initialization, which is precisely the suppression described in Proposition~\ref{prop:initial_zero}.

As training proceeds and $s_E(t)$ becomes large, the opposite end of the sum becomes dominant. For every accessible signature order $m$, the largest degree is now the same, namely $k=K$. Consequently, the dominant term in every $\fA_{\alpha,m}$ contains the same number $K-1$ of embedding factors, and all signature orders therefore share the same leading parameter scale $s_U(t)s_E(t)^{K-1}$. The scale separation between different signature orders thus disappears: higher-order signatures are no longer intrinsically suppressed and can act together with lower-order signatures on the embedding dynamics.

After this common parameter scale is factored out, the remaining difference between signature orders comes from their order-dependent combinatorial prefactors (see~\eqref{eq:W}). For the degree-$K$ contribution, this prefactor contains $B_m(K)=(m+1)!\left\{K-1\atop m+1\right\}$, which grows as $B_m(K)\sim(m+1)^{K-1}$ for fixed $m$ and large $K$. Hence, when $K\gg L$, the highest accessible order $m=L-1$ has the largest prefactor. This explains why the late-stage dynamics are jointly influenced by signatures of different orders, while the highest-order contribution can be the strongest among them.

\subsection{Relations across Signature Orders from a Fourier Frequency Perspective}
\label{sec:frequency}

Proposition~\ref{prop:late_stage} shows that, once the initialization-induced suppression is lifted, all accessible signature-order contributions have the same leading parameter scale and can jointly affect the embedding dynamics. Although the highest-order contribution is the largest among them, this result does not specify how the statistical structures associated with different orders are related. In particular, it leaves open whether different signature orders introduce mutually independent structures or whether their effects follow an intrinsic organization. We next examine this question from a frequency perspective.

Signatures of different orders are not independent statistical objects. By Lemma~\ref{lem:signature_consistency}, an order-$(m-1)$ signature is obtained by marginalizing any context mode of an order-$m$ signature. For a fixed token $\alpha$, the first two label signatures satisfy
\begin{equation*}
\varphi_{\alpha}^{y;0}(y)=\mathbb{P}\left(Y=y\mid \alpha\right),\qquad
\varphi_{\alpha}^{y;1}(y,x)=\mathbb{P}\left(Y=y,x_{J_1^{(1)}}=x\mid \alpha\right),
\end{equation*}
and hence
\begin{equation*}
\varphi_{\alpha}^{y;0}(y)=\sum_{x\in\fV}\varphi_{\alpha}^{y;1}(y,x).
\end{equation*}
Denote the real Fourier basis by $\vc_k^{\cos}=\left(\cos\frac{2\pi k j}{p}\right)_{j=0}^{p-1}$, where $k=0,1,\ldots,p/2$ is the frequency and $p=|\fV|$ in our settings. Since $\vc_0^{\cos}=\vone$, the marginalization above can be written as
\begin{equation*}
\vvarphi_{\alpha}^{y;0}=\vvarphi_{\alpha}^{y;1}\vc_0^{\cos}.
\end{equation*}
More generally, contracting any context mode of an order-$m$ signature with $\vc_0^{\cos}$ gives the corresponding order-$(m-1)$ signature. In particular, contracting the last context mode yields
\begin{equation*}
\vvarphi_{\alpha}^{y;m-1}=\vvarphi_{\alpha}^{y;m}\vc_0^{\cos},
\qquad m=1,2,\ldots,L-1.
\end{equation*}
Thus, lower-order signatures are already contained in the zero-frequency projections of higher-order signatures, while the nonzero-frequency components encode the additional structures available at higher orders. The effects of different signature orders therefore need not be interpreted as the appearance of unrelated structures.

This relation suggests a possible temporal organization of the empirical dynamics: the embedding may first reflect the zero-frequency structure shared with lower-order signatures and later incorporate additional task-dependent nonzero-frequency components. This temporal statement is not implied by Proposition~\ref{prop:late_stage}; we test it empirically by analyzing the frequency components of the first-order signature in $F_{\rm add}$ and $F_{\rm mod}$.
For each number token $\alpha$ and each frequency $k$, we define
\begin{equation*}
\vh^{(1)}_{\alpha,k}=\vvarphi^{y;1}_\alpha \vc_k^{\rm cos},\qquad k=0,1,\ldots,p/2.
\end{equation*}
We then compute the pairwise cosine-similarity matrix induced by this frequency component:
\begin{equation*}
S^{\mathrm{sig},(1)}_{k;a,b}=\frac{\langle \vh^{(1)}_{a,k},\vh^{(1)}_{b,k}\rangle}{\|\vh^{(1)}_{a,k}\|_2\|\vh^{(1)}_{b,k}\|_2},\qquad a,b\in\mathcal{X}.
\end{equation*}
Finally, for each frequency $k$, we measure the agreement
\begin{equation*}
\rho_{(1),k}^{(t)}=\operatorname{Corr}\left(\{S^{\mathrm{sig},(1)}_{k;a,b}:a<b\},\{S^{{\rm emb},(t)}_{a,b}:a<b\}\right).
\end{equation*}

We track $\rho_{(1),k}^{(t)}$ over training in $F_{\rm add}$ and $F_{\rm mod}$. Figure~\ref{fig:fourier_frequency_epoch_heatmap} A, B reveal a clear frequency-wise transition. At the earliest stage, the embedding similarity structure is most strongly correlated with the $k=0$ component and several low-frequency components. As training proceeds, the correlation with the $k=0$ component decreases, while correlations with nonzero components increase, including the task-specific components at $k=20$ in $F_{\rm add}$ and $k=19$ in $F_{\rm mod}$.

\begin{figure}[htb]
    \centering
    \includegraphics[width=1\linewidth]{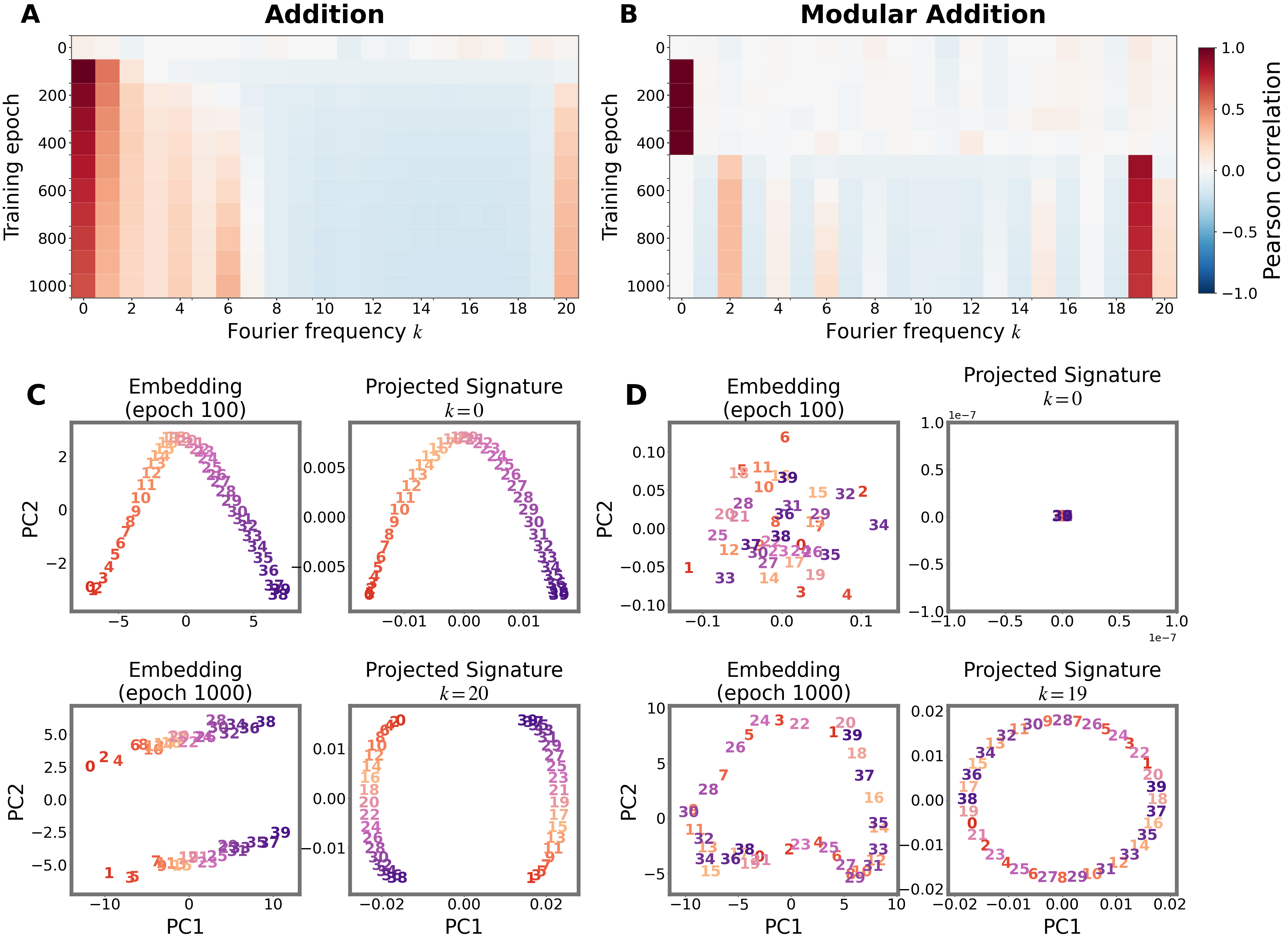}
    \caption{A: Heatmap of $\rho_{(1),k}^{(t)}$ across training epochs $t$ and Fourier frequencies $k$ for the addition task. B: The same correlation heatmap for the modular-addition task. C: PCA projections for the addition task, showing $\vw^E_\alpha$ at epochs 100 and 1000 together with $\vh^{(1)}_{\alpha,k}$ at frequencies $k=0$ and $k=20$, respectively. D: PCA projections for the modular-addition task, showing $\vw^E_\alpha$ at epochs 100 and 1000 together with $\vh^{(1)}_{\alpha,k}$ at frequencies $k=0$ and $k=19$, respectively.}
    \label{fig:fourier_frequency_epoch_heatmap}
\end{figure}

Figure~\ref{fig:fourier_frequency_epoch_heatmap} C, D further compare the number embeddings with the corresponding $\vh^{(1)}_{\alpha,k}$ at the early checkpoint (epoch 100) and the final checkpoint (epoch 1000). Early in training, the embedding space closely resembles the zero-frequency projection of the first-order signature. Later, the nonzero-frequency components also become strongly expressed and closely match the more complex embedding geometry. These results indicate that the effects of different signature orders do not emerge as unrelated structures. Instead, the empirical transition from lower-order to higher-order signature effects can be interpreted as a progressive spectral refinement from the shared zero-frequency structure to additional task-specific nonzero-frequency components.

A particularly interesting phenomenon is that changing only the random seed can lead to substantially different embedding geometries, even when the dataset and training configuration remain unchanged. We find that this variation can be explained by the embeddings aligning with different frequency components of the signatures. The corresponding results are presented in Appendix~\ref{app:seed_embedding}.

\subsection{A Possible Benefit of Context Staircase.}
Context Staircase suggests a particular learning order induced by small initialization. At the beginning of training, low-order signatures have a stronger effect on the embedding dynamics, while higher-order signatures become increasingly important as training proceeds. The zeroth-order signature only describes the relation between a token and the label. The first-order signature additionally includes one context token, and each higher order incorporates one more context token. Therefore, Context Staircase gradually introduces statistical relations involving increasingly more contextual information, from simple token--label relations to more complex relations involving the full context.

Importantly, this ordering does not imply faster learning in every task. Its effect depends on whether the low-order signatures already contain information useful for the task. For example, in the modular-addition task, all number tokens have the same zeroth-order signature. The statistical structure emphasized at the beginning of training therefore cannot distinguish different tokens or support the solution of the task. Learning only begins to make substantial progress when higher-order signatures become influential. In such a case, small initialization can actually delay task learning at the early stage.

The possible benefit of small initialization should therefore be understood as imposing an ordered rather than an immediately task-optimal learning process. Instead of allowing statistical relations of different orders to affect the embeddings at comparable scales from the beginning, it induces Context Staircase, in which simpler relations are learned first and more context-dependent relations are incorporated later. This simple-to-complex ordering may provide a useful inductive bias for representation learning, although whether it accelerates or delays optimization depends on which signature orders contain the information required by the task.

This perspective also suggests an analogy between Context Staircase and the frequency principle of neural networks. The frequency principle states that neural networks tend to learn low-frequency components of a target function before higher-frequency ones during training~\citep{xu2020frequency,xu2019training,xu2024overview,rahaman2019spectral}. Although the notion of complexity is different in the two settings, the learning order is conceptually similar. The frequency principle orders structures by their frequency, whereas Context Staircase orders statistical relations by the amount of contextual information they involve. Both phenomena therefore reflect an implicit preference of neural-network training dynamics for learning relatively simple structures before more complex ones. From this perspective, Context Staircase may be viewed as an analogous implicit bias in the space of contextual statistics, providing a possible statistical counterpart to the low-to-high frequency learning behavior observed in neural networks.

\subsection{Activation Order Controls the Accessible Signature Order} 

Theorem~\ref{thm:grad_ffn_y} shows that an activation of order $K$ exposes label signatures only up to order $\min\{L-1,K-1\}$. This motivates a direct architectural test: in a construction where lower-order signatures cannot distinguish the number tokens, successful learning should emerge only when the activation can expose the required full-context signature. We consider length-$L$ $F_{\rm mod}$ tasks, where lower-order signatures are identical across different number tokens and only the highest-order signature $\vvarphi^{y;L-1}_\alpha$ can distinguish them. In this construction, the gradient-flow decomposition therefore predicts a sharp change when the activation order reaches the sequence length. To test this prediction, we employ the activation $\sigma(\vz)=\vz^{\odot K}$ and vary both the sequence length $L$ and the activation order $K$. For each configuration, we record the maximum training accuracy and visualize the cosine-similarity matrix of the learned token embeddings.
As shown in Figure~\ref{fig:length_order}, when the activation order is smaller than the sequence length, the model is unable to learn the task at all; correspondingly, the embeddings of different numerical tokens collapse into the same direction.
In contrast, once the activation order reaches or exceeds the sequence length, the model can successfully learn the task, and the embedding similarity matrix develops a clear nontrivial structure, consistent with the signature orders accessible in the gradient-flow decomposition. 

\begin{figure}[htb]
    \centering
    \includegraphics[width=1\linewidth]{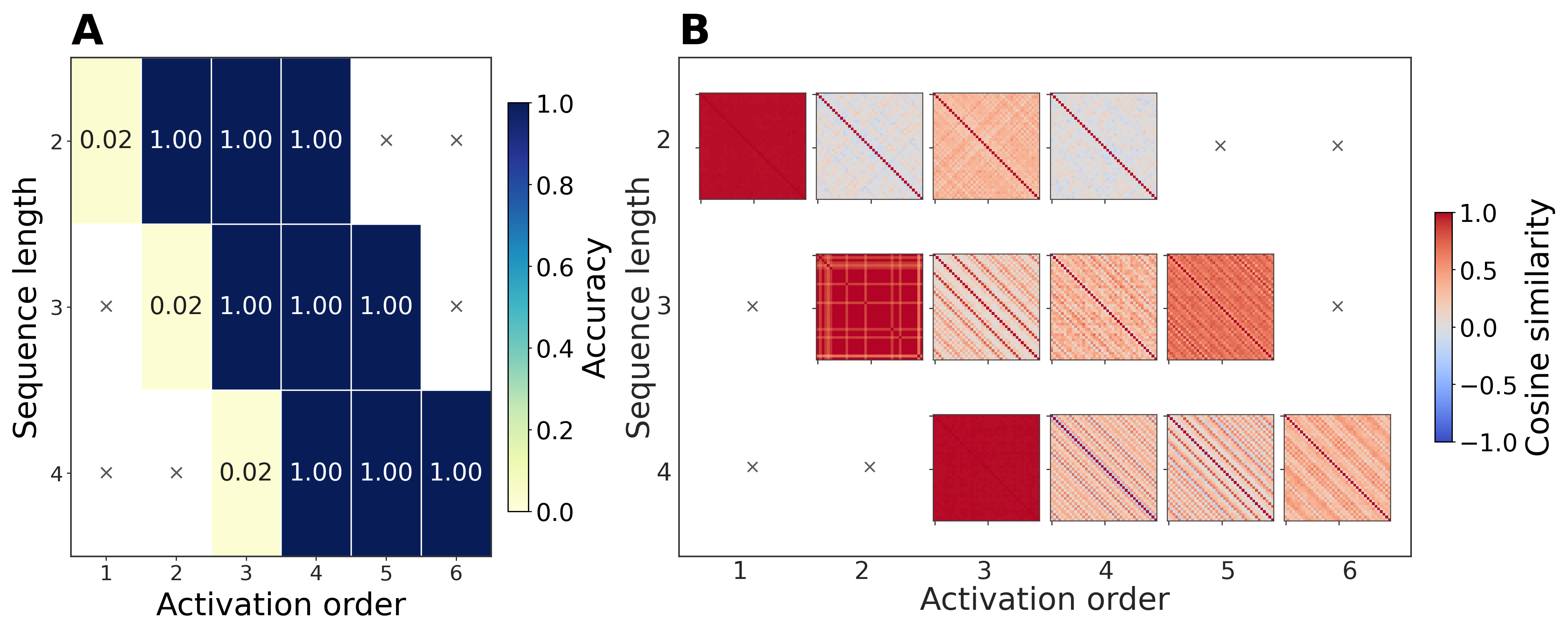}
    \caption{A: Heatmap of the maximum training accuracy achieved for modular-addition tasks with different sequence lengths and activation orders. B: Cosine-similarity matrices of the learned token embeddings for the corresponding task settings with different sequence lengths and activation orders.}
    \label{fig:length_order}
\end{figure}

\section{Attention-Modulated Signature Dynamics}
The feed-forward experiments establish that embedding geometry can progressively align with a hierarchy of signatures, while the gradient-flow decomposition explains why those signatures enter the updates. We next ask how self-attention changes this picture. The central difference is that attention routes statistical information non-uniformly: the contribution of a token depends not only on its contextual distribution, but also on the attention mass assigned to it during prediction. 

\subsection{Model Definition}
Recall that the input embedding matrix for a sequence $X=(x_1,\ldots,x_L)$ is
\begin{equation*}
\vW^E_X=[\vw^E_{x_1},\vw^E_{x_2},\ldots,\vw^E_{x_L}]\in\mathbb R^{d\times L}.
\end{equation*}
Denote that $\vW^Q,\vW^K,\vW^V\in\sR^{d_k}\times d$ and $\vW^O\in\sR^{d\times d_k}$. For each position $s$, define the query, key, and value vectors as
\begin{equation}\label{eq:qkv}
\vq_s:=\vW^Q\vw^E_{x_s},\qquad \vk_s:=\vW^K\vw^E_{x_s},\qquad \vv_s:=\vW^V\vw^E_{x_s}.
\end{equation}
We consider that the model utilizes the last position to output the prediction. Denote that $\vW^{QK}:=\vW^{K,T}\vW^Q$ and $\vW^{OV}:=\vW^O\vW^V$, then the attention logit from the last position to position $s$ is $z_s=\frac{(\vw^E_{x_s})^T\vW^{QK}\vw^E_{x_L}}{\sqrt d}$.
The last-row attention weights are
\begin{equation*}
a_s:=\frac{\exp(z_s)}{\sum_{r=1}^{L}\exp(z_r)},\qquad s=1,\ldots,L.
\end{equation*}
Then the attention model $\vf_{\rm attn}$ is
\begin{equation}\label{eq:attn}
\begin{aligned}
&\vg_{\mathrm{attn}}(X):=\vW^O\sum_{s=1}^{L}a_s\vv_s:=\vW^{OV}\sum_{s=1}^{L}a_s\vw^E_{x_s},\\
&\vf_{\rm attn}(X)=\vW^U\vg_{\mathrm{attn}}(X).
\end{aligned}
\end{equation}

\subsection{Gradient-flow Analysis}
We now analyze how the attention mechanism affects the embedding dynamics. By~\eqref{eq:task_pre}, the embedding gradient can be decomposed into a label-driven term and prediction-induced term:
\begin{equation*}
    \frac{d\vw^E_{\alpha}}{dt} =\left.\frac{d\vw^E_{\alpha}}{dt}\right|^{y} - \left.\frac{d\vw^E_{\alpha}}{dt}\right|^{p}.
\end{equation*}
The previous section shows that the label-driven term is the main source through which label signatures enter the embedding dynamics. We therefore focus on $\left.\frac{d\vw^E_{\alpha}}{dt}\right|^{y}$ for the attention model. \eqref{eq:qkv} indicates that in $\vf_{\rm attn}$, the gradient with respect to the embedding $\vw^E_\alpha$ arises through three components: the value vector, query vector and the key vector. Therefore, we decompose that
\begin{equation*}
    \left.\frac{d\vw^E_\alpha}{dt}\right|^y = \left.\frac{d\vw^E_\alpha}{dt}\right|^y_{\rm value} + \left.\frac{d\vw^E_\alpha}{dt}\right|^y_{\rm query} + \left.\frac{d\vw^E_\alpha}{dt}\right|^y_{\rm key}.
\end{equation*}

 As will be seen from the gradient-flow decomposition below, the attention model still expresses the embedding dynamics as a combination of label signatures. However, there are two significant differences between the $\vf_{\rm attn}$ and the $\vf_{\rm ffn}$ model. First, self-attention is position-sensitive. In $\vf_{\rm attn}$, distinct context positions play different roles in the computation. The signatures defined above condition on a sampled occurrence of $\alpha$, but aggregate the remaining context positions without distinguishing their locations. Here, however, we need to consider the distributions of these tokens at different positions. For any positions $s_1,s_2\in [L]$, we define $\vvarphi_{\alpha}^{y;x_{s_1}}\in\sR^{V\times V}$ and $\vvarphi_{\alpha}^{y;x_{s_1},x_{s_2}}\in\sR^{V\times V\times V}$ by
\begin{equation}
\begin{aligned}
&\vvarphi^{y;x_{s_1}}_\alpha(\nu,\beta_{s_1})
:=\mathbb P(y=\nu,x_{s_1}=\beta_{s_1}\mid \alpha),\\
&\vvarphi^{y;x_{s_1},x_{s_2}}_\alpha(\nu,\beta_{s_1},\beta_{s_2})
:=\mathbb P(y=\nu,x_{s_1}=\beta_{s_1},x_{s_2}=\beta_{s_2}\mid \alpha),\\
\end{aligned}
\end{equation}
where $\nu,\beta_{s_1},\beta_{s_2}\in\fV$. Moreover, the position of $\alpha$ itself should also be taken into account, especially whether it appears at the last position of the sequence. For any positions $s_0,s_1,s_2\in [L]$, we define $\vvarphi_{\alpha,x_{s_0}}^{y;x_{s_1}}\in\sR^{V\times V}$ and $\vvarphi_{\alpha,x_{s_0}}^{y;x_{s_1},x_{s_2}}\in\sR^{V\times V\times V}$ by 
\begin{equation}
\begin{aligned}
&\vvarphi_{\alpha,x_{s_0}}^{y;x_{s_1}}(\nu,\beta_{s_1})
:=\mathbb P\left(y=\nu,x_{s_1}=\beta_{s_1}\mid x_{s_0}=\alpha\right),\\
&\vvarphi_{\alpha,x_{s_0}}^{y;x_{s_1},x_{s_2}}(\nu,\beta_{s_1},\beta_{s_2})
:=\mathbb P\left(y=\nu,x_{s_1}=\beta_{s_1},x_{s_2}=\beta_{s_2}\mid x_{s_0}=\alpha\right),
\end{aligned}
\end{equation}
where $\nu,\beta_{s_1},\beta_{s_2}\in\fV$. These quantities record the label and contextual distributions associated with $\alpha$ while keeping track of the positions that enter the attention computation.

Second, attention does not transmit these signatures uniformly. Instead, the contribution of each token is multiplied by its attention score. Under the occurrence-conditioned construction above, $a_I$ is the attention score assigned to the selected occurrence of $\alpha$. We also define the total attention mass received by token $\alpha$ in $X$ as
\begin{equation*}
A_{\alpha}(X):=\sum_{s=1}^{L}a_s\mathbf 1\{x_s=\alpha\}.
\end{equation*}
This quantity measures how much attention is routed through all occurrences of $\alpha$ in the sequence. Correspondingly, for any positions $s_1,s_2\in [L-1]$, we introduce the following attention-weight factors:
\begin{equation}
\begin{aligned}
&\vpsi_{\alpha}^{y}\left(\nu\right)
:= \mathbb E[a_I\mid y=\nu,\alpha],\\
&\vpsi^{y,x_L}_{\alpha}(\nu,\beta_L)
:= \mathbb E\left[a_I\mid y=\nu,x_L=\beta_L,\alpha\right],\\
&\vpsi^{y,x_L,x_{s_1}}_{\alpha}(\nu,\beta_L,\beta_{s_1})
:= \mathbb E\left[a_Ia_{s_1}\mid y=\nu,x_L=\beta_L,x_{s_1}=\beta_{s_1},\alpha\right],\\
&\vpsi_{\alpha,x_L}^{y,\rightarrow x_{s_1}}\left(\nu,\beta_{s_1}\right)
:= \mathbb E \left[a_{s_1}\mid y=\nu, x_{s_1}=\beta_{s_1}, x_L=\alpha\right],\\
&\vpsi_{\alpha,x_L}^{y,\rightarrow x_{s_1},x_{s_2}}\left(\nu,\beta_{s_1},\beta_{s_2}\right)
:= \mathbb E\left[a_{s_1}a_{s_2}\mid y=\nu,x_{s_1}=\beta_{s_1},x_{s_2}=\beta_{s_2},x_L=\alpha\right].
\end{aligned}
\end{equation}
With these definitions, the attention-induced embedding dynamics can be interpreted as encoding position-specific signatures modulated by attention-dependent weights. Formally, we have the following result
\begin{theorem}\label{thm:attention}
Consider the attention model $\vf_{\rm attn}$ under the last-token prediction setting. For any token $\alpha\in\fV$, the label-driven component of the embedding gradient can be decomposed into value, key, and query contributions:
\begin{equation*}
\left.\frac{d\vw^E_\alpha}{dt}\right|^y_{\rm attn}
=
\left.\frac{d\vw^E_\alpha}{dt}\right|^y_{\rm value}
+
\left.\frac{d\vw^E_\alpha}{dt}\right|^y_{\rm key}
+
\left.\frac{d\vw^E_\alpha}{dt}\right|^y_{\rm query}.
\end{equation*}
More explicitly, these three terms are given by
\begin{equation}
\begin{aligned}
\left.\frac{d\vw^E_\alpha}{dt}\right|^y_{\rm value}
&=
Lr_\alpha\left(\vW^U\vW^{OV}\right)^{T}
\left(
\vvarphi^{y;0}_\alpha\odot\vpsi^y_{\alpha}
\right),
\\
\left.\frac{d\vw^E_\alpha}{dt}\right|^y_{\rm key}
&=
\frac{Lr_{\alpha}}{\sqrt{d}}\vW^{QK}
\left[
\widetilde{\vW}^{\rm key}_\alpha\cdot
\left(
\vpsi^{y,x_L}_{\alpha}\odot\vvarphi^{y,x_L}_\alpha
\right)
-\widetilde{\vW}^{{\rm key}}\cdot\sum_{s_1=1}^{L}
\vpsi^{y,x_L,x_{s_1}}_{\alpha}\odot\vvarphi^{y,x_L,x_{s_1}}_\alpha
\right],
\\
\left.\frac{d\vw^E_\alpha}{dt}\right|^y_{\rm query}
&=
\frac{r_{\alpha}^{x_L}}{\sqrt{d}}\vW^{QK}
\left[
\widetilde{\vW}^{{\rm query},1}\cdot\sum_{s_1=1}^L
\vpsi_{\alpha,x_L}^{y,\rightarrow x_{s_1}}\odot\vvarphi_{\alpha,x_L}^{y;x_{s_1}}
-\widetilde{\vW}^{{\rm query},2}\cdot\sum_{s_1,s_2=1}^L
\vpsi_{\alpha,x_L}^{y,\rightarrow x_{s_1},x_{s_2}}\odot\vvarphi_{\alpha,x_L}^{y;x_{s_1},x_{s_2}}
\right].
\end{aligned}
\end{equation}
Here $\odot$ denotes element-wise multiplication. The coefficient tensors are defined as follows. For the key contribution, $\widetilde{\vW}^{\rm key}_\alpha\in\left(\sR^{d}\right)^{V\times V}$ and $\widetilde{\vW}^{{\rm key}}\in\left(\sR^{d}\right)^{V\times V\times V}$ with element
\begin{equation*}
\begin{aligned}
\widetilde{\vW}^{\rm key}_{\alpha}\left(\nu,\beta_1\right)
&=
\left((\vw^U_{\nu})^{T}\vW^{OV}\vw^E_\alpha\right)\vw^E_{\beta_1},
\\
\widetilde{\vW}^{{\rm key}}\left(\nu,\beta_1,\beta_2\right)
&=
\left((\vw^U_{\nu})^{T}\vW^{OV}\vw^E_{\beta_1}\right)\vw^E_{\beta_2},
\end{aligned}
\end{equation*}
where $\nu,\beta_1,\beta_2\in\fV$. Similarly, for the query contribution, $\widetilde{\vW}^{{\rm query},1}\in\left(\sR^{d}\right)^{V\times V}$ and $\widetilde{\vW}^{{\rm query},2}\in\left(\sR^{d}\right)^{V\times V\times V}$ with element
\begin{equation*}
\begin{aligned}
\widetilde{\vW}^{{\rm query},1}\left(\nu,\beta_1\right)
&=
\left((\vw^U_\nu)^T\vW^{OV}\vw^E_{\beta_1}\right)\vw^E_{\beta_1},
\\
\widetilde{\vW}^{{\rm query},2}\left(\nu,\beta_1,\beta_2\right)
&=
\left((\vw^U_\nu)^T\vW^{OV}\vw^E_{\beta_1}\right)\vw^E_{\beta_2}.
\end{aligned}
\end{equation*}
\end{theorem}

Theorem~\ref{thm:attention} indicates that the main effect of self-attention is to modulate the signature terms by attention scores. The embedding dynamics are still expressed as a combination of signatures, but each signature is now weighted according to how strongly the corresponding tokens are attended to.

It is worth noting that Theorem~\ref{thm:attention} keeps the attention-score term in its full form, so that the role of attention in the embedding dynamics remains explicit and interpretable. As a consequence, the resulting expression involves label signatures of at most second order. Higher-order and more complex signatures could be obtained by further expanding the attention-score term, in a manner similar to the derivation of Theorem~\ref{thm:grad_ffn_y}. We do not pursue this direction here, as our main goal is to characterize the direct effect of attention scores on the embedding dynamics.


\subsection{Empirical Consequences of Attention Modulation}
We examine two empirical consequences of the attention decomposition. First, we test whether the shift from early zero-order dominance to later higher-order influence observed in the feed-forward model also appears in an attention-only model. Second, we isolate the role of attention scores in controlling the rate and magnitude of embedding evolution.

To examine the first perspective, we again use the length-2 tasks $F_{\rm add}$ and $F_{\rm mod}$. In these tasks, the attention scores assigned to the two number tokens are uniform, so the embedding dynamics are primarily governed by the label signatures. Figure~\ref{fig:atten_task}A--H shows the embedding structures of the two tasks at the early and late stages of training, revealing a phenomenon similar to that observed in the FFN analysis. Specifically, the embedding space is first reshaped by the zero-order label signatures in the early stage, while the first-order signature effect becomes clearly visible later in training. These results show that the progressive signature-alignment phenomenon is not specific to the feed-forward model: an attention-only architecture also exhibits early zero-order dominance followed by the emergence of higher-order structure.

To investigate the effect of the attention score,
 we construct the following tasks. 
 \begin{task}
     Given anchor set $\fA$ and noise set $\fR$ with $\fA\cap\fR=\emptyset$, the sequence $X$ is formulated as 
     \begin{equation*}
         X=\left[r_1,\ldots r_{m-1},a,r_{m+1}\ldots,r_L\right],
     \end{equation*}
     where $r_i\in\fR$ and $a\in\fA$. the label $y=0$ if $a \mod{2} = 0$ else $1$.
 \end{task}
 
In this task, we set $\fA=\{11,\dots,20\}$ and $\fR=\{21,\dots,60\}$. At initialization, the attention scores assigned to different tokens are approximately uniform, so the embedding dynamics are primarily driven toward the corresponding signature directions. During training, however, the attention scores assigned to the $r$ tokens rapidly decay to nearly zero. Consequently, the associated gradient magnitudes become negligible, and the embeddings of the $r$ tokens exhibit little further evolution. Figure~\ref{fig:atten_task} I,J exhibits the structure of the anchor token $a\in\fA$ and the noise token $r\in\fR$. Anchor token $a$ are divided into two group by the parity, and the noise token are aligned in the one direction, which is consistent with their label signature. The key difference is that the rate of structure formation is controlled by attention. Figure~\ref{fig:atten_task} K displays the average received attention score to the anchor tokens and the noise tokens during training, indicating that the anchor tokens $a$ receive substantially larger average attention scores than the noise tokens $r$. Correspondingly, their embedding geometry changes much more rapidly, as reflected by the faster increase of embedding norm over training in Figure~\ref{fig:atten_task} L. In contrast, the noise tokens receive much smaller attention scores, and their embeddings evolve more slowly. 

\begin{figure}[htb]
    \centering
    \includegraphics[width=1\linewidth]{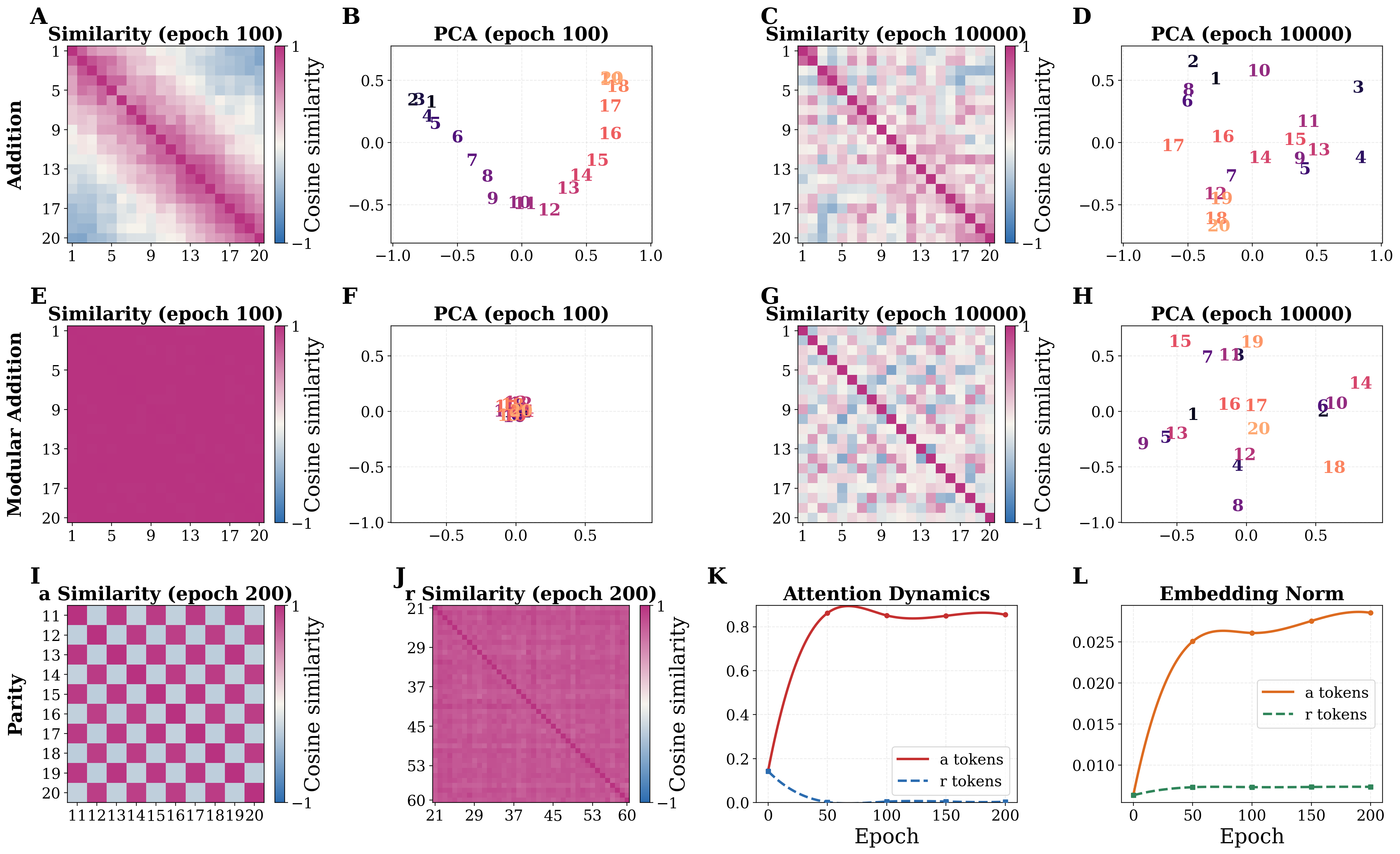}
    
    \caption{A--H: Embedding cosine-similarity matrices and PCA projections for the addition and modular-addition tasks at early and late training stages. A,B show early-stage addition; C,D show late-stage addition; E,F show early-stage modular addition; G,H show late-stage modular addition. I,J: Embedding cosine-similarity matrices of anchor tokens and noise tokens in the parity task. K: Mean attention dynamics for anchor tokens and noise tokens. L: Embedding norm dynamics for anchor tokens and noise tokens.}
    \label{fig:atten_task}
\end{figure}

\section{Path-Specific Signatures in Transformer Training}
The controlled analyses above identify two complementary mechanisms: nonlinear feed-forward paths expose signatures of different orders, while attention paths modulate position-specific signatures through learned routing weights. We now use these results to formulate path-specific hypotheses for a full decoder-only Transformer and evaluate them on a real language-model training run.

Consider a decoder-only Transformer with $N_{\mathrm{layer}}$ layers and sequence length $L$. Let $s\in\{1,\ldots,L-1\}$ denote a prediction position, so that the target token is $x_{s+1}$. We write the final residual stream at position $s$ as
\begin{equation*}
\vh^{N_{\mathrm{layer}}}_s=\vh^0_s+\sum_{l=0}^{N_{\mathrm{layer}}-1}\va^l_s+\sum_{l=0}^{N_{\mathrm{layer}}-1}\vm^l_s,
\end{equation*}
where $\va^l_s$ and $\vm^l_s$ denote the residual updates produced by the attention module and the FFN module in layer $l$, respectively. This notation keeps the full layer dependence inside $\va^l_s$ and $\vm^l_s$, and only uses the additive structure of the residual stream.

We focus on the label-dependent part of the gradient flow. Taking the expectation over training sequences and prediction positions, we have
\begin{equation*}
\left.\frac{d\vw^E_\alpha}{dt}\right|^y=\mathbb{E}\left[\left(\frac{\partial \vh^{N_{\mathrm{layer}}}_s}{\partial \vw^E_\alpha}\right)^T\vW^{U,T}\ve_{x_{s+1}}\right].
\end{equation*}
Using the residual decomposition above, this gradient can be decomposed into three types of path contributions:
\begin{equation*}
\left.\frac{d\vw^E_\alpha}{dt}\right|^y=\vG^{\mathrm{direct}}_\alpha+\sum_{l=0}^{N_{\mathrm{layer}}-1}\vG^{\mathrm{attn},l}_\alpha+\sum_{l=0}^{N_{\mathrm{layer}}-1}\vG^{\mathrm{ffn},l}_\alpha,
\end{equation*}
where
\begin{equation*}
\vG^{\mathrm{direct}}_\alpha:=\mathbb{E}\left[\left(\frac{\partial \vh^0_s}{\partial \vw^E_\alpha}\right)^T\vW^{U,T}\ve_{x_{s+1}}\right],\quad
\vG^{\mathrm{attn},l}_\alpha:=\mathbb{E}\left[\left(\frac{\partial \va^l_s}{\partial \vw^E_\alpha}\right)^T\vW^{U,T}\ve_{x_{s+1}}\right],
\end{equation*}
and
\begin{equation*}
\vG^{\mathrm{ffn},l}_\alpha:=\mathbb{E}\left[\left(\frac{\partial \vm^l_s}{\partial \vw^E_\alpha}\right)^T\vW^{U,T}\ve_{x_{s+1}}\right].
\end{equation*}
Thus, the embedding gradient in a full Transformer is an additive sum of the direct path, attention paths, and FFN paths. This identity provides the bridge between the modular analyses above and the embedding dynamics of a full language model.

\paragraph{Direct path.}
We first analyze the direct contribution $\vG_{\alpha}^{\mathrm{direct}}$. Since the input residual stream at position $s$ contains the embedding of the current token $x_s$, we have
\begin{equation*}
\frac{\partial \vh_s^0}{\partial \vw_{\alpha}^{E}}=\mathbf{1}_{x_s=\alpha}\vI.
\end{equation*}
Substituting this identity into the direct gradient contribution gives
\begin{equation*}
\vG_{\alpha}^{\mathrm{direct}}=\mathbb{E}\left[\mathbf{1}_{x_s=\alpha}\vW^{U,T}\ve_{x_{s+1}}\right].
\end{equation*}
Let $r_{\alpha}^{\mathrm{nt}}=\mathbb{P}(x_s=\alpha)$ be the probability that token $\alpha$ appears at the prediction position, and define the next-token signature of $\alpha$ as $\vvarphi_{\alpha}^{y;\mathrm{nt}}\in\sR^{V}$ whose $\nu$-th element is
\begin{equation*}
\vvarphi_{\alpha}^{y;\mathrm{nt}}(\nu)=\mathbb{P}(x_{s+1}=\nu\mid x_s=\alpha).
\end{equation*}
Then the direct contribution can be rewritten as 
\begin{equation*}
\vG_{\alpha}^{\mathrm{direct}}=r_{\alpha}^{\mathrm{nt}}\vW^{U,T}\vvarphi_{\alpha}^{y;\mathrm{nt}}.
\end{equation*}
Therefore, the direct embedding-readout path updates $\vw_{\alpha}^{E}$ according to the distribution of tokens that follow $\alpha$. In other words, before considering the indirect effects of attention and FFN modules, the most immediate statistical signal received by the embedding of $\alpha$ is its next-token signature.

\paragraph{FFN path.}
 The FFN module has already been analyzed in the previous sections. There, we showed that a position-wise nonlinear module can introduce signature information into the embedding dynamics, and that the order of the induced signature depends on the effective input structure and the activation order. In the full Transformer, however, the FFN module receives the residual stream at a single position. Therefore, at the beginning of training, the input to the FFN module at position $s$ is dominated by the residual direct term:
\begin{equation*}
\vh^\ell_s \approx \vh^0_s=\vw^E_{x_s}.
\end{equation*}
Thus, from the perspective of the FFN module, the effective input sequence has length one. Consequently, in the early stage, the FFN path only induces zero-order label signatures which is the next-token signature under the next-token prediction.

As training proceeds, the attention outputs and the residual stream become comparable in scale. For example, after the first attention module, the input to the following FFN at position $s$ can be written schematically as
\begin{equation*}
\vh^0_s+\va^0_s\approx \vw^E_{x_s}+\vW^{OV,0}\sum_{j\leq s}A^0_{s,j}\vw^E_{x_j}.
\end{equation*}
Once $\left\|\va^0_s\right\|$ becomes comparable to $\left\|\vh^0_s\right\|$, the FFN input becomes a weighted combination of the embeddings of the previous tokens $x_1,\ldots,x_s$. Therefore, by the FFN analysis in the previous section, the FFN path can then introduce higher-order signature information into the embedding gradient. These analysis indicate that $\vG_\alpha^{{\rm ffn},l}$ could be formulated as a combination of signatures with different orders.


\paragraph{Attention path.}
The effect of the attention path follows directly from the analysis in the previous section. In particular, Theorem~\ref{thm:attention} shows that, at the early stage of training, the leading structure introduced by the attention path is an attention-weighted zero-order label signature. For a token $\alpha$, let $I_s\sim {\rm Unif}([s])$ be sampled independently of the training sequence, and condition on $x_{I_s}=\alpha$. In this prefix setting, we use $\mathbb P(\,\cdot\mid\alpha)$ as shorthand for $\mathbb P(\,\cdot\mid x_{I_s}=\alpha)$. The total attention mass at prediction position $s$ is
\[
A_\alpha(X_{:s})
=
\sum_{j\leq s} a_{s,j}\mathbf{1}\{x_j=\alpha\}.
\]
Then the corresponding signature takes the form
$\vvarphi^{y;0}_\alpha\odot\vpsi^y_\alpha$, where
$\vvarphi^{y;0}_\alpha(\nu)=\mathbb P(x_{s+1}=\nu\mid \alpha)$
and
$\vpsi^y_\alpha(\nu)=\mathbb E[a_{s,I_s}\mid x_{s+1}=\nu,\alpha]$.
Therefore, the attention path contains the same zero-order token--label statistics as before, but reweights them according to how much attention is assigned to occurrences of $\alpha$. Tokens receiving larger attention mass consequently exert a stronger influence on the corresponding embedding update.

As training proceeds, the key- and query-dependent terms derived in Theorem~\ref{thm:attention} introduce additional position-specific and higher-order signatures. We do not repeat their derivation here and focus below on the early-stage attention-weighted zero-order signature.

\subsection{Experimental Validation}

We next investigate how the statistical structures encoded by token embeddings evolve during real language-model training. Our theoretical analysis predicts that the embedding geometry is initially dominated by low-order signatures, while richer higher-order signatures gradually become encoded as training proceeds. We first verify this progression from the next-token signature to the first-order signature. We then provide an additional verification of the attention-path prediction derived from the path analysis.

We train a 0.7B-parameter decoder-only Transformer on 36B tokens. Let $\mathcal T$ denote the set of the 200 most frequent eligible tokens in the training corpus. For each token $\alpha\in\mathcal T$, we first estimate its empirical next-token signature $\vvarphi^{y;\mathrm{nt}}_\alpha$, and construct the corresponding similarity matrix
\begin{equation*}
S^{\mathrm{next}}_{\alpha,\beta}=\frac{\left\langle\vvarphi^{y;\mathrm{nt}}_\alpha,\vvarphi^{y;\mathrm{nt}}_\beta\right\rangle}{\left|\vvarphi^{y;\mathrm{nt}}_\alpha\right|_2\left|\vvarphi^{y;\mathrm{nt}}_\beta\right|_2},\qquad \alpha,\beta\in\mathcal T.
\end{equation*}

For each epoch $t$, we compute the embedding similarity matrix
\begin{equation*}
S^{\mathrm{emb},(t)}_{\alpha,\beta}=\frac{\left\langle\vw^{E,(t)}_\alpha,\vw^{E,(t)}_\beta\right\rangle}{\left|\vw^{E,(t)}_\alpha\right|_2\left|\vw^{E,(t)}_\beta\right|_2},\qquad \alpha,\beta\in\mathcal T,
\end{equation*}
and measure the agreement between the embedding geometry and the next-token signature geometry by
\begin{equation*}
\rho_{\mathrm{next}}^{(t)}=\operatorname{Corr}\left(\left\{S^{\mathrm{emb},(t)}_{\alpha,\beta}:\alpha<\beta\right\},\left\{S^{\mathrm{next}}_{\alpha,\beta}:\alpha<\beta\right\}\right).
\end{equation*}

To examine whether the embedding geometry progressively captures richer statistical structures, we further estimate the first-order signature. For an occurrence of $\alpha$ at position $s$, we record whether another token $\beta$ has appeared anywhere in the prefix $X_{<s}$, regardless of its position or multiplicity. The resulting first-order signature is
\begin{equation}
\label{eq:first_order_real}
\vvarphi^{y;1}_\alpha(\nu,\beta)=\mathbb P\left(x_{s+1}=\nu,\beta\in X_{<s}\mid x_s=\alpha\right),\qquad \nu,\beta\in\mathcal V.
\end{equation}
Since $\vvarphi^{y;1}_\alpha$ contains an additional context-token dimension, it cannot be compared directly with the embedding geometry. We learn a shared projection along the context-token dimension together with a linear readout. Specifically, at epoch $t$, the probe produces an embedding-like representation
\begin{equation}
\label{eq:first_order_probe}
\widehat{\vw}^{E,(t)}_\alpha=\vvarphi^{y;1}_\alpha\mathbf v^{(t)},
\end{equation}
where $\mathbf v^{(t)}$ is a learned probe over the context-token dimension. We randomly split the anchor tokens into a fixed training set and a held-out test set. At each epoch, the probe parameters are learned using only the training tokens, while evaluation is performed exclusively on similarities between held-out tokens. Let
\begin{equation*}
\widehat S^{(1),(t)}_{\alpha,\beta}=\frac{\left\langle\widehat{\vw}^{E,(t)}_\alpha,\widehat{\vw}^{E,(t)}_\beta\right\rangle}{\left|\widehat{\vw}^{E,(t)}_\alpha\right|_2\left|\widehat{\vw}^{E,(t)}_\beta\right|_2}
\end{equation*}
denote the similarity predicted by the first-order signature probe. The probe is trained directly in the similarity space by minimizing the MSE between $\widehat S^{(1),(t)}_{\alpha,\beta}$ and $S^{\mathrm{emb},(t)}_{\alpha,\beta}$. 

Figures~\ref{fig:next-token}A and B verify the transition from low-order to higher-order signatures. Figure~\ref{fig:next-token}A shows that $\rho_{\mathrm{next}}^{(t)}$ rises rapidly from nearly zero, reaches a pronounced peak at the beginning of training, and then gradually decreases to a stable nonzero level. This indicates that the embedding geometry is initially dominated by the next-token signature, while additional statistical structures emerge during subsequent training. Figure~\ref{fig:next-token}B reports the held-out similarity MSE of the first-order signature probe. Although the embedding geometry changes rapidly during the earliest stage of optimization, the probe error quickly decreases and remains consistently low throughout the remainder of training. Since the probe is trained only on one subset of anchor tokens and evaluated exclusively on unseen tokens, this result shows that a shared rank-one projection of the first-order signature generalizes across tokens and captures a substantial fraction of the learned embedding geometry.

As an additional verification of the path analysis, we also examine the attention path. For layer $\ell$ and epoch $t$, let $\bar A^{\ell,(t)}_{s,j}$ denote the attention weight from position $s$ to position $j$, averaged over attention heads. Let $I_s\sim{\rm Unif}([s])$ be independent of the training sequence, and use the same occurrence-conditioned shorthand $\mathbb E[\,\cdot\mid\alpha]=\mathbb E[\,\cdot\mid x_{I_s}=\alpha]$. The attention-path signature associated with token $\alpha$ is estimated as
\begin{equation}
\label{eq:attn_phi}
\vvarphi^{y;\mathrm{attn},\ell,(t)}_\alpha(\nu)=\mathbb E\left[\bar A^{\ell,(t)}_{s,I_s}\mathbf 1_{\{x_{s+1}=\nu\}}\mid\alpha\right],\qquad \nu\in\mathcal V.
\end{equation}
We then construct
\begin{equation*}
S^{\mathrm{attn\text{-}sig},\ell,(t)}_{\alpha,\beta}=\frac{\left\langle\vvarphi^{y;\mathrm{attn},\ell,(t)}_\alpha,\vvarphi^{y;\mathrm{attn},\ell,(t)}_\beta\right\rangle}{\left|\vvarphi^{y;\mathrm{attn},\ell,(t)}_\alpha\right|_2\left|\vvarphi^{y;\mathrm{attn},\ell,(t)}_\beta\right|_2},\qquad \alpha,\beta\in\mathcal T.
\end{equation*}
For each layer $\ell$ and epoch $t$, we define
\begin{equation*}
\vz^{\ell,(t)}_\alpha=\vW^{O,\ell,(t)}\vW^{V,\ell,(t)}\vw^{E,(t)}_\alpha,
\end{equation*}
construct
\begin{equation*}
S^{\mathrm{attn\text{-}logits},\ell,(t)}_{\alpha,\beta}=\frac{\left\langle\vz^{\ell,(t)}_\alpha,\vz^{\ell,(t)}_\beta\right\rangle}{\left|\vz^{\ell,(t)}_\alpha\right|_2\left|\vz^{\ell,(t)}_\beta\right|_2},
\end{equation*}
and compute
\begin{equation*}
\rho_{\mathrm{attn}}^{\ell,(t)}=\operatorname{Corr}\left(\left\{S^{\mathrm{attn\text{-}logits},\ell,(t)}_{\alpha,\beta}:\alpha<\beta\right\},\left\{S^{\mathrm{attn\text{-}sig},\ell,(t)}_{\alpha,\beta}:\alpha<\beta\right\}\right).
\end{equation*}
Figure~\ref{fig:next-token}C shows that the similarity structure induced by the attention-path representations exhibits a clear early-stage correlation with the corresponding attention-weighted signature, providing empirical support for the path-specific prediction derived in the theoretical analysis.

\begin{figure}[htbp]
    \centering
    \includegraphics[width=1\linewidth]{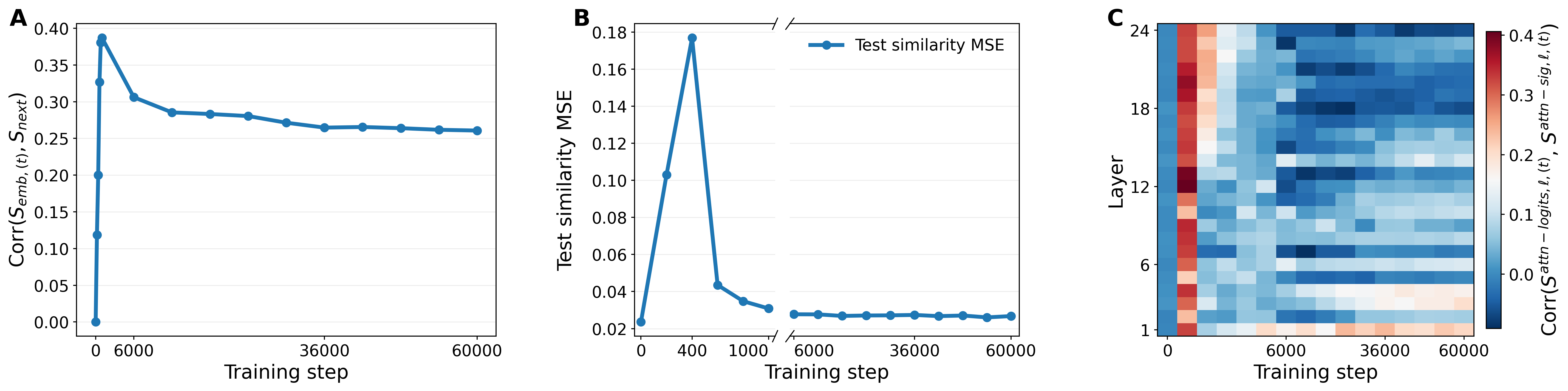}
    \caption{A: Correlation between the embedding similarity matrix $S^{\mathrm{emb},(t)}$ and the next-token signature similarity matrix $S^{\mathrm{next}}$ over training steps, computed on the 200 most frequent tokens. B: held-out similarity MSE between the pairwise cosine similarities predicted by the rank-one first-order-signature probe and those of the token embeddings across training epochs; the broken horizontal axis separates the early-training epochs from the later epochs. C: Layer-wise heatmap of the correlation between the attention-path similarity matrix $S^{\mathrm{attn-logits},\ell,(t)}$ and the attention-weighted signature similarity matrix $S^{\mathrm{attn-sig},\ell,(t)}$ over training steps.}
    \label{fig:next-token}
\end{figure}

\section{The Role of Signatures in Model Learning}
The preceding sections establish that learned embeddings can align with identifiable statistical signatures and that these signatures appear explicitly in the embedding updates. We now ask a separate question: does such a structure actually help the model learn the task?

\subsection{Relation between Signature and Learning}
The empirical results above show that the embedding space can reflect a hierarchy of signatures, while the gradient-flow analysis identifies these signatures as components of the embedding update. A natural question is whether learning such signature structures actually facilitates the learning of the task itself. In particular, does the emergence of a mathematically meaningful embedding geometry necessarily imply that the model has learned the corresponding task?

We investigate this question from three complementary perspectives. We first show that encoding a signature structure is not by itself sufficient for solving a task. We then test whether directly providing task-relevant signature information can accelerate learning. Finally, we show that the order at which discriminative information first appears in the signature hierarchy can strongly affect when the task is learned.

\paragraph{Signature structure alone is not sufficient for task learning.}

We first illustrate this point using the binary modular-addition task $F_{\rm mod}$ (see Task~\ref{task:2token_math}). Previous studies on modular addition often associate the emergence of structured token representations with successful task learning~\citep{liu2022towards,nanda2023progress,gromov2023grokking,mallinar2025emergence}. Consistent with this picture, in the FFN model, once the embeddings become aligned with the first-order signatures and form the corresponding modular structure, the model rapidly learns the task (Figure~\ref{fig:sig_emb} A, B).

The attention-only model provides a contrasting example (Figure~\ref{fig:sig_emb} A, C). Its learnable embeddings also become dominated by first-order signature information and exhibit a clear mathematical structure, yet the model fails to solve the task. Thus, a meaningful signature-aligned embedding geometry is not sufficient by itself: the downstream architecture must also be able to transform this representation into the target operation. In this setting, the attention-only model is not expressive enough to effectively exploit the learned signature structure.

\begin{figure}[htb]
    \centering
    \includegraphics[width=1\linewidth]{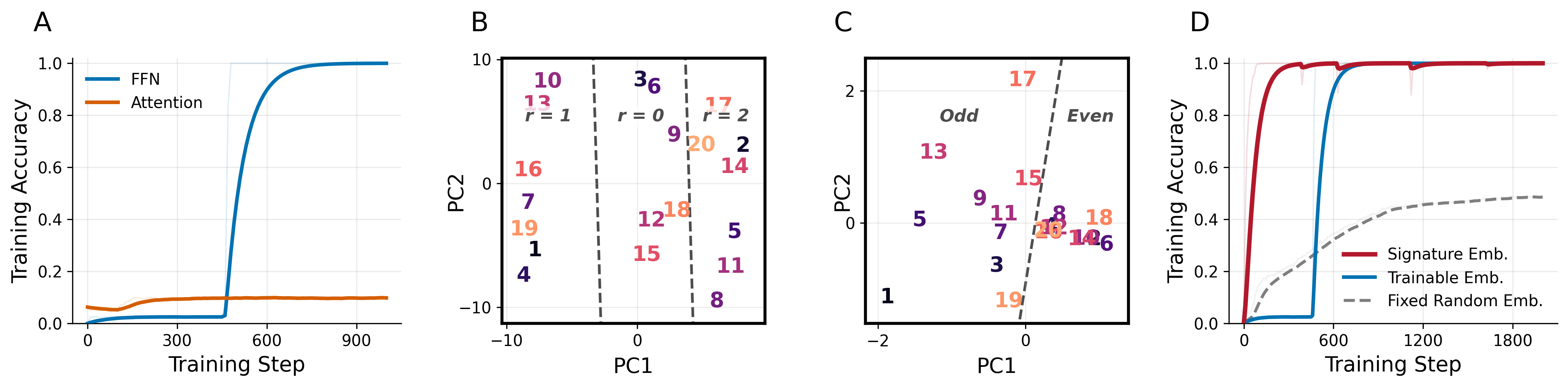}
    \caption{A: Training accuracy of the FFN model and the attention-only model. B: PCA projection of the final FFN token embeddings. C: PCA projection of the final attention-only token embeddings. D: Training accuracy of the FFN model with fixed first-order signature embeddings, trainable embeddings, and fixed random embeddings.}
    \label{fig:sig_emb}
\end{figure}

\paragraph{Task-relevant signatures can accelerate learning.}

We next test whether the signature itself contains information that can directly facilitate optimization. We train the FFN model with a fixed first-order signature embedding, where each token is represented by a random linear transformation of $\boldsymbol{\varphi}_{a}^{y;1}$ (see Appendix for details). The resulting model converges substantially faster than the standard trainable-embedding baseline (Figure~\ref{fig:sig_emb} D). This confirms that first-order signatures contain task-relevant information that can directly facilitate learning.

Interestingly, fixed random embeddings also outperform trainable embeddings at the early stage of this modular-addition task. In this setting, all number tokens have the same zeroth-order label signature. The early embedding updates therefore contain little information for distinguishing different tokens, whereas fixed random embeddings are approximately pairwise orthogonal and preserve token distinctions from the beginning. This suggests that not only the presence of signature information, but also the order at which discriminative information appears in the signature hierarchy, can affect the learning dynamics.

\paragraph{Signature order controls the onset of task learning.}

We test this prediction more systematically by considering binary modular addition $F_{\rm mod}$
\[
y=(x_1+x_2)\bmod M,
\qquad
x_1,x_2\in\{1,\ldots,N\},
\]
and varying the modulus $M$ while keeping $N$ fixed.

The $\vvarphi_{\alpha}^{y;0}$ in this task have a simple dependence on $M$ and $N$. When $M$ divides $N$, all tokens have identical zeroth-order label signatures. Their pairwise signature similarities are therefore all equal to one. In this case, the $\vvarphi_{\alpha}^{y;0}$ provides no information for distinguishing each token required by the task, and the model must rely on higher-order signature information to learn the modular rule.

When $M$ does not divide $N$, the residue classes contain unequal numbers of tokens. As a result, the zeroth-order signatures are no longer identical across all tokens and already provide discriminative information at the beginning of training. Their similarity matrices consequently exhibit a regular residue-dependent structure rather than being uniformly one.

We verify this prediction with $N=20$ and vary $M$ across a range of values. Figure~\ref{fig:modulus_learning} reports the first epoch at which the training accuracy reaches $95\%$, together with the pairwise similarity matrices of the $\vvarphi_{\alpha}^{y;0}$. A striking slowdown occurs precisely when $M$ divides $N$: for $M=4,5,10,$ and $20$, the similarity matrices are uniformly one and the model requires thousands of epochs to reach high accuracy. For the other tested values of $M$, the $\vvarphi_{\alpha}^{y;0}$ already distinguish different residue classes, and the model typically learns the task within only a few hundred epochs.

\begin{figure}[htb]
    \centering
    \includegraphics[width=1\linewidth]{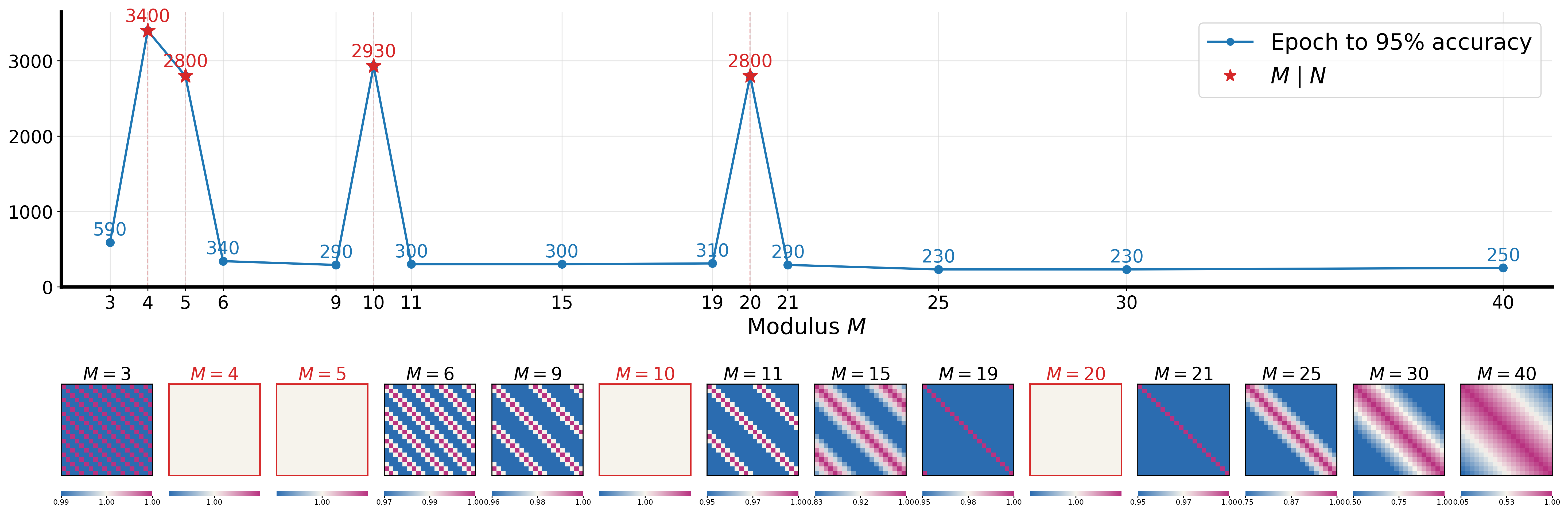}
    \caption{Learning speed of $F_{\rm mod}$ under different moduli $M$, with $N=20$. The vertical axis shows the first epoch at which training accuracy reaches $95\%$; gray points denote runs where $M | N$. The accompanying heatmaps show the pairwise similarities between zeroth-order signatures. }
    \label{fig:modulus_learning}
\end{figure}

These results establish a direct connection between the signature hierarchy and task learning. Signatures can provide useful representations for solving a task, but whether and when this information becomes useful depends both on the architecture that processes the embeddings and on the order at which task-relevant distinctions emerge in the signature hierarchy.

\subsection{An Example of Task Solving via Signatures}

We then use a simple order-selection task to illustrate how signatures can support task solving when the model is able to learn the task successfully.
Let the label space be $\mathcal{Y}=\{1,\dots,N\}$. 
Each input sequence has the form
\begin{equation*}
S=(c,x_1,x_2), \qquad c\in\{a,b\},\quad x_1,x_2\in\mathcal{Y},
\end{equation*}
where $x_1$ and $x_2$ are numerical tokens and $c$ is an anchor token that determines the selection rule. We take $c$ uniformly from $\{a,b\}$ and $x_1,x_2$ independently and uniformly from $\mathcal{Y}$. The target is defined as
\begin{equation*}
Y=f(S)=
\begin{cases}
\max(x_1,x_2), & c=a,\\[4pt]
\min(x_1,x_2), & c=b.
\end{cases}
\end{equation*}
Thus, the task is not simply to compare two numbers. 
Instead, the model must select one of the two order statistics according to the contextual anchor token $c$. This task is motivated by the familiar ambiguity in comparing numbers such as $3.9$ and $3.11$. When they are interpreted as decimal numbers, $3.9 > 3.11$. However, when they are interpreted as version numbers, $3.11 > 3.9$~\citep{xie2024order,yang2024number,chen2025steering}.

We now show that the superposition of zero-order signatures is sufficient to solve the task. For a sequence $(c,x_1,x_2)$, we define the signature-induced score for label $\nu$ as
\begin{equation*}
\Phi_{c,x_1,x_2}(y)=\vvarphi_c^{y;0}(y)+\vvarphi_{x_1}^{y;0}(y)+\vvarphi_{x_2}^{y;0}(y).
\end{equation*}
The prediction induced by zero-order signature superposition is
\begin{equation*}
\hat{y}=\arg\max_{y\in\mathcal{Y}}\Phi_{c,x_1,x_2}(y).
\end{equation*}

We first consider the contribution of the two numerical tokens. Fix a numerical token $x$ and denote the other numerical token by $z$, which is uniformly sampled from $\mathcal{Y}$. There are three cases. If $z<x$, the output is $x$ when the anchor is $a$, which occurs with probability $1/2$. If $z>x$, the output is $x$ when the anchor is $b$, again with probability $1/2$. Thus, whenever $z\neq x$, one of the two equally likely anchor choices outputs $x$, while the other outputs $z$. If $z=x$, both the maximum and minimum are equal to $x$, so the output is always $x$.

Equivalently, we can decompose the label distribution into two parts. One half of the probability mass is always assigned to $x$, corresponding to the anchor choice that selects $x$ when $z\neq x$. The other half follows the value of $z$. Since $z$ is uniformly distributed over $\mathcal{Y}$, this second part contributes $1/2N$ to every label, including $y=x$. Therefore,
\begin{equation*}
\vvarphi_x^{y;0}
=
\frac{1}{2}\mathbf{1}_{y=x}
+
\frac{1}{2N}.
\end{equation*}

Therefore, the sum of the two numerical signatures is
\begin{equation*}
\vvarphi_{x_1}^{y;0}+\vvarphi_{x_2}^{y;0}=\frac{1}{N}+\frac{1}{2}\mathbf{1}_{y=x_1}+\frac{1}{2}\mathbf{1}_{y=x_2}.
\end{equation*}
Thus, the two numerical tokens create two equal candidate peaks at $y=x_1$ and $y=x_2$, while all non-candidate labels only receive the constant background value $1/N$. Since the numerical signatures contribute an additional peak of height $1/2$ at $x_1$ and $x_2$, whereas the anchor signatures vary only at scale $O(1/N)$, the maximizer of $\Phi_{c,x_1,x_2}(y)$ lies among the two candidate labels.

Assume without loss of generality that $x_1<x_2$. When $c=a$, the anchor signature is   
\begin{equation*}
\vvarphi_a^{y;0}(y)=\frac{2y-1}{N^2},
\end{equation*}
which is monotonically increasing in $y$. The total signature score becomes
\begin{equation*}
\Phi_{a,x_1,x_2}(y)=\frac{2y-1}{N^2}+\frac{1}{N}+\frac{1}{2}\mathbf{1}_{y=x_1}+\frac{1}{2}\mathbf{1}_{y=x_2}.
\end{equation*}
The two candidate labels have scores
\begin{equation*}
\Phi_{a,x_1,x_2}(x_1)=\frac{2x_1-1}{N^2}+\frac{1}{N}+\frac{1}{2},\qquad \Phi_{a,x_1,x_2}(x_2)=\frac{2x_2-1}{N^2}+\frac{1}{N}+\frac{1}{2}.
\end{equation*}
Since $x_2>x_1$, we have
\begin{equation*}
\Phi_{a,x_1,x_2}(x_2)>\Phi_{a,x_1,x_2}(x_1).
\end{equation*}
Hence the maximum is attained at the larger candidate:
\begin{equation*}
\arg\max_{y\in\mathcal{Y}}\Phi_{a,x_1,x_2}(y)=x_2=\max(x_1,x_2).
\end{equation*}

When $c=b$, the anchor signature is
\begin{equation*}
\vvarphi_b^{y;0}(y)=\frac{2(N-y)+1}{N^2},
\end{equation*}
which is monotonically decreasing in $y$. The total signature score becomes
\begin{equation*}
\Phi_{b,x_1,x_2}(y)=\frac{2(N-y)+1}{N^2}+\frac{1}{N}+\frac{1}{2}\mathbf{1}_{y=x_1}+\frac{1}{2}\mathbf{1}_{y=x_2}.
\end{equation*}
The two candidate labels have scores
\begin{equation*}
\Phi_{b,x_1,x_2}(x_1)=\frac{2(N-x_1)+1}{N^2}+\frac{1}{N}+\frac{1}{2},\qquad \Phi_{b,x_1,x_2}(x_2)=\frac{2(N-x_2)+1}{N^2}+\frac{1}{N}+\frac{1}{2}.
\end{equation*}
Since $x_1<x_2$, we have
\begin{equation*}
\Phi_{b,x_1,x_2}(x_1)>\Phi_{b,x_1,x_2}(x_2).
\end{equation*}
Hence the maximum is attained at the smaller candidate:
\begin{equation*}
\arg\max_{y\in\mathcal{Y}}\Phi_{b,x_1,x_2}(y)=x_1=\min(x_1,x_2).
\end{equation*}

Therefore, the min--max task can be solved by the additive superposition of zero-order signatures. This construction does not claim that every trained model implements exactly this score; rather, it demonstrates that the statistical information contained in the signatures is sufficient to support the required prediction rule when the architecture can access it.

\begin{figure}[htbp]
    \centering
    \includegraphics[width=1\linewidth]{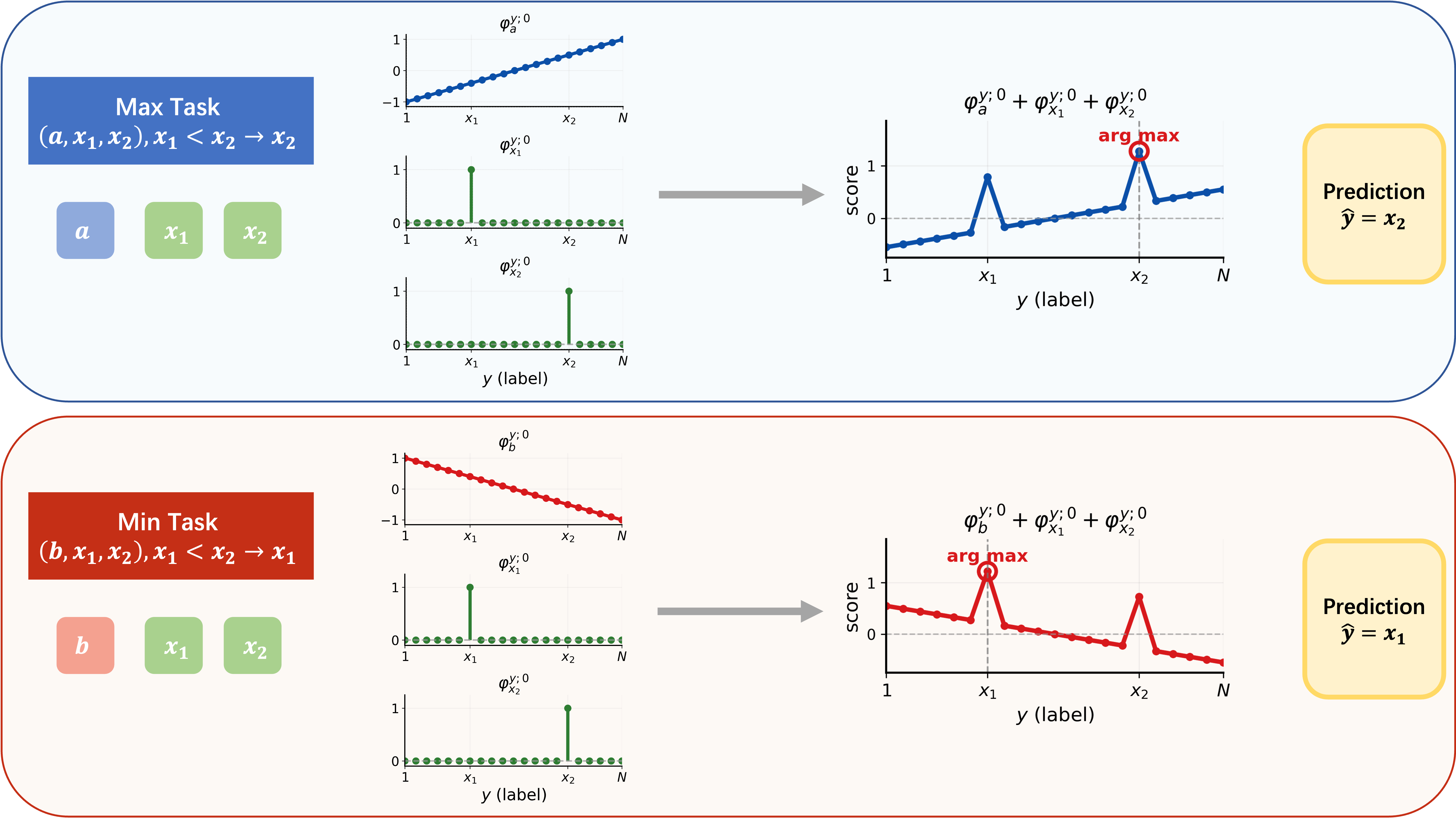}
    \caption{Mechanism of zero-order signature superposition in the min--max task.}
    \label{fig:min-max}
\end{figure}

\subsection{Signatures Encode Semantic Structure}
\label{sec:signature_semantics}
We further find that signatures themselves already contain interpretable semantic structures, suggesting that they provide an important statistical basis from which embeddings acquire semantic organization. Previous studies have extensively documented that learned word and language-model representations exhibit rich semantic geometry, including linear relational structures, semantic clustering, and low-dimensional organization of real-world concepts~\citep{mikolov2013linguistic,pennington2014glove,yaghoobzadeh2019probing,gurnee2024language}. Most of these studies, however, focus on characterizing or probing the resulting representation geometry. Here, we instead ask whether such structures can already be traced back to identifiable statistics of the training data.

To this end, we compute the next-token signatures $\vvarphi^{y;{\rm nt}}_{\alpha}$ directly from the same 36B-token corpus used to train our language model and visualize the signatures of several groups of tokens using PCA. As shown in Fig.~\ref{fig:signature_semantics} A, the signatures of numerical tokens form a continuous trajectory that approximately follows their numerical order. Figure~\ref{fig:signature_semantics} B presents the signatures of the seven weekdays, which exhibit a cyclic organization consistent with their periodic semantic relationship. Beyond ordered structures, signatures also encode categorical information. In Fig.~\ref{fig:signature_semantics} C, we jointly project the signatures of tokens from several semantic categories, including colors, family roles, countries, and programming languages. Tokens belonging to the same semantic category are naturally grouped together, while different categories occupy distinct regions of the projected space.

These results show that semantic geometries commonly observed in learned representations are already present in the token-conditioned statistics of the training corpus. Combined with our dynamical analysis showing how these signatures progressively enter the embedding updates, this provides a possible statistical origin for the emergence of semantic organization in the learned embedding space.

\begin{figure}[htb]
    \centering
    \includegraphics[width=\linewidth]{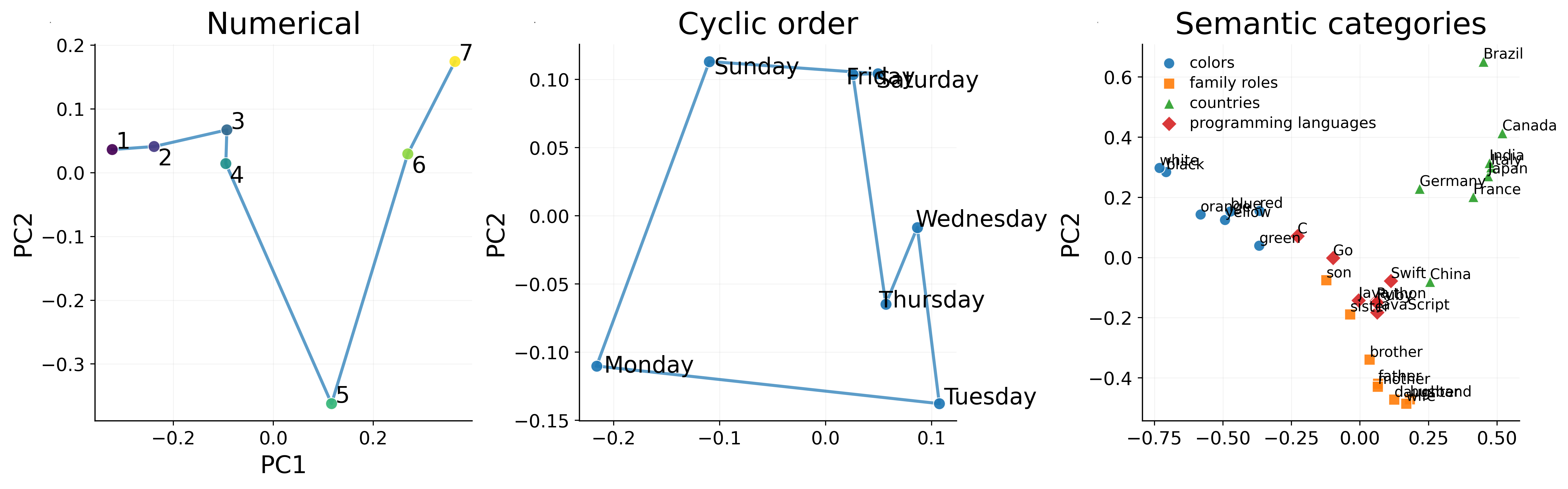}
    \caption{
    A: PCA projection of the signatures of numerical tokens, connected according to their numerical order.
    B: PCA projection of the signatures of weekdays, connected according to their cyclic order.
    C: Joint PCA projection of the signatures of tokens from different semantic categories, including colors, family roles, countries, and programming languages.
    }
    \label{fig:signature_semantics}
\end{figure}

\section{Conclusion}

In this work, we identified a progressive signature-alignment phenomenon in token-embedding dynamics. Across controlled sequence-prediction tasks, early embedding geometries are well explained by simple token-conditioned label distributions. Later in training, the effects of multiple higher-order contextual signatures become jointly visible rather than following a strict one-at-a-time replacement, with the highest accessible order providing the strongest contribution under the late-stage conditions of our analysis. The same evolution admits a spectral interpretation: low-order structure corresponds to zero-frequency projections of higher-order signatures, while task-specific nonzero-frequency components emerge later in training. This perspective also accounts for the distinct geometries obtained from different random seeds as different spectral selections within a shared signature hierarchy.

We used gradient-flow analysis to provide a mechanistic explanation for these observations. In feed-forward models, signatures of different orders appear explicitly in the embedding gradient, and small initialization suppresses higher-order contributions at the beginning of training. In self-attention models, position-specific signatures are modulated by attention scores, so learned routing determines which token statistics have the strongest effect on the embeddings. These results explain how the data distribution and computational architecture jointly determine the statistical components available to shape embedding updates, without requiring the learned geometry to be treated as a purely post-hoc object.

The controlled analyses further motivate path-specific hypotheses in full Transformers. In a real language-model training run, we find that early embedding geometry tracks next-token signatures associated with the direct and FFN paths, while attention-path representations track attention-weighted signatures across layers. The correlations subsequently decay as additional structures enter the model, consistent with the progressive nature of the signature-alignment picture.

Finally, we showed that learning a meaningful signature-aligned geometry is not by itself sufficient for solving a task. Whether the encoded structure is useful depends on the computations available after the embedding lookup. When the architecture can exploit the relevant signature information, signature-based embeddings can directly support prediction and accelerate optimization; when it cannot, a structured embedding space may emerge without successful task learning.

Overall, our results connect an empirical phenomenon in embedding geometry with an interpretable gradient-flow mechanism. The signature perspective provides a unified language for describing what statistical structures appear in token embeddings, how those structures change during training, how architectural paths modulate them, and when they become useful for task solving.

\acks{This work is sponsored by the National Key R$\&$D Program
of China Grant No. 2022YFA100\\8200  (Z. X.), the National Natural Science Foundation of China Grant No. 92570001 (Z. X.), 12371511 (Z. X.), 12422119 (Z. X.), 2025 Key Technology R\&D Program ``New Generation Information Technology'' Project 25511103100 of Shanghai Municipal Science and Technology Commission (Z. X.).}


\newpage

\appendix

\section{Experimental Configurations}
\label{app:experimental_configurations}

This section provides the complete experimental configurations for the controlled FFN experiments, the attention-only experiments, the signature-based embedding experiments, and the real language-model training experiment. Unless otherwise specified, all reported results are obtained using the common settings described below.

\subsection{FFN Experiments}
\label{app:ffn_experiments}

\subsubsection{FFN Architecture}
\label{app:ffn_architecture}

For the controlled FFN experiments, we set $d=128$ unless otherwise specified. The model contains no hidden bias and no output bias. No normalization layer, residual connection, dropout, or positional embedding is used.

For Figures~1--3 and Figure~7, the activation function is $\sigma(z)=\operatorname{ReLU}(z)$,
applied coordinate-wise. For the activation-order experiment in Figure~4, we instead use the monomial activation $\sigma(z)=z^{\odot K}$.

\subsubsection{Prescribed Label-Distribution Experiment}
\label{app:prescribed_distribution}

This experiment corresponds to Task~\ref{task:random_label} and Figure~1A. We use one anchor token $a=100$ and the context-token set $\mathcal{X}=\{0,1,\ldots,50\}$.
Each input sequence has length three and takes the form
\begin{equation}
    X=(a,x_1,x_2),
    \qquad x_1,x_2\in\mathcal{X}.
\end{equation}
The complete input set is
\begin{equation}
    \mathcal{D}_a
    =
    \left\{
    (a,x_1,x_2):
    x_1,x_2\in\mathcal{X}
    \right\},
\end{equation}
which contains $51^2=2601$ distinct input sequences.

For every input $X\in\mathcal{D}_a$, its target label is sampled from a prescribed distribution $p_a$:
\begin{equation}
    y(X)\sim p_a.
\end{equation}
The sampled labels are sampled once before training and kept fixed. The label vocabulary is $\mathcal{Y}=\fX$.

We consider the following four prescribed distributions.

\paragraph{Uniform distribution.}
The uniform distribution is defined by
\begin{equation}
    p_a(y)
    =
    \frac{1}{|\mathcal{Y}|},
    \qquad y\in\mathcal{Y}.
\end{equation}

\paragraph{Normal distribution.}
The unnormalized probability mass is defined as
\begin{equation}
    \widetilde{p}_a(y)
    =
    \exp\left(
    -\frac{(y-\mu)^2}{2\sigma^2}
    \right),
\end{equation}
where $\mu=25,\sigma=1$.
The discrete distribution is obtained by normalization:
\begin{equation}
    p_a(y)
    =
    \frac{\widetilde{p}_a(y)}
    {\sum_{\nu\in\mathcal{Y}}\widetilde{p}_a(\nu)}.
\end{equation}

\paragraph{Poisson distribution.}
The unnormalized probability mass is
\begin{equation}
    \widetilde{p}_a(y)
    =
    \frac{\lambda^y e^{-\lambda}}{y!},
\end{equation}
where $\lambda=10$. The distribution is truncated to $\mathcal{Y}$ and renormalized:
\begin{equation}
    p_a(y)
    =
    \frac{\widetilde{p}_a(y)}
    {\sum_{\nu\in\mathcal{Y}}\widetilde{p}_a(\nu)}.
\end{equation}

\paragraph{Exponential distribution.}
The unnormalized probability mass is
\begin{equation}
    \widetilde{p}_a(y)
    =
    \exp(-\lambda y),
\end{equation}
where $\lambda=5$. The distribution is normalized over $\mathcal{Y}$:
\begin{equation}
    p_a(y)
    =
    \frac{\widetilde{p}_a(y)}
    {\sum_{\nu\in\mathcal{Y}}\widetilde{p}_a(\nu)}.
\end{equation}

Figure~1A uses the checkpoint at epoch 200. The embedding-induced label distribution is computed as
\begin{equation}
    \widehat{p}_a
    =
    \operatorname{softmax}
    \left(
    \vw^{U}\vw^{E}_a
    \right).
\end{equation}
The Pearson correlation coefficient is computed between $p_a$ and $\widehat{p}_a$. The Kullback--Leibler divergence is computed as
\begin{equation}
    D_{\mathrm{KL}}(p_a\Vert\widehat{p}_a)
    =
    \sum_{y\in\mathcal{Y}}
    p_a(y)
    \log
    \frac{p_a(y)}
    {\widehat{p}_a(y)},
\end{equation}
and the Jensen--Shannon divergence is
\begin{equation}
    D_{\mathrm{JS}}(p_a,\widehat{p}_a)
    =
    \frac{1}{2}
    D_{\mathrm{KL}}(p_a\Vert m)
    +
    \frac{1}{2}
    D_{\mathrm{KL}}(\widehat{p}_a\Vert m),
    \qquad
    m=\frac{p_a+\widehat{p}_a}{2}.
\end{equation}

\subsubsection{Arithmetic Tasks}
\label{app:arithmetic_tasks}

For the arithmetic experiments, the input number-token set is
\begin{equation}
    \mathcal{X}_N
    =
    \{1,\ldots,N\}.
\end{equation}
Each input is an ordered pair
\begin{equation}
    X=(x_1,x_2),
    \qquad
    x_1,x_2\in\mathcal{X}_N.
\end{equation}
We use the complete input set
\begin{equation}
    \mathcal{D}
    =
    \mathcal{X}_N\times\mathcal{X}_N,
\end{equation}
containing $N^2$ samples. For Figure~1, we set $N=20$. and  uses the checkpoint at epoch 100.

\subsubsection{Addition and Modular-Addition Dynamics}
\label{app:addition_modular_dynamics}

For Figures~2 and~3, we set $N=40$.
The input set is again
\begin{equation}
    \mathcal{X}_{40}
    =
    \{1,\ldots,40\}.
\end{equation}

The addition model and the modular-addition model are trained for 1200 epochs. The early and late epochs shown in Figure~2 are 100 and 1100. The epochs used in Figure~3 are 100 and 1000.

\subsubsection{Optimized First-Order Signature Projection}
\label{app:signature_probe}

For each token $\alpha$, the first-order label signature is represented as a matrix
\begin{equation}
    \boldsymbol{\varphi}^{y;1}_\alpha
    \in
    \mathbb{R}^{|\mathcal{Y}|\times|\mathcal{X}|}.
\end{equation}
Given a probe vector
\begin{equation}
    \mathbf{v}
    \in
    \mathbb{R}^{|\mathcal{X}|},
\end{equation}
we define the projected signature
\begin{equation}
    \mathbf{h}^{(1)}_\alpha(\mathbf{v})
    =
    \boldsymbol{\varphi}^{y;1}_\alpha\mathbf{v}
    \in
    \mathbb{R}^{|\mathcal{Y}|}.
\end{equation}
The induced pairwise similarity is
\begin{equation}
    S^{\mathrm{sig},(1)}_{a,b}(\mathbf{v})
    =
    \frac{
    \left\langle
    \mathbf{h}^{(1)}_a(\mathbf{v}),
    \mathbf{h}^{(1)}_b(\mathbf{v})
    \right\rangle
    }{
    \left\|
    \mathbf{h}^{(1)}_a(\mathbf{v})
    \right\|_2
    \left\|
    \mathbf{h}^{(1)}_b(\mathbf{v})
    \right\|_2
    }.
\end{equation}

For every epoch $t$, the probe vector is optimized independently by solving
\begin{equation}
    \mathbf{v}^{(t)}
    =
    \arg\max_{\mathbf{v}}
    \operatorname{Corr}
    \left(
    \operatorname{vec}_{\triangle}
    \left(
    \mathbf{S}^{\mathrm{sig},(1)}(\mathbf{v})
    \right),
    \operatorname{vec}_{\triangle}
    \left(
    \mathbf{S}^{\mathrm{emb},(t)}
    \right)
    \right),
\end{equation}
where $\operatorname{vec}_{\triangle}$ extracts the strictly upper-triangular entries of a symmetric matrix.

The correlation objective is Pearson. The probe vector is initialized as
\begin{equation}
    v_i
    \sim
    \mathcal N (0,|\fX|^{-0.5}).
\end{equation}
It is optimized using Adam with learning rate 0.01 for 300 steps. The probe vector is trained using five-fold cross-validation over token pairs. For each epoch, it is optimized on four folds and evaluated on the remaining fold, and the reported results are averaged over the five held-out folds.

\subsubsection{Fourier Analysis of the First-Order Signature}
\label{app:fourier_analysis}

For Figure~3, we use the real cosine basis
\begin{equation}
    \mathbf{c}^{\cos}_k
    =
    \left(
    \cos\left(\frac{2\pi k j}{p}\right)
    \right)_{j=0}^{p-1},
    \qquad
    k=0,1,\ldots,\frac{p}{2},
\end{equation}
where
\begin{equation}
    p=N=40.
\end{equation}

For each token $\alpha$ and frequency $k$, the projected signature is
\begin{equation}
    \mathbf{h}^{(1)}_{\alpha,k}
    =
    \boldsymbol{\varphi}^{y;1}_\alpha
    \mathbf{c}^{\cos}_k.
\end{equation}
Its pairwise cosine-similarity matrix is
\begin{equation}
    S^{\mathrm{sig},(1)}_{k;a,b}
    =
    \frac{
    \left\langle
    \mathbf{h}^{(1)}_{a,k},
    \mathbf{h}^{(1)}_{b,k}
    \right\rangle
    }{
    \left\|
    \mathbf{h}^{(1)}_{a,k}
    \right\|_2
    \left\|
    \mathbf{h}^{(1)}_{b,k}
    \right\|_2
    }.
\end{equation}
The agreement with the embedding geometry is measured by
\begin{equation}
    \rho^{(t)}_{(1),k}
    =
    \operatorname{Corr}
    \left(
    \operatorname{vec}_{\triangle}
    \left(
    \mathbf{S}^{\mathrm{sig},(1)}_k
    \right),
    \operatorname{vec}_{\triangle}
    \left(
    \mathbf{S}^{\mathrm{emb},(t)}
    \right)
    \right).
\end{equation}
We use Pearson correlation. The heatmaps include epochs sampled every 100 epochs.

For the PCA visualizations, the selected late-stage frequencies are $k=20$
for the addition task and $k=19$
for the modular-addition task. These frequencies are selected because they maximize the correlation at the displayed epoch.

We analyze only the cosine components because the zero-frequency sine basis is identically zero and therefore cannot represent the required zero-frequency projection.

\subsubsection{Activation-Order Sweep}
\label{app:activation_order}

This experiment corresponds to Figure~4. We consider modular-addition sequences of length
$L\in\{2,3,4\}$,
and use the monomial activation
$\sigma(z)=z^{\odot K},
    K\in\{1,2,3,4,5,6\}$. For a sequence
\begin{equation}
    X=(x_1,\ldots,x_L),
    \qquad
    x_i\in\mathcal{X}_N,
\end{equation}
the target is
\begin{equation}
    y
    =
    1+\left(
    \sum_{i=1}^{L}x_i
    \right)
    \bmod N,
\end{equation}
where $N=50$. For each pair $(L,K)$, the training set contains all $N^L$ sequences. Each model is trained for 1000 epochs.

\subsection{Attention-Only Experiments}
\label{app:attention_experiments}

\subsubsection{Attention Architecture}
\label{app:attention_architecture}

We use $d=128, d_k=64$. The attention model uses causal masking. It contains learned positional embeddings. It contains no residual connection, no normalization, and no FFN module. The attention and output layers use no bias.

\subsubsection{Attention Arithmetic Experiments}
\label{app:attention_arithmetic}

The addition and modular-addition tasks in Figure~5A--H use $N=20$.
The input dataset contains all $N^2$ ordered pairs. The models are trained for 10000 epochs. The early and late epochs visualized in Figure~5 are 100 and 10000.

\subsubsection{Anchor--Noise Parity Task}
\label{app:parity_task}

This experiment corresponds to Figure~5I--L. We use the anchor-token set
\begin{equation}
    \mathcal{A}
    =
    \{11,12,\ldots,20\}
\end{equation}
and the noise-token set
\begin{equation}
    \mathcal{R}
    =
    \{21,22,\ldots,60\}.
\end{equation}

Each sequence has length 8.
A sequence contains one anchor token $a\in\mathcal{A}$ and $L-1$ noise tokens sampled from $\mathcal{R}$:
\begin{equation}
    X
    =
    (r_1,\ldots,r_{m-1},a,r_{m+1},\ldots,r_L).
\end{equation}
The anchor position $m$ is sampled uniformly from $\{1,\ldots,L\}$]. The anchor token is sampled uniformly from $\mathcal{A}$. Noise tokens are sampled independently with replacement from $\mathcal{R}$.

The training set is generated once with 10000 samples. The model is trained for 200 epochs.

For a sequence $X$, the attention mass assigned to the anchor token is
\begin{equation}
    A_{\mathrm{anchor}}(X)
    =
    \sum_{s=1}^{L}
    a_s\mathbf{1}_{\{x_s\in\mathcal{A}\}},
\end{equation}
and the attention mass assigned to the noise tokens is
\begin{equation}
    A_{\mathrm{noise}}(X)
    =
    \sum_{s=1}^{L}
    a_s\mathbf{1}_{\{x_s\in\mathcal{R}\}}.
\end{equation}
The curves in Figure~5K report the average over the full dataset of these quantities. For a single-anchor sequence, $A_{\mathrm{anchor}}(X)$ is the attention weight assigned to the anchor position.

The average embedding norms are computed as
\begin{equation}
    E_{\mathrm{anchor}}(t)
    =
    \frac{1}{|\mathcal{A}|}
    \sum_{\alpha\in\mathcal{A}}
    \left\|
    \vw^{E,(t)}_\alpha
    \right\|_2,
\end{equation}
and
\begin{equation}
    E_{\mathrm{noise}}(t)
    =
    \frac{1}{|\mathcal{R}|}
    \sum_{\alpha\in\mathcal{R}}
    \left\|
    \vw^{E,(t)}_\alpha
    \right\|_2.
\end{equation}

\subsection{Signature-Based Embedding Experiments}
\label{app:signature_embedding}

This section provides the details of the fixed signature-embedding experiment shown in Figure~7D.

\subsubsection{Fixed First-Order Signature Embedding}
\label{app:fixed_signature_embedding}

For every token $\alpha\in\mathcal{X}$, we first compute its empirical first-order label signature
\begin{equation}
    \boldsymbol{\varphi}^{y;1}_\alpha
    \in
    \mathbb{R}^{|\mathcal{Y}|\times|\mathcal{X}|}.
\end{equation}
The fixed signature embedding is constructed using a random linear projection:
\begin{equation}
    \vw^{E,\mathrm{sig}}_\alpha
    =
    \lambda \mathbf{R}\vvarphi^{y;1}_\alpha\vv,
\end{equation}
where
\begin{equation}
    \lambda\in\sR,\mathbf{R}
    \in
    \mathbb{R}^{d\times|\mathcal{Y}|}, \vv\in\sR^{|\fX|}.
\end{equation}
The entries of $\mathbf{R}$ and $\vv$ are sampled independently and fixed during training. The scale parameter $\lambda$ are trainable.

\subsubsection{Baselines}
\label{app:signature_embedding_baselines}

We compare the following two embedding settings.

\paragraph{Trainable embedding.}
The input embedding matrix is initialized according to the common initialization rule and optimized jointly with the remaining model parameters.

\paragraph{Fixed random embedding.}
Each token is assigned a random embedding
\begin{equation}
    \vw^{E,\mathrm{rand}}_\alpha=\lambda\vv_\alpha,
\end{equation}
where $\vv_\alpha\in\sR^{d}$ remains fixed throughout training and the scale parameter $\lambda$ is trainable.

\subsubsection{FFN and Attention Comparison in Figure~7}
\label{app:ffn_attention_comparison}

The experiment in Figure~7A--C uses the task $F_{\rm mod}$ with $L=2, N=40$.
The models are trained for 1000 epochs using identical optimization settings. The PCA projections in Figure~7B and Figure~7C use the final epochs at epoch 1000.

\subsection{LLMs}

\paragraph{Model Architecture}
We use a 0.7B dense decoder-only Transformer throughout the language-model experiments. The model follows a standard decoder-only Transformer architecture with a pre-norm design. Compared with the original GPT-3 architecture~\cite{brown2020language}, we replace LayerNorm with RMSNorm and the standard MLP with SwiGLU. The model consists of 24 Transformer layers with hidden size 1024, 16 attention heads, attention head dimension 96, and an MLP intermediate dimension of 4096. Rotary Position Embeddings (RoPE) are adopted for positional encoding, and the vocabulary size is 60,416. The model contains approximately 680M non-embedding parameters.

\paragraph{Training Methods}
We train the model using the standard next-token prediction objective with Megatron-LM~\cite{megatron-lm} and the AdamW optimizer. We use a learning-rate ratio $r_{\mathrm{lr}}=0.1$, a warm-up ratio $r_{\mathrm{warmup}}=0.03$, a peak learning rate of $\eta_{\mathrm{peak}}=2.5\times10^{-4}$, and a global batch size of $B=256$.

Training is performed using bfloat16 mixed precision, while gradient accumulation is carried out in fp32 precision. Both attention dropout and hidden dropout are disabled throughout training.

\paragraph{Data}
The experiments were conducted on a high-quality bilingual (Chinese–English) corpus containing 36B tokens and spanning multiple domains, including 13.5B web tokens, 9B Wikipedia tokens, 4.5B code tokens, and 9B mathematics tokens.

\section{Theoretical Details}
\subsection{Proof of Theorem~\ref{thm:grad_ffn_y},\ref{thm:grad_ffn_p}}
Recalling that
\begin{equation*}
\vf_{\mathrm{ffn}}(X):=\vW^U\vg_{\mathrm{ffn}}(X)=\vW^U\sigma\left(\sum_{i=1}^{L}\vw^E_{x_i}\right).
\end{equation*}
With the Assumption~\ref{assump:polynomial}, we have that
\begin{equation*}
\sigma'\left(\sum_{i=1}^L\vw^E_{x_i}\right)=\sum_{k=1}^{K}kC_k \left(\sum_{i=1}^L\vw^E_{x_i}\right)^{\odot(k-1)}.
\end{equation*}
Therefore,
\begin{equation*}
\left.\frac{d\vw^E_\alpha}{dt}\right|^y=Lr_\alpha\sum_{k=1}^{K}kC_k\vJ^y_{\alpha,k},\qquad \left.\frac{d\vw^E_\alpha}{dt}\right|^p=Lr_\alpha\sum_{k=1}^{K}kC_k\vJ^p_{\alpha,k},
\end{equation*}
where
\begin{equation*}
\vJ^y_{\alpha,k}:=\mathbb E\left[\left(\sum_{i=1}^L\vw^E_{x_i}\right)^{\odot(k-1)}\odot \vW^{U,T}\ve_y\mid \alpha\right],
\end{equation*}
and
\begin{equation*}
\vJ^p_{\alpha,k}:=\mathbb E\left[\left(\sum_{i=1}^L\vw^E_{x_i}\right)^{\odot(k-1)}\odot \vW^{U,T}\vp(X)\mid \alpha\right].
\end{equation*}

\paragraph{Label-driven term.}

Conditioned on $\alpha$, we write the as
\begin{equation*}
\sum_{i=1}^L\vw^E_{x_i}=\vw^E_\alpha+\sum_{j=1}^{L-1}\vw^E_{z_j}.
\end{equation*}
For the $k$-th polynomial term, the coordinate-wise multinomial expansion gives
\begin{align}
\left(\sum_{i=1}^L\vw^E_{x_i}\right)^{\odot(k-1)}=\sum_{m=0}^{\min\{L-1,k-1\}}\sum_{\substack{l\geq0,n_1,\cdots,n_m\geq1\\l+\sum_{r=1}^mn_r=k-1}}\frac{(k-1)!}{l!\prod_{j\in J}n_j!}(\vw^E_\alpha)^{\odot l}\odot\sum_{J=\left\{j_1,\cdots,j_m\right\}\subseteq[L-1]}\bigodot_{r=1}^{m}(\vw^E_{z_{j_r}})^{\odot n_r}.
\end{align}
Let $\vw^U_\nu:=\vW^{U,T}\ve_\nu\in\mathbb R^d$ denote the unembedding vector of token $\nu$. Substituting the above expansion into $\vJ^y_{\alpha,k}$ yields
\begin{align}
\vJ^y_{\alpha,k}=&\sum_{m=0}^{\min\{L-1,k-1\}}\sum_{\substack{l\geq0,n_1,\cdots,n_m\geq1\\l+\sum_{r=1}^mn_r=k-1}}\frac{(k-1)!}{l!\prod_{r=1}^mn_r!}(\vw^E_\alpha)^{\odot l}\odot \nonumber\\
&\quad \sum_{J=\left\{j_1,\cdots,j_m\right\}\subseteq[L-1]}\mathbb E\left[\vw^U_{\nu}\odot\bigodot_{r=1}^{m}(\vw^E_{z_{j_r}})^{\odot n_r}\mid \alpha\right].
\end{align}
For $J=\{j_1,\ldots,j_m\}$, it's noted that
\begin{equation*}
    \mathbb E\left[\vw^U_{\nu}\odot\bigodot_{r=1}^{m}(\vw^E_{z_{j_r}})^{\odot n_r}\mid \alpha\right] = \sum_{\nu,\beta_r\in\fV}\mathbb P\left(y=\nu,z_{j_r}=\beta_r\mid\alpha\right)\vw^U_{\nu}\odot\bigodot_{r=1}^{m}(\vw^E_{\beta_r})^{\odot n_r}.
\end{equation*}
After summing over all exponent assignments, the resulting kernel is symmetric in its $m$ context indices. Hence averaging over ordered distinct tuples is equivalent to averaging over unordered subsets, and the sum over the $\binom{L-1}{m}$ subsets can be written as
\begin{align*}
    \vJ^y_{\alpha,k}
    =\sum_{m=0}^{\min\{L-1,k-1\}}
    \binom{L-1}{m}
    \sum_{\nu,\beta_1,\ldots,\beta_m\in\fV}
    \varphi_\alpha^{y;m}(\nu,\beta_1,\ldots,\beta_m)\\
    \qquad\times
    \left[
    \sum_{\substack{l\geq0,n_1,\ldots,n_m\geq1\\l+\sum_{r=1}^mn_r=k-1}}
    \frac{(k-1)!}{l!\prod_{r=1}^{m}n_r!}
    (\vw^E_\alpha)^{\odot l}\odot\vw^U_\nu\odot
    \bigodot_{r=1}^{m}(\vw^E_{\beta_r})^{\odot n_r}
    \right].
\end{align*}
Define that $\fW^{y,m,k}_{\alpha}\in\left(\sR^{d}\right)^{\underbrace{V\times\cdots\times V}_{m+1}}$ is a tensor with order $m+1$, whose elements is a $d$-dimensional vector with
\begin{equation*}
    \fW^{y,m,k}_{\alpha}\left(\nu,\beta_1,\cdots,\beta_m\right)=\binom{L-1}{m}\sum_{\substack{l\geq0,n_1,\cdots,n_m\geq1\\l+\sum_{r=1}^mn_r=k-1}}\frac{(k-1)!}{l!\prod_{r=1}^{m}n_r!}(\vw^E_\alpha)^{\odot l}\odot\vw^U_{\nu}\odot\bigodot_{r=1}^{m}(\vw^E_{\beta_r})^{\odot n_r},
\end{equation*}
then we have that
\begin{align*}
    \vJ^y_{\alpha,k} = \sum_{m=0}^{\min\left\{L-1,k-1\right\}} \fW^{y,m,k}_{\alpha}\cdot \vvarphi_{\alpha}^{y;m}, \quad \left.\frac{d\vw^E_\alpha}{dt}\right|^y = Lr_\alpha\sum_{k=1}^{K}kC_k\vJ_{\alpha,k}^y=Lr_\alpha\sum_{m=0}^{\min\left\{L-1,K-1\right\}} \fA_{\alpha,m}\cdot\vvarphi_{\alpha}^{y;m},
\end{align*}
which adimits the Theorem~\ref{thm:grad_ffn_y}.

\paragraph{Prediction-induced term.}

We next derive the operator associated with the prediction-induced term:
\begin{equation*}
\vJ^p_{\alpha,k}=\mathbb E\left[\left(\sum_{i=1}^L\vw^E_{x_i}\right)^{\odot(k-1)}\odot \vW^{U,T}\vp(X)\mid \alpha\right].
\end{equation*}
At the early stage of training, the logits are small. Therefore, we use the first-order softmax expansion around the origin: $\vp(X)=\frac{1}{V}\vone+O(\|\vf(X)\|)$
Since $\vf(X)=\vW^U\sigma(\vh(X))$, we obtain
\begin{equation*}
\vW^{U,T}\vp(X)=\vb_0+O(\|\vf(X)\|),
\end{equation*}
where $\vb_0:=\frac{1}{V}\vW^{U,T}\vone$. Using the same multinomial expansion of $\vh(X)^{\odot(k-1)}$, we get
\begin{align*}
\vJ^{p}_{\alpha,k}=&\sum_{m=0}^{\min\{L-1,k-1\}}\sum_{\substack{l\geq 0,\ n_1,\cdots,n_m\geq 1\\ l+\sum_{r=1}^mn_r=k-1}}\sum_{\substack{J=\left\{j_1,\cdots,j_m\right\}\subseteq[L-1]}}\frac{(k-1)!}{l!\prod_{r=1}^mn_r!} \nonumber\\
&\quad \mathbb E\left[\vb_0\odot(\vw^E_\alpha)^{\odot l}\odot\bigodot_{r=1}^m(\vw^E_{z_{j_r}})^{\odot n_r}\mid \alpha\right].
\end{align*}
For $J=\{j_1,\ldots,j_m\}$
\begin{align*}
&\mathbb E\left[\vb_0\odot(\vw^E_\alpha)^{\odot l}\odot\bigodot_{r=1}^m(\vw^E_{z_{j_r}})^{\odot n_r}\mid \alpha\right]\nonumber\\
=&\sum_{\beta_1,\ldots,\beta_m\in\fV}\mathbb P\left(z_{j_r}=\beta_r\mid\alpha\right)\vb_0\odot(\vw^E_\alpha)^{\odot l}\odot\bigodot_{r=1}^m(\vw^E_{\beta_r})^{\odot n_r} .
\end{align*}
After the same symmetrization over context positions, this becomes
\begin{align*}
    \vJ^{p}_{\alpha,k}=&\sum_{m=0}^{\min\{L-1,k-1\}}\binom{L-1}{m}
    \sum_{\beta_1,\ldots,\beta_m\in\fV}
    \varphi_\alpha^m(\beta_1,\ldots,\beta_m)\\
    &\left[\sum_{\substack{l\geq 0,\ n_1,\ldots,n_m\geq 1\\ l+\sum_{r=1}^mn_r=k-1}}\frac{(k-1)!}{l!\prod_{r=1}^mn_r!}\vb_0\odot(\vw^E_\alpha)^{\odot l}\odot\bigodot_{r=1}^m(\vw^E_{\beta_r})^{\odot n_r}\right].
\end{align*}
Therefore, define $\mathcal U^{m,k}_{\alpha}\in\left(\sR^d\right)^{\underbrace{V\times\cdots\times V}_m}$ is a tensor with order $m$, whose element is a $d$-dimensional vector with
\begin{equation*}
\mathcal U^{m,k}_{\alpha}\left(\beta_1,\ldots,\beta_m\right):=\binom{L-1}{m}\sum_{\substack{l\geq 0,\ n_1,\ldots,n_m\geq 1\\ l+\sum_{r=1}^mn_r=k-1}}\frac{(k-1)!}{l!\prod_{r=1}^mn_r!}\vb_0\odot(\vw^E_\alpha)^{\odot l}\odot\bigodot_{r=1}^m(\vw^E_{\beta_r})^{\odot n_r}.
\end{equation*}
Then we have
\begin{equation*}
    \vJ_{\alpha,k}^{p} = \sum_{m=0}^{\min\left\{L-1,k-1\right\}} \fU_{\alpha}^{m,k}\cdot \vvarphi_{\alpha}^m,\quad
    \left.\frac{d\vw^E_\alpha}{dt}\right|^p = Lr_\alpha\sum_{m=1}^{\min\left\{L-1,K-1\right\}}\fB_{\alpha}^{m}\cdot\vvarphi_{\alpha}^{m}+\vR_{\alpha}.
\end{equation*}
where $\fB_{\alpha,m}=\sum_{k=1}^KkC_k\fU_{\alpha}^{m,k}$.

\subsection{Supplementary of Remark 1}\label{app:second_order_softmax}
In this section, we will further discuss the formulation of $\left.\frac{d\vw^E_\alpha}{dt}\right|^p$ in $\vf_{\rm ffn}$, with an expansion to the linear term of the softmax function and we will show that more higher-order signatures are induced. Recall that:
\begin{equation*}
\vJ^p_{\alpha,k}=\mathbb E\left[\left(\sum_{i=1}^L\vw^E_{x_i}\right)^{\odot(k-1)}\odot \vW^{U,T}\vp(X)\mid \alpha\right].
\end{equation*}
At the early stage of training, we use the first-order softmax expansion around the origin:
\begin{equation*}
\vp(X)=\frac{1}{V}\vone+\frac{1}{V}\vH\vf(X)+O(\|\vf(X)\|^2),
\end{equation*}
where
\begin{equation*}
\vH:=\vI_V-\frac{1}{V}\vone\vone^T .
\end{equation*}
Since $\vf(X)=\vW^U\sigma(\vh(X))$, we obtain
\begin{equation*}
\vW^{U,T}\vp(X)=\vb_0+\frac{1}{V}\vM\sigma\left(\sum_{i=1}^L\vw^E_{x_i}\right)+O(\|\vf(X)\|^2),
\end{equation*}
where
\begin{equation*}
\vb_0:=\frac{1}{V}\vW^{U,T}\vone,\qquad \vM:=\vW^{U,T}\vH\vW^U .
\end{equation*} Thus $\vJ^p_{\alpha,k}$ decomposes into a uniform-prediction part and a linearized-logit part:
\begin{equation*}
\vJ^p_{\alpha,k}=\vJ^{p,0}_{\alpha,k}+\vJ^{p,1}_{\alpha,k}+\vR^p_{\alpha,k},
\end{equation*}
where
\begin{equation*}
\vJ^{p,0}_{\alpha,k}:=\mathbb E\left[\left(\sum_{i=1}^L\vw^E_{x_i}\right)^{\odot(k-1)}\odot \vb_0\mid \alpha\right],
\end{equation*}
\begin{equation*}
\vJ^{p,1}_{\alpha,k}:=\frac{1}{V}\mathbb E\left[\left(\sum_{i=1}^L\vw^E_{x_i}\right)^{\odot(k-1)}\odot \vM\sigma(\vh(X))\mid \alpha\right],
\end{equation*}
and $\vR^p_{\alpha,k}$ contains the higher-order softmax remainder. The analysis on $\vJ^{p,0}_{\alpha,k}$ has been completed in the previous section.

\paragraph{Linearized-logit part.}

Using $\sigma\left(\sum_{i=1}^L\vw^E_{x_i}\right)=\sum_{q=1}^{K}C_q\left(\sum_{i=1}^L\vw^E_{x_i}\right)^{\odot q}$, we have
\begin{equation*}
\vJ^{p,1}_{\alpha,k}=\frac{1}{V}\sum_{q=1}^{K}C_q\mathbb E\left[\left(\sum_{i=1}^L\vw^E_{x_i}\right)^{\odot(k-1)}\odot \vM\left(\sum_{i=1}^L\vw^E_{x_i}\right)^{\odot q}\mid \alpha\right].
\end{equation*}
It's noted that
\begin{align*}
    &\mathbb E\left[\left(\sum_{i=1}^L\vw^E_{x_i}\right)^{\odot(k-1)}\odot \vM\left(\sum_{i=1}^L\vw^E_{x_i}\right)^{\odot q}\mid \alpha\right]\\
   =& \sum_{m=0}^{\min\left\{L-1,k-1+q\right\}}\sum_{\substack{l_1,l_2\geq0,n_{1}^{(1)},\ldots,n_{m_1}^{(1)},n_{1}^{(2)},\ldots,n_{m_2}^{(2)}\geq1\\m_1+m_2=m\\l_1+l_2+\sum_{r=1}^{m_1}n^{(1)}_r+\sum_{r=1}^{m_2}n^{(2)}_r=k-1+q}}\sum_{\substack{J_1=\left\{j_{1}^{(1)},\ldots,j_{m_1}^{(1)}\subseteq[L-1]\right\}\\J_2=\left\{j_{1}^{(1)},\ldots,j_{m_2}^{(2)}\subseteq[L-1]\right\}\\J_1\cap J2=\emptyset}}\frac{(k-1)!}{l_1!\prod_{r=1}^{m_1}n^{(1)}_r!}\frac{q!}{l_2!\prod_{r=1}^{m_2}n^{(2)}_r!}\\\quad &\mathbb E\left[(\vw^E_{\alpha})^{l_1}\odot\bigodot_{r=1}^{m_1}(\vW^E_{z_{j^{(1)}_r}})^{\odot n^{(1)}_r}\odot\vM\left((\vw^E_{\alpha})^{l_2}\odot\bigodot_{r=1}^{m_2}(\vW^E_{z_{j^{(2)}_r}})^{\odot n^{(2)}_r}\right)\mid\alpha\right]
\end{align*}
since
\begin{align*}
    &\mathbb E\left[(\vw^E_{\alpha})^{l_1}\odot\bigodot_{r=1}^{m_1}(\vW^E_{z_{j^{(1)}_r}})^{\odot n^{(1)}_r}\odot\vM\left((\vw^E_{\alpha})^{l_2}\odot\bigodot_{r=1}^{m_2}(\vW^E_{z_{j^{(2)}_r}})^{\odot n^{(2)}_r}\right)\mid\alpha\right]\\
    =&\sum_{\beta_1,\beta_2,\ldots,\beta_m\in\fV}\mathbb P\left(z_{j^{(1)}_1}=\beta_1,\ldots,z_{j^{(1)}_{m_1}}=\beta_{m_1},z_{j^{(2)}_1}=\beta_{m_1+1},\ldots,z_{j^{(2)}_{m_2}}=\beta_m\mid \alpha\right)\\
    &\qquad (\vw^E_{\alpha})^{l_1}\odot\bigodot_{r=1}^{m_1}(\vW^E_{\beta_r})^{\odot n^{(1)}_r}\odot\vM\left((\vw^E_{\alpha})^{l_2}\odot\bigodot_{r=1}^{m_2}(\vW^E_{\beta_{m_1+r}})^{\odot n^{(2)}_r}\right),
\end{align*}
Define that $\fU_{\alpha,1}^{m,k}\in \left(\sR^d\right)^{\underbrace{V\times\ldots\times V}_m}$ with element
\begin{align*}
    \fU_{\alpha,1}^{m,k}\left(\beta_1,\ldots,\beta_m\right) = &\binom{L-1}{m}\frac{1}{V}\sum_{q=\max\left\{1,m-k+1\right\}}^K\sum_{\substack{l_1,l_2\geq0,n_{1}^{(1)},\ldots,n_{m_1}^{(1)},n_{1}^{(2)},\ldots,n_{m_2}^{(2)}\geq1\\m_1+m_2=m\\l_1+l_2+\sum_{r=1}^{m_1}n^{(1)}_r+\sum_{r=1}^{m_2}n^{(2)}_r=k-1+q}}\frac{(k-1)!}{l_1!\prod_{r=1}^{m_1}n^{(1)}_r!}\frac{q!}{l_2!\prod_{r=1}^{m_2}n^{(2)}_r!}\\&(\vw^E_{\alpha})^{l_1}\odot\bigodot_{r=1}^{m_1}(\vW^E_{\beta_r})^{\odot n^{(1)}_r}\odot\vM\left((\vw^E_{\alpha})^{l_2}\odot\bigodot_{r=1}^{m_2}(\vW^E_{\beta_{m_1+r}})^{\odot n^{(2)}_r}\right) \in\sR^d,
\end{align*}
then we have 
\begin{equation*}
    \vJ_{\alpha,k}^{p,1}=\sum_{m=0}^{\min\left\{L-1,k-1+K\right\}}\fU_{\alpha,1}^{m,k}\cdot \vvarphi_{\alpha}^{m}
\end{equation*}

Combining the uniform-prediction and linearized-logit parts, we write
\begin{equation*}
\vJ^p_{\alpha,k}=\sum_{m=0}^{\min\left\{L-1,k-1\right\}}\fU_{\alpha,0}^{m,k}\cdot\vvarphi^m_\alpha+\sum_{m=0}^{\min\left\{L-1,k-1+K\right\}}\fU_{\alpha,1}^{m,k}\vvarphi_{\alpha}^m+\vR_{\alpha}^{p,k}.
\end{equation*}
Summation of $k$, we obtain
\begin{align*}
\left.\frac{d\vw^E_\alpha}{dt}\right|^p&=Lr_\alpha\sum_{k=1}^{K}kC_k\left(\sum_{m=0}^{\min\left\{L-1,k-1\right\}}\fU_{\alpha,0}^{m,k}\cdot\vvarphi^m_\alpha+\sum_{m=0}^{\min\left\{L-1,k-1+K\right\}}\fU_{\alpha,1}^{m,k}\vvarphi_{\alpha}^m+\vR_{\alpha}^{p,k}\right)\\
&=Lr_\alpha\left(\sum_{m=0}^{\min\left\{L-1,K-1\right\}}
\fB_{\alpha}^m\cdot\vvarphi_\alpha^m + \sum_{m=0}^{\min\left\{L-1,2K-1\right\}}
\fC_{\alpha}^m\cdot\vvarphi_\alpha^m\right)+\fR_\alpha.
\end{align*}
where $\fC_{\alpha,m}=\sum_{k=1}^KkC_k\fU_{\alpha,1}^{m,k}$.

\subsection{Proof of Proposition~\ref{prop:initial_zero}}

\begin{lemma}
\label{lem:signature_consistency}
For any token $\alpha$ with $r_\alpha>0$ and any $1\leq m\leq L-1$, the label signatures satisfy
\begin{equation*}
\sum_{\nu,\beta_1,\ldots,\beta_m\in\fV}
\varphi_\alpha^{y;m}(\nu,\beta_1,\ldots,\beta_m)=1.
\end{equation*}
Moreover, for every context mode $q\in[m]$,
\begin{align*}
\sum_{\beta_q\in\fV}
\varphi_\alpha^{y;m}(\nu,\beta_1,\ldots,\beta_m)=
\varphi_\alpha^{y;m-1}
(\nu,\beta_1,\ldots,\beta_{q-1},\beta_{q+1},\ldots,\beta_m).
\end{align*}
Consequently,
\begin{equation*}
\|\boldsymbol{\varphi}^{y;m}_\alpha\|_F
\leq
\|\boldsymbol{\varphi}^{y;m-1}_\alpha\|_F
\leq\cdots\leq
\|\boldsymbol{\varphi}^{y;0}_\alpha\|_F.
\end{equation*}
The same normalization, marginal-consistency, and norm-monotonicity properties hold for the co-occurrence signatures after omitting the label index.
\end{lemma}

\begin{proof}
By definition, $\boldsymbol{\varphi}^{y;m}_\alpha$ is the joint distribution of the label and the tokens at an ordered tuple of $m$ distinct context positions sampled uniformly from $\mathcal I_m$. Its entries are therefore nonnegative and sum to one.

Fix $q\in[m]$. After deleting the $q$-th coordinate from a uniformly sampled tuple in $\mathcal I_m$, the remaining ordered $(m-1)$-tuple is uniform on $\mathcal I_{m-1}$. Indeed, every fixed element of $\mathcal I_{m-1}$ has exactly $L-m$ preimages in $\mathcal I_m$, and
\begin{equation*}
(L-1)_m=(L-1)_{m-1}(L-m).
\end{equation*}
Summing over the token value at the deleted coordinate therefore gives precisely the order-$(m-1)$ marginal, proving the stated consistency relation.

To prove the norm inequality, fix all indices except $\beta_q$ and use nonnegativity:
\begin{equation*}
\sum_{\beta_q\in\fV}
\varphi_\alpha^{y;m}(\nu,\beta_1,\ldots,\beta_m)^2
\leq
\left(
\sum_{\beta_q\in\fV}
\varphi_\alpha^{y;m}(\nu,\beta_1,\ldots,\beta_m)
\right)^2.
\end{equation*}
Summing this inequality over all remaining indices and applying the marginal-consistency identity yields
\begin{equation*}
\|\boldsymbol{\varphi}^{y;m}_\alpha\|_F^2
\leq
\|\boldsymbol{\varphi}^{y;m-1}_\alpha\|_F^2.
\end{equation*}
Iterating over $m$ proves the result. The co-occurrence case is identical.
\end{proof}

\begin{proof}[Proposition~\ref{prop:initial_zero}]
We first estimate the size of each contracted order-$m$ operator at initialization. Recall that the order-$m$ label-driven operator can be decomposed according to the polynomial degree:
\begin{equation*}
\fA_{\alpha,m}=\sum_{k=m+1}^{K}kC_k\fW_{\alpha}^{y;m,k}.
\end{equation*}
The lowest possible degree is $k=m+1$. For fixed $m$ and $k$, each entry of $\fW_{\alpha}^{y;m,k}$ is a finite linear combination of coordinate-wise products containing one unembedding factor and $k-1$ embedding factors. Since each factor has coordinate scale $d^{-\gamma}$, each entry has second moment at most of order $d^{-2k\gamma}$. The numbers of terms in these finite sums are bounded because $L$ and $K$ are fixed.

Let $\boldsymbol{\psi}_m$ be any deterministic tensor with the same shape as the order-$m$ signature. For each coordinate $a\in[d]$, the $a$-th coordinate of $\fW_{\alpha}^{y;m,k}\cdot\boldsymbol{\psi}_m$ is a weighted sum of entries of $\fW_{\alpha}^{y;m,k}$. By the above moment estimate and Cauchy's inequality,
\begin{equation*}
\mathbb E\left\|\left[\fW_{\alpha}^{y;m,k}\cdot\boldsymbol{\psi}_m\right]_a\right|^2\lesssim d^{-2k\gamma}\|\boldsymbol{\psi}_m\|_F^2 .
\end{equation*}
Summing over $a\in[d]$ gives
\begin{equation*}
\left(\mathbb E\left\|\fW_{\alpha}^{y;m,k}\cdot\boldsymbol{\psi}_m\right\|_2^2\right)^{1/2}\lesssim d^{1/2-k\gamma}\|\boldsymbol{\psi}_m\|_F .
\end{equation*}
Since $k\geq m+1$, the largest contribution at initialization comes from the lowest-degree term $k=m+1$, while all higher-degree terms contain additional factors of size $d^{-\gamma}$. Hence
\begin{equation*}
\left(\mathbb E\left\|\fA_{\alpha,m}\cdot\boldsymbol{\psi}_m\right\|_2^2\right)^{1/2}\lesssim d^{1/2-(m+1)\gamma}\|\boldsymbol{\psi}_m\|_F .
\end{equation*}

For $m=0$, the leading term is the degree-one term:
\begin{equation*}
\fA_{\alpha,0}^{\mathrm{lead}}=C_1\fW_{\alpha}^{y;0,1}.
\end{equation*}
For any deterministic vector $\boldsymbol{\psi}_0$, this term satisfies
\begin{equation*}
\left[\fW_{\alpha}^{y;0,1}\cdot\boldsymbol{\psi}_0\right]_a=\sum_{\ell\in\mathcal V}\boldsymbol{\psi}_0(\ell)\vw^U_\ell(a).
\end{equation*}
Since the unembedding coordinates are independent with variance of order $d^{-2\gamma}$, we have
\begin{equation*}
\mathbb E\left\|\left[\fW_{\alpha}^{y;0,1}\cdot\boldsymbol{\psi}_0\right]_a\right|^2\asymp d^{-2\gamma}\|\boldsymbol{\psi}_0\|_F^2 .
\end{equation*}
After summing over $a\in[d]$, this yields
\begin{equation*}
\left(\mathbb E\left\|\fA_{\alpha,0}\cdot\boldsymbol{\psi}_0\right\|_2^2\right)^{1/2}\asymp d^{1/2-\gamma}\|\boldsymbol{\psi}_0\|_F ,
\end{equation*}
where the higher-degree terms in $\fA_{\alpha,0}$ are smaller by additional powers of $d^{-\gamma}$.

Now take $\boldsymbol{\psi}_m=\boldsymbol{\varphi}^{y;m}_\alpha$. For every fixed $m\geq 1$, the ratio between the root-mean-square scales of the order-$m$ and order-$0$ contractions is bounded by
\begin{equation*}
\frac{\left(\mathbb E\left\|\fA_{\alpha,m}\cdot\boldsymbol{\varphi}^{y;m}_\alpha\right\|_2^2\right)^{1/2}}{\left(\mathbb E\left\|\fA_{\alpha,0}\cdot\boldsymbol{\varphi}^{y;0}_\alpha\right\|_2^2\right)^{1/2}}\lesssim d^{-m\gamma}\frac{\|\boldsymbol{\varphi}^{y;m}_\alpha\|_F}{\|\boldsymbol{\varphi}^{y;0}_\alpha\|_F}.
\end{equation*}
By Lemma~\ref{lem:signature_consistency},
\begin{equation*}
\frac{\|\boldsymbol{\varphi}^{y;m}_\alpha\|_F}{\|\boldsymbol{\varphi}^{y;0}_\alpha\|_F}\leq 1.
\end{equation*}
The factor $\binom{L-1}{m}$ contained in $\fA_{\alpha,m}$ is independent of $d$ because $L$ is fixed. Therefore,
\begin{equation*}
\frac{\left(\mathbb E\left\|\fA_{\alpha,m}\cdot\boldsymbol{\varphi}^{y;m}_\alpha\right\|_2^2\right)^{1/2}}{\left(\mathbb E\left\|\fA_{\alpha,0}\cdot\boldsymbol{\varphi}^{y;0}_\alpha\right\|_2^2\right)^{1/2}}\lesssim d^{-m\gamma}=o(1).
\end{equation*}
Thus every fixed higher-order label-signature contraction is asymptotically smaller than the zero-order contraction at initialization.

\end{proof}

\subsection{Proof of Proposition~\ref{prop:late_stage}}
For each signature order $m$, the label-driven contribution can be decomposed according to the polynomial degree:
\begin{equation*}
\fA_{\alpha,m}(t)\cdot\boldsymbol{\varphi}^{y;m}_\alpha=\sum_{k=m+1}^{K}T_{\alpha,m,k}(t),\qquad T_{\alpha,m,k}(t):=kC_k\fW_{\alpha}^{y;m,k}(t)\cdot\boldsymbol{\varphi}^{y;m}_\alpha .
\end{equation*}
We write
\begin{equation*}
\vw^E_\beta(t)=s_E(t)\bar{\vw}^E_\beta(t),\qquad \vw^U_\nu(t)=s_U(t)\bar{\vw}^U_\nu(t).
\end{equation*}
The degree-$k$ term contains one unembedding factor and $k-1$ embedding factors. Hence, under the assumed non-degeneracy of normalized contractions, its contracted norm has the scale
\begin{equation*}
\|T_{\alpha,m,k}(t)\|_2\asymp |kC_k|\widetilde B_m(k)s_U(t)s_E(t)^{k-1}\|\boldsymbol{\varphi}^{y;m}_\alpha\|_F ,
\end{equation*}
where the combinatorial coefficient is
\begin{equation*}
B_m(k):=\sum_{\substack{h>0,\ n_1,\ldots,n_m\geq 1\\ h+\sum_{r=1}^{m}n_r=k-1}}\frac{(k-1)!}{h!\prod_{r=1}^{m}n_r!},
\qquad
\widetilde B_m(k):=\binom{L-1}{m}B_m(k).
\end{equation*}
Here $\widetilde B_m(k)$ is independent of $t$ and incorporates the deterministic position-averaging factor induced by the normalized signature definition.

We first compare different degrees for a fixed $m$. For any $k<K$, we have
\begin{equation*}
\frac{\|T_{\alpha,m,k}(t)\|_2}{\|T_{\alpha,m,K}(t)\|_2}\lesssim \frac{|kC_k|\widetilde B_m(k)}{|KC_K|\widetilde B_m(K)}s_E(t)^{k-K}.
\end{equation*}
Since $k-K<0$, this ratio converges to zero when $s_E(t)$ becomes sufficiently large. Therefore, for every fixed $m$, the highest-degree term $T_{\alpha,m,K}(t)$ dominates all lower-degree terms. Because there are only finitely many $m=0,\ldots,L-1$, there exists a common threshold $S_*$ such that whenever $s_E(t)\geq S_*$,
\begin{equation*}
\left\|\fA_{\alpha,m}(t)\cdot\boldsymbol{\varphi}^{y;m}_\alpha\right\|_2\asymp \|T_{\alpha,m,K}(t)\|_2,\qquad 0\leq m\leq L-1 .
\end{equation*}
Since $K$ and $L$ are fixed, the coefficients $|KC_K|\widetilde B_m(K)$ are independent of $t$. Together with the uniform non-degeneracy of the normalized contractions and signatures, this gives
\begin{equation*}
\left\|\fA_{\alpha,m}(t)\cdot\boldsymbol{\varphi}^{y;m}_\alpha\right\|_2
\asymp s_U(t)s_E(t)^{K-1},
\qquad 0\leq m\leq L-1.
\end{equation*}
Hence all signature orders have the same late-stage scaling in the parameter magnitudes. Unlike at initialization, their hierarchy is no longer determined by different powers of $s_E(t)$.

It remains only to compare the order-dependent prefactors of these same-scale terms. At degree $K$, their relative magnitudes are governed by the combinatorial coefficient $\widetilde B_m(K)$ together with the normalized contractions and signature factors. Since the factor $\binom{L-1}{m}$ is independent of $K$, the exponential dependence on $K$ is determined by $B_m(K)$, which satisfies
\begin{equation*}
B_m(K)=(m+1)!\left\{{K-1\atop m+1}\right\},
\end{equation*}
where $\left\{{n\atop r}\right\}$ is the Stirling number of the second kind. For fixed $m$ and large $K$,
\begin{equation*}
B_m(K)\sim (m+1)^{K-1}.
\end{equation*}
Therefore, for any $m<L-1$,
\begin{equation*}
\frac{\widetilde B_{L-1}(K)}{\widetilde B_m(K)}
\sim
\frac{1}{\binom{L-1}{m}}
\left(\frac{L}{m+1}\right)^{K-1}
\to\infty,\qquad K\to\infty .
\end{equation*}
Since the label signatures and normalized contractions are uniformly non-degenerate, choosing $K$ sufficiently large gives
\begin{equation*}
\|T_{\alpha,L-1,K}(t)\|_2\geq C\|T_{\alpha,m,K}(t)\|_2,\qquad 0\leq m<L-1,
\end{equation*}
for a sufficiently large constant $C>1$.

Having established that all signature orders share the same late-stage scale, we now compare their prefactors. When $s_E(t)\geq S_*$ and $K$ is sufficiently large, the highest-order prefactor is the largest. Therefore, the full highest-order contribution is the largest among the same-scale signature-order contributions:
\begin{equation*}
\left\|\fA_{\alpha,L-1}(t)\cdot\boldsymbol{\varphi}^{y;L-1}_\alpha\right\|_2\geq \left\|\fA_{\alpha,m}(t)\cdot\boldsymbol{\varphi}^{y;m}_\alpha\right\|_2,\qquad 0\leq m<L-1 .
\end{equation*}
Therefore, once $s_E(t)\geq S_*$ and $K$ is sufficiently large, the order-$(L-1)$ term is the largest among all signature contributions. This completes the proof.

\subsection{Proof of Theorem~\ref{thm:attention}}
We first compute the value-path contribution.

\begin{align}
    \left.\frac{d\vw^E_\alpha}{dt}\right|_{\mathrm{value}}^y=\mathbb E\left[\sum_{s=1}^{L}a_s\mathbf 1\{x_s=\alpha\}\vW^{OV,T}\vW^{U,T}\ve_y\right].
\end{align}

Define the total attention mass received by token $\alpha$ as
\begin{equation*}
    A_{\alpha}(X):=\sum_{s=1}^{L}a_s\mathbf 1\{x_s=\alpha\}.
\end{equation*}
Then we have
\begin{equation*}
\left.\frac{d\vw^E_\alpha}{dt}\right|_{\mathrm{value}}^y=Lr_\alpha\vW'^{T}\mathbb E\left[a_I\ve_y\mid \alpha\right].
\end{equation*}
We recall the label signature
\begin{equation*}
\vvarphi^{y;0}_\alpha(\nu):=\mathbb P(y=\nu\mid \alpha),
\end{equation*}
and the attention-weighting vector
\begin{equation*}
\vpsi^y_{\alpha}(\nu):=\mathbb E[a_I\mid y=\nu,\alpha].
\end{equation*}
Then the attention-weighted label distribution satisfies
\begin{equation*}
\mathbb E[a_I\ve_y\mid \alpha]=\vvarphi^{y;0}_\alpha\odot\vpsi^y_{\alpha}.
\end{equation*}
Therefore,
\begin{equation*}
\left.\frac{d\vw^E_\alpha}{dt}\right|_{\mathrm{value}}^y=Lr_\alpha\vW'^{T}(\vvarphi^{y;0}_\alpha\odot\vpsi^y_{\alpha}).
\end{equation*}
This is the first signature term. It has the same form as a label-signature update, except that the label signature is reweighted by the attention mass assigned to token $\alpha$.

We next compute the contribution from the derivative of the attention scores. Define the attention-weighted value average by
\begin{equation*}
\bar{\vw}^E_X:=\sum_{s=1}^{L}a_s\vw^E_{x_s} .
\end{equation*}
Then
\begin{equation*}
\vg_{\mathrm{attn}}(X)=\vW^{OV}\bar{\vw}^E_X.
\end{equation*}
The softmax derivative satisfies
\begin{equation*}
\frac{\partial a_s}{\partial z_u}=a_s(\mathbf 1\{s=u\}-a_u).
\end{equation*}
Therefore,
\begin{align}
\frac{\partial \bar{\vw}^E_{X}}{\partial z_u}&=\sum_{s=1}^{L}\frac{\partial a_s}{\partial z_u}\vw^E_{x_s} \nonumber\\
&=\sum_{s=1}^{L}a_s(\mathbf 1\{s=u\}-a_u)\vw^E_{x_s}. \nonumber\\
\end{align}
The score-path contribution to the embedding gradient is therefore
\begin{equation*}
\left.\frac{d\vw^E_\alpha}{dt}\right|_{\mathrm{score}}^y=\mathbb E\left[\sum_{u=1}^{L}
\left\langle \vW^{OV,T}\vW^{U,T}\ve_y,\sum_{s=1}^{L}a_s(\mathbf 1\{s=u\}-a_u)\vw^E_{x_s}\right\rangle \frac{\partial z_u}{\partial \vw^E_\alpha}\right].
\end{equation*}

It remains to compute $\partial z_u/\partial \vw^E_\alpha$. 
the embedding $\vw^E_\alpha$ affects $z_u$ in two ways. If $x_L=\alpha$, it affects the last query vector. If $x_u=\alpha$, it affects the key vector at position $u$. Therefore,
\begin{equation*}
\frac{\partial z_u}{\partial \vw^E_\alpha}=\frac{1}{\sqrt d}\mathbf 1\{x_u=\alpha\}\vW^{QK}\vw^E_{x_L}+\frac{1}{\sqrt d}\mathbf 1\{x_L=\alpha\}\vW^{QK,T}\vw^E_{x_u}.
\end{equation*}
Substituting this into the score-path contribution gives
\begin{align}
\left.\frac{d\vw^E_\alpha}{dt}\right|_{\mathrm{score}}^y=&\frac{1}{\sqrt d}\mathbb E\left[\sum_{u=1}^{L}
\left\langle \vW^{OV,T}\vW^{U,T}\ve_y,\sum_{s=1}^{L}a_s(\mathbf 1\{s=u\}-a_u)\vw^E_{x_s}\right\rangle\mathbf 1\{x_u=\alpha\}\vW^{QK}\vw^E_{x_L} \right] \nonumber\\
&+\frac{1}{\sqrt d}\mathbb E\left[\sum_{u=1}^{L}
\left\langle \vW^{OV,T}\vW^{U,T}\ve_y,\sum_{s=1}^{L}a_s(\mathbf 1\{s=u\}-a_u)\vw^E_{x_s}\right\rangle\mathbf 1\{x_L=\alpha\}\vW^{QK,T}\vw^E_{x_u}\right].
\end{align}
We call the first term the key-score contribution and the second term the query-score contribution:
\begin{equation*}
\left.\frac{d\vw^E_\alpha}{dt}\right|_{\mathrm{score}}^y=\left.\frac{d\vw^E_\alpha}{dt}\right|^y_{\rm key}+\left.\frac{d\vw^E_\alpha}{dt}\right|^y_{\rm query}.
\end{equation*}

Consider the term $\left.\frac{d\vw^E_\alpha}{dt}\right|^y_{\rm key}$, we have that
\begin{align*}
    \left.\frac{d\vw^E_\alpha}{dt}\right|^y_{\rm key} &= \frac{1}{\sqrt{d}}\mathbb E\left[ \sum_{u=1}^{L} \vW^{QK}\vw^E_{x_L}\vw^{U,T}_{\nu}\vW^{OV}\left(\sum_{s=1}^{L}a_s(\mathbf 1\{s=u\}-a_u)\vw^E_{x_s}\right)\mathbf 1\{x_u=\alpha\}\right]\\
    &= \frac{1}{\sqrt{d}}\mathbb E\left[ \vW^{QK}\vw^E_{x_L}\vw^{U,T}_{\nu}\vW^{OV}\left(\sum_{u=1}^{L} \left(\sum_{s=1}^{L}a_s(\mathbf 1\{s=u\}-a_u)\vw^E_{x_s}\right)\mathbf 1\{x_u=\alpha\}\right)\right].\\
\end{align*}
It's noted that
\begin{align*}
    &\sum_{u=1}^{L} \left(\sum_{s=1}^{L}a_s(\mathbf 1\{s=u\}-a_u)\vw^E_{x_s}\right)\mathbf 1\{x_u=\alpha\} \\
    =& \sum_{u=1}^L\sum_{s=1}^L \mathbf 1\{x_u=\alpha\}\mathbf 1\{s=u\}a_s\vw^E_{x_s} - \sum_{u=1}^L\sum_{s=1}^L \mathbf 1\{x_u=\alpha\}a_ua_s\vw^E_{x_s}\\
    =& \sum_{u=1}^L \mathbf 1\{x_u=\alpha\}a_u\vw^E_{\alpha} - \sum_{u=1}^L\mathbf 1\{x_u=\alpha\}a_u\sum_{s=1}^La_s\vw^E_{x_s}\\
    =& \left(\sum_{u=1}^L\mathbf 1\{x_u=\alpha\}a_u\right)\left(\vw^E_{\alpha} - \bar{\vw}^E_X\right).
\end{align*}
Then we have
\begin{align*}
   \left.\frac{d\vw^E_\alpha}{dt}\right|^y_{\rm key} &= \frac{1}{\sqrt{d}}\mathbb E\left[\vW^{QK}\vw^E_{x_L}\left(\sum_{u=1}^L\mathbf 1\{x_u=\alpha\}a_u\right) \vw^{U,T}_{\nu}\vW^{OV}\left(\vw^E_{\alpha} - \bar{\vw}^E_X\right)\right]\\
    &= \frac{Lr_{\alpha}}{\sqrt{d}}\vW^{QK}\mathbb E\left[a_I\vw^E_{x_L}\vw^{U,T}_{\nu}\vW^{OV}\left(\vw^E_{\alpha} - \bar{\vw}^E_X\right)\mid \alpha\right].\\
\end{align*}

Conditioned on $\alpha$, this term depends on the joint distribution of the label $y$ and the last token $x_L$. Define the label-last-token signature
\begin{equation*}
\vvarphi^{y,x_L}_\alpha(\nu,\beta):=\mathbb P(y=\nu,x_L=\beta\mid \alpha),\quad \vvarphi^{y,x_L,x_s}_\alpha(\nu,\beta_L,\beta_s):=\mathbb P(y=\nu,x_L=\beta_L,x_s=\beta_s\mid \alpha).
\end{equation*}
For each pair $(\nu,\beta)$, define the key-side attention-derivative weight
\begin{align*}
\vpsi^{y,x_L,(1)}_{\alpha}(\nu,\beta)&:=\mathbb E\left[a_I\mid y=\nu,x_L=\beta,\alpha\right].\\
\vpsi^{y,x_L,x_s,(1)}_{\alpha}(\nu,\beta_L,\beta_s) &:= \mathbb E\left[a_Ia_s\mid y=\nu,x_L=\beta_L,x_s=\beta_s,\alpha\right],\qquad s=1,2,\cdots,L.
\end{align*}
and $\widetilde{\vW}^{\rm key}\in \left(\sR^d\right)^{V\times V},\widetilde{\vW}^{{\rm key},s} \in \left(\sR^d\right)^{V\times V\times V}$, where
\begin{equation*}
    \widetilde{\vW}^{\rm key}\left(\nu,\beta_L\right) = \vw^E_{\beta_L}\vw^{U,T}_{\nu}\vW^{OV}\vw^E_\alpha,\qquad \widetilde{\vW}^{{\rm key,s}}\left(\nu,\beta_L,\beta_s\right) = \vw^E_{\beta_L}\vw^{U,T}_{\nu}\vW^{OV}\vw^E_{\beta_s}.
\end{equation*}
Then
\begin{equation*}
 \left.\frac{d\vw^E_\alpha}{dt}\right|^y_{\rm key} = \frac{Lr_{\alpha}}{\sqrt{d}}\vW^{QK}\left(\widetilde{\vW}^{\rm key}\cdot \left(\vpsi^{y,x_L,(1)}_{\alpha}\odot\vvarphi^{y,x_L}_\alpha\right) - \sum_{s=1}^{L}\widetilde{\vW}^{{\rm key},s}\cdot \left(\vpsi^{y,x_L,x_s,(1)}_{\alpha}\odot\vvarphi^{y,x_L,x_s}_\alpha\right)\right).
\end{equation*}

We now rewrite the query-score term in signature form. The label-driven query-score term is
\begin{equation*}
\left.\frac{d\vw^E_\alpha}{dt}\right|^y_{\rm query}:=\frac{1}{\sqrt d}\mathbb E\left[\mathbf 1\{x_L=\alpha\}\vW^{QK,T}\sum_{u=1}^{L}\vw^E_{x_u}
\vw^{U,T}_y\vW^{OV}\sum_{s=1}^{L}a_s(\mathbf 1\{s=u\}-a_u)\vw^E_{x_s}\right].
\end{equation*}
It's noted that
\begin{align*}
&\sum_{u=1}^{L}\vw^E_{x_u}
\vw^{U,T}_y\vW^{OV}\sum_{s=1}^{L}a_s(\mathbf 1\{s=u\}-a_u)\vw^E_{x_s}\\
=& \sum_{u=1}^L \vw^E_{x_u}\vw^{U,T}_y\vW^{OV}a_u\vw^E_{x_u} - \sum_{u=1}^L \vw^E_{x_u}\vw^{U,T}_y\vW^{OV}a_u\sum_{s=1}^La_s\vw^E_{x_s}\\
=& \sum_{u=1}^La_u\vw^E_{x_u}\vw^{U,T}_y\vW^{OV}\vw^E_{x_u} - \left(\sum_{u=1}^La_u\vw^E_{x_u}\right)\vw^{U,T}_\nu\vW^{OV}\left(\sum_{u=1}^La_u\vw^E_{x_u}\right).
\end{align*}
Let $\vvarphi_{\alpha,x_L}^{y;x_s}(\nu,\beta_s) = \mathbb P\left(y=\nu,x_s=\beta_s\mid x_L=\alpha\right),\vvarphi_{\alpha,x_L}^{y;x_s,x_u}(\nu,\beta_s,\beta_u) = \mathbb P\left(y=\nu,x_s=\beta_s,x_u=\beta_u\mid x_L=\alpha\right)$ and define that
\begin{align*}
    &\vpsi_{\alpha,x_L}^{y,\rightarrow x_s}\left(\nu,\beta_s\right) = \mathbb E \left[a_s\mid y=\nu, x_s=\beta_s, x_L=\alpha\right],\\ &\vpsi_{\alpha,x_L}^{y,\rightarrow x_s,x_u}\left(\nu,\beta_s,\beta_u\right) = \mathbb E\left[a_sa_u\mid y=\nu,x_s=\beta_s,x_u=\beta_u,x_L=\alpha\right],
\end{align*}
and 
\begin{align*}
    \widetilde{\vW}^{{\rm query},u}_{\alpha}\left(\nu,\beta_u\right) = \vw^E_{\beta_u}\vw^{U,T}_\nu\vW^{OV}\vw^E_{\beta_u},\qquad \widetilde{\vW}^{{\rm Query},u,s}_\alpha \left(\nu,\beta_u,\beta_s\right)=\vw^E_{\beta_u}\vw^{U,T}_\nu\vW^{OV}\vw^E_{\beta_s}.
\end{align*}
Then we have
\begin{align*}
    \left.\frac{d\vw^E_\alpha}{dt}\right|^y_{\rm query} = &\frac{r_{\alpha}^{x_L}}{\sqrt{d}}\vW^{QK}\\
    &\left(\sum_{u=1}^L\widetilde{\vW}_{\alpha}^{{\rm query},u}\cdot\left(\vvarphi_{\alpha,x_L}^{y;x_s}\odot\vpsi_{\alpha,x_L}^{y,\rightarrow x_s}\right)-\sum_{u,s=1}^L\widetilde{\vW}_{\alpha}^{{\rm query},u,s}\cdot\left(\vvarphi_{\alpha,x_L}^{y;x_s,x_u}\odot\vpsi_{\alpha,x_L}^{y,\rightarrow x_s,x_u}\right)\right).
\end{align*}

\section{Multi-stage embedding evolution in length-three sequences.}\label{app:length3}

The experiments above consider sequence-length-two tasks, where only the zeroth- and first-order signatures are accessible. To examine the late-stage regime with multiple higher-order contributions, we now consider a sequence-length-three task in which both first- and second-order signatures can influence the embedding dynamics. Proposition~\ref{prop:late_stage} suggests that, once the initialization-induced suppression is lifted, these accessible orders have the same leading parameter scale and can therefore contribute jointly, while the second-order contribution is expected to be the largest among them.

\begin{task}
Let $\mathcal{A}$ and $\mathcal{X}$ be finite sets of integers, where $\mathcal{A}$ is the set of anchor tokens. We consider input sequences
\begin{equation*}
    X=(a_1,a_2,x),
    \qquad
    a_1,a_2\in\mathcal{A},
    \quad
    x\in\mathcal{X},
\end{equation*}
with input space $\mathcal{D}=\mathcal{A}\times\mathcal{A}\times\mathcal{X}$. The target label is defined as
\begin{equation*}
    y=F(a_1,a_2,x)
    =(a_1+a_2+x)\bmod |\mathcal{X}|.
\end{equation*}
\end{task}

The anchor tokens and the contextual token are drawn from disjoint token ranges. For any fixed anchor token, marginalizing over the remaining inputs produces the same label distribution. Consequently, all anchor tokens have identical zeroth-order label signatures. In contrast, the joint token--label statistics retained by the first- and second-order signatures depend on the anchor identity and can therefore distinguish different anchor tokens.

Figure~\ref{fig:length3_dynamics} shows the evolution of the anchor-token embedding geometry. At the early stage, the cosine similarities are nearly uniform and the PCA projection collapses into a compact cluster, consistent with the identical zeroth-order signatures. At the intermediate stage, a banded similarity pattern emerges and the PCA projection forms a smooth one-dimensional curve, indicating that contextual signature effects have begun to appear. At the late stage, the similarity matrix becomes less purely banded and the PCA projection develops a more complex geometry, consistent with the joint influence of the first- and second-order signatures rather than the replacement of one order by the other.

To quantify this evolution, let $S^{(m)}$ denote the pairwise similarity structure induced by the $m$th-order signatures. We compare the off-diagonal entries of $S^{(m)}$ with those of the embedding cosine-similarity matrix. The correlation with $S^{(0)}$ briefly approaches one at the beginning of training, but then becomes strongly negative and gradually moves back toward zero. In contrast, the correlations with both $S^{(1)}$ and $S^{(2)}$ rapidly approach one and remain high throughout training. Thus, the late-stage geometry simultaneously reflects both accessible higher-order signatures. Moreover, $S^{(2)}$ provides a consistently stronger fit than $S^{(1)}$ in the later stage, matching the theoretical picture that the same-scale higher-order contributions coexist while the highest-order contribution is the largest. The zeroth-order geometry is therefore transient, whereas the first- and second-order statistics jointly shape the later embedding space.

\begin{figure}[htb]
    \centering
    \includegraphics[width=1\linewidth]{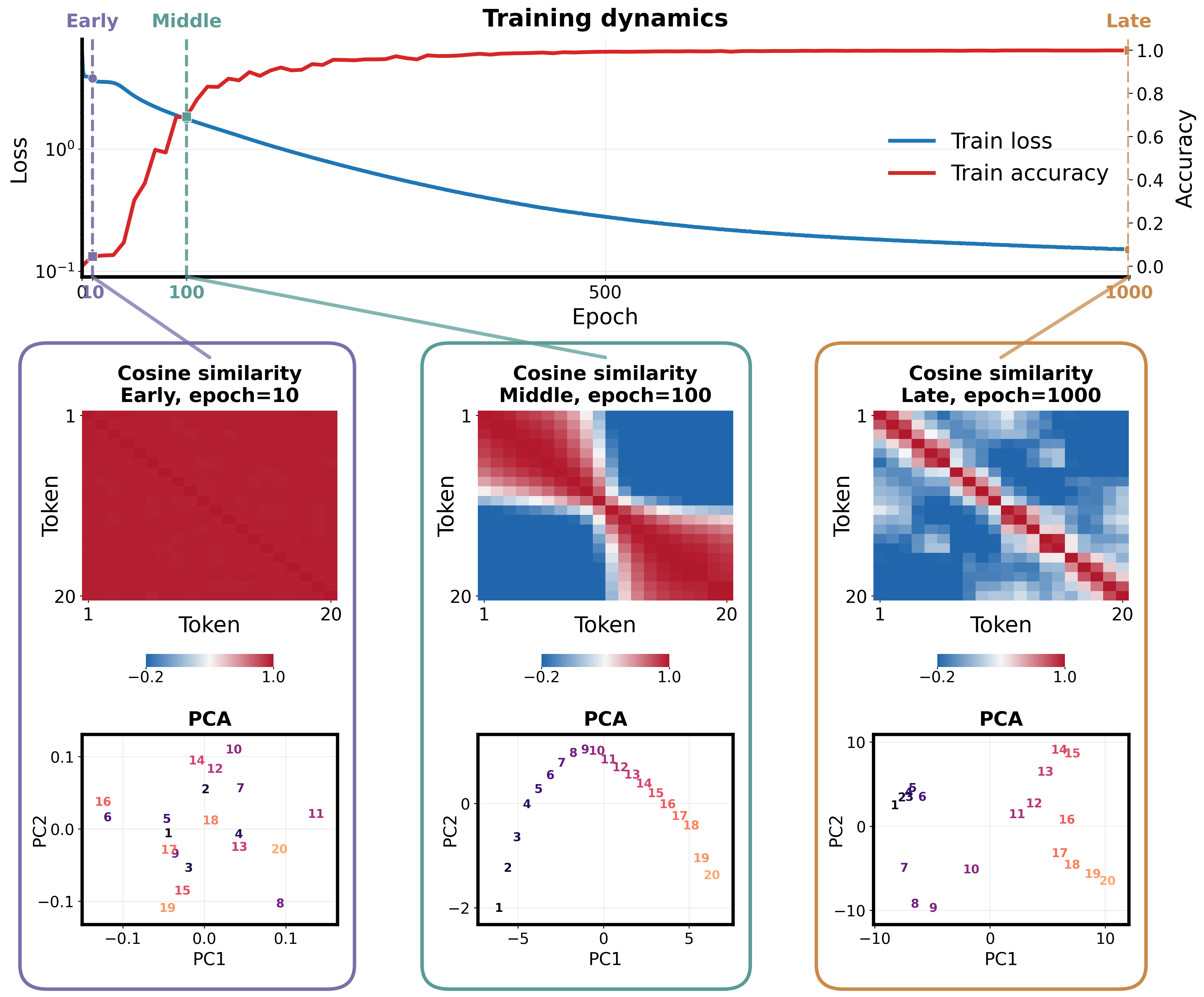}
    \caption{The top panel shows the training loss and accuracy over 1,000 epochs, with dashed vertical lines marking the early, middle, and late checkpoints at epochs 10, 100, and 1,000. The three lower columns show the anchor-token cosine-similarity matrices and their two-dimensional PCA projections at the corresponding checkpoints. Token indices are displayed directly in the PCA projections.
}
    \label{fig:length3_dynamics}
\end{figure}

\begin{figure}[htb]
    \centering
    \includegraphics[width=0.7\linewidth]{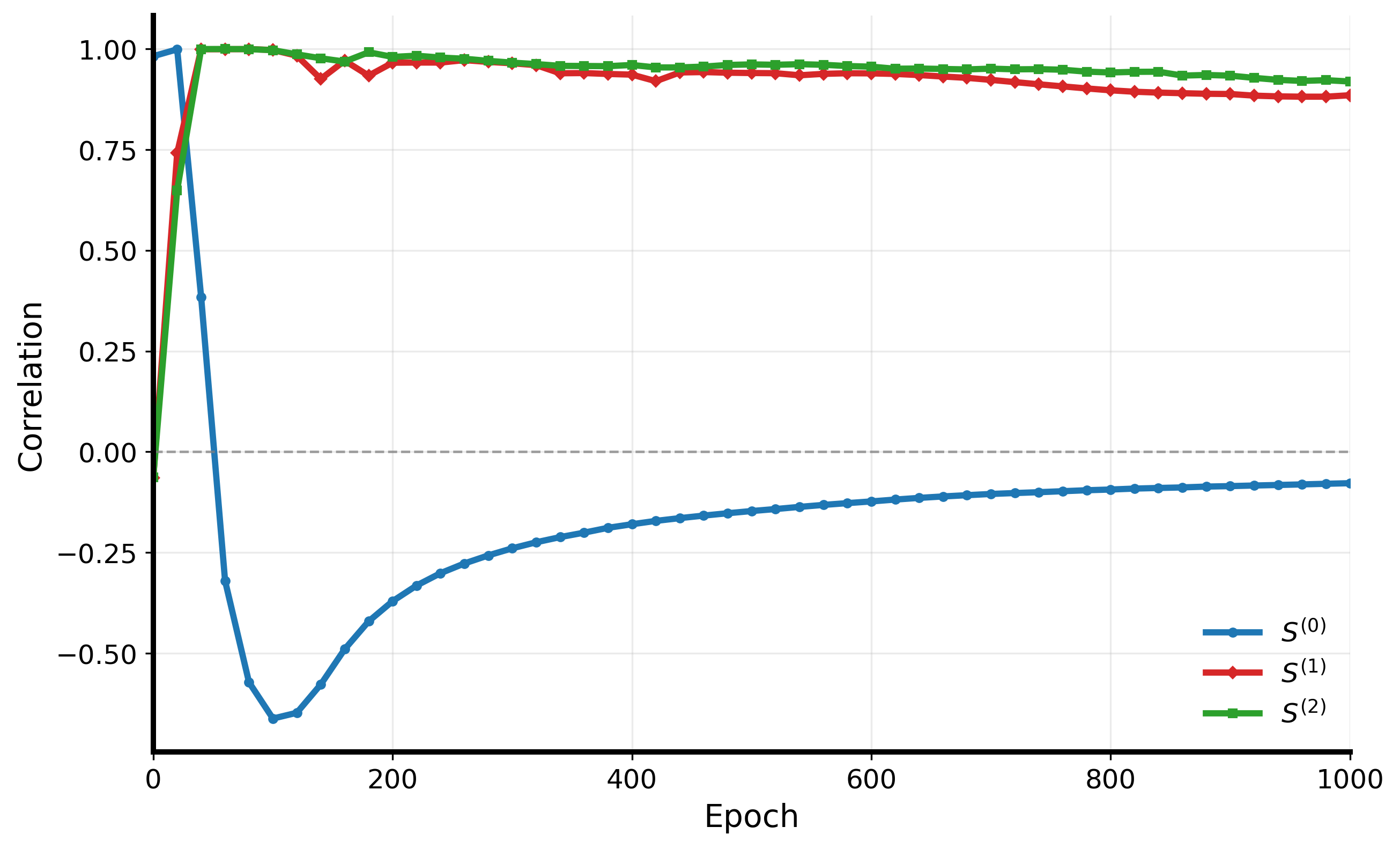}
    \caption{Correlation between the off-diagonal entries of the anchor-token embedding cosine-similarity matrix and the similarity matrices induced by the zeroth-, first-, and second-order signatures, denoted by $S^{(0)}$, $S^{(1)}$, and $S^{(2)}$, respectively, over training epochs. The horizontal dashed line marks zero correlation.}
    \label{fig:leng3_corr}
\end{figure}

\section{Different Seeds, Different Embedding}\label{app:seed_embedding}

Even with the dataset, architecture, and training configuration fixed, changing only the random seed can lead to visibly different final embedding geometries. We observe this phenomenon in both the addition and modular-addition tasks. Rather than being arbitrary, these differences can be explained by the frequency components of the higher-order signatures.

For each seed, we compare the pairwise similarity structure of the final embeddings with that induced by each Fourier component of the corresponding signature. Specifically, we identify
\begin{equation*}
    k^{*}
    =
    \arg\max_{k>10}
    \operatorname{Corr}
    \bigl(
        S_{\mathrm{emb}},
        S_{\mathrm{sig}}^{(k)}
    \bigr),
\end{equation*}
where $S_{\mathrm{emb}}$ is the embedding similarity matrix and
$S_{\mathrm{sig}}^{(k)}$ is the similarity matrix induced by frequency
$k$. As shown in Figures~\ref{fig:seed_addition} and
\ref{fig:seed_modular_addition}, the selected frequency $k^{*}$ varies
across seeds. Correspondingly, the PCA geometry of the embeddings closely
matches the projection of the signature component at the selected
frequency. Thus, different initializations lead the embeddings to align
with different spectral components of the same higher-order signature,
producing distinct but signature-governed geometries.
\begin{figure}[htpb]
    \centering
    \includegraphics[width=1\linewidth]{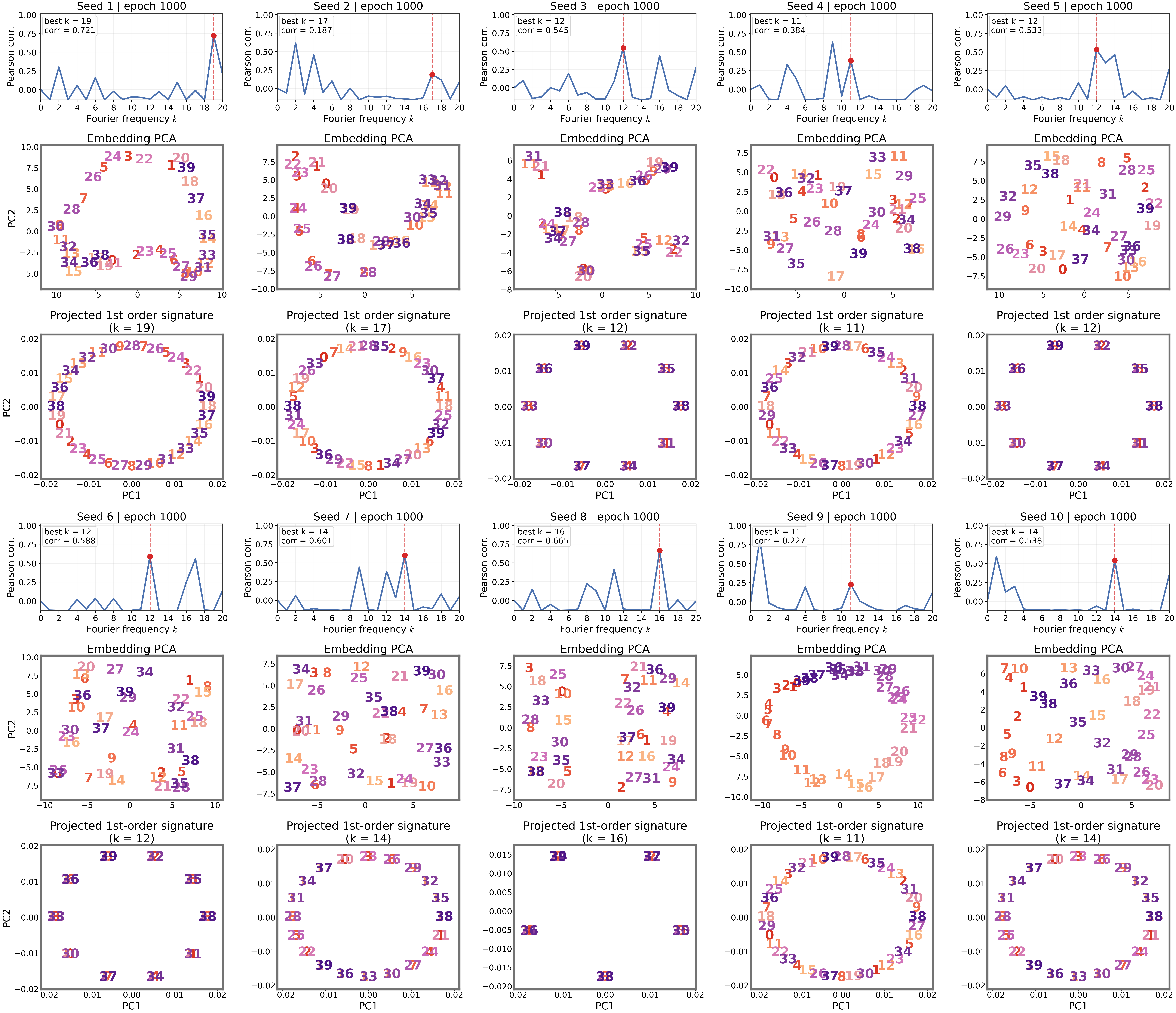}
    \caption{Results for ten random seeds at epoch 1,000. For each seed, the top panel shows the Pearson correlation between the embedding similarity structure and that induced by each Fourier frequency $k$ of the first-order signature, with the red marker and dashed line indicating the frequency of maximum correlation. The middle and bottom panels show the two-dimensional PCA projections of the learned embeddings and the first-order signature component at the selected frequency, respectively. Token indices are displayed directly in the PCA projections.}
    \label{fig:seed_modular_addition}
\end{figure}

\begin{figure}[htpb]
    \centering
    \includegraphics[width=1\linewidth]{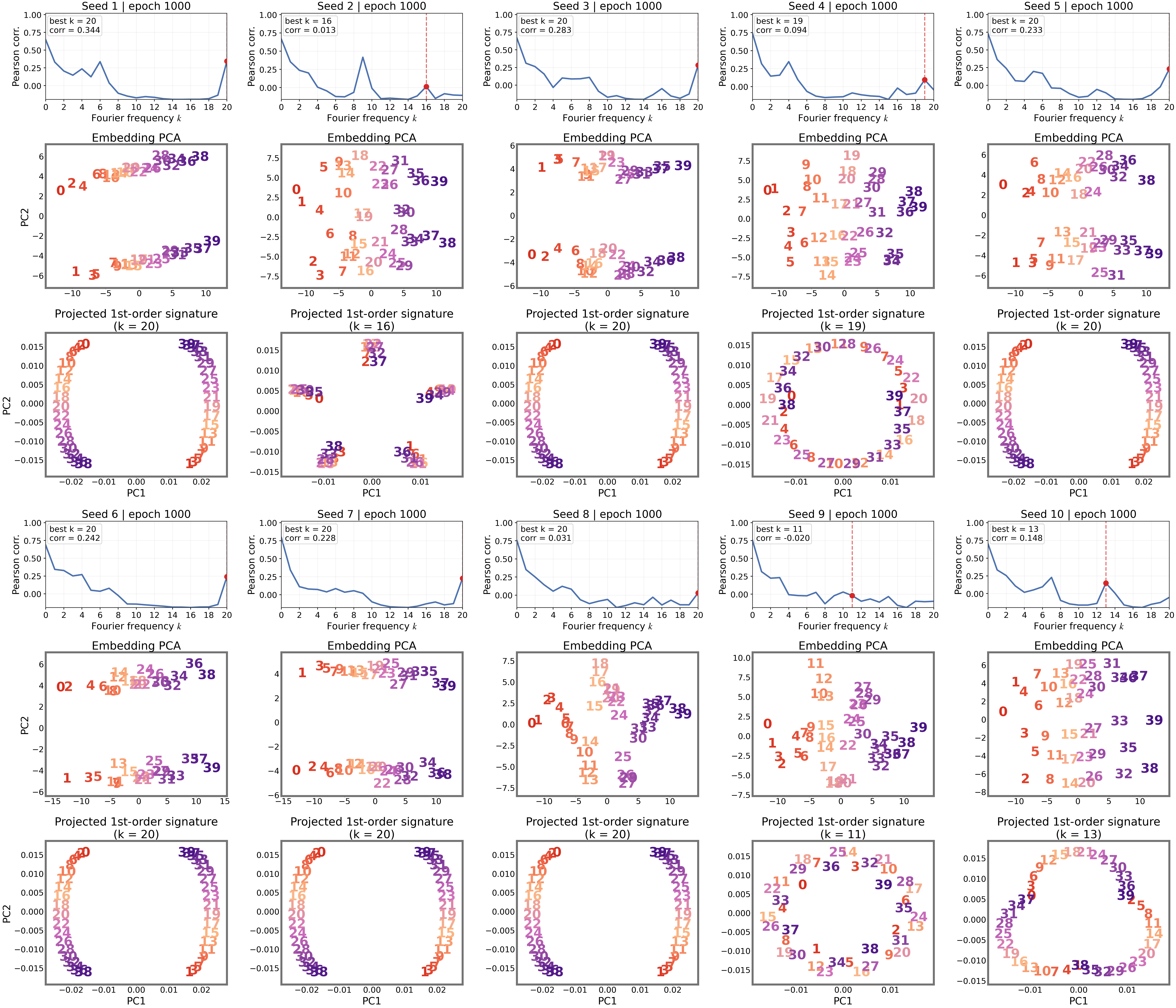}
    \caption{Results for ten random seeds at epoch 1,000. For each seed, the top panel shows the Pearson correlation between the embedding similarity structure and that induced by each Fourier frequency $k$ of the first-order signature, with the red marker and dashed line indicating the frequency of maximum correlation. The middle and bottom panels show the two-dimensional PCA projections of the learned embeddings and the first-order signature component at the selected frequency, respectively. Token indices are displayed directly in the PCA projections.}
    \label{fig:seed_addition}
\end{figure}

\clearpage

\vskip 0.2in
\bibliography{sample}

\end{document}